\documentclass{article}

\usepackage{arxiv}

\usepackage[utf8]{inputenc} 
\usepackage[T1]{fontenc}    
\usepackage{url}            
\usepackage{booktabs}       
\usepackage{amsfonts}       
\usepackage{nicefrac}       
\usepackage{microtype}      
\usepackage{lipsum}         
\usepackage{graphicx}
\usepackage{natbib}
\usepackage{doi}

\usepackage{algorithm}
\usepackage{algorithmic}

\usepackage{enumitem}
\usepackage{amsmath}
\usepackage{stmaryrd}
\usepackage{amssymb}
\usepackage{alltt}
\usepackage{tikz}
\usepackage{tikz-qtree}
\usetikzlibrary{shapes,arrows}
\usetikzlibrary{backgrounds,fit,arrows,positioning, calc}
\usepackage{array}
\usepackage{amsthm}
\newtheorem{Definition}{Definition}
\newtheorem{Lemma}{Lemma}
\newtheorem{Theorem}{Theorem}
\newtheorem{Example}{Example}

\title{KOS-TL (Knowledge Operation System Type Logic): \\ A Constructive Foundation for Executable Knowledge Systems}

\newif\ifuniqueAffiliation
\uniqueAffiliationtrue

\ifuniqueAffiliation 
\author{ \href{https://orcid.org/0000-0002-4298-1834}{\includegraphics[scale=0.06]{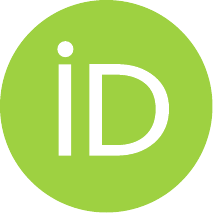}\hspace{1mm}Chen Peng}\thanks{Chen Peng, Doctor of Computer Science, born in May 1979, from Nanfeng County, Jiangxi Province.} \\
	School of Information Science \\
	Beijing University of Language and Culture\\
	Beijing 100081 \\
	\texttt{chenpeng@blcu.edu.cn} \\
}
\else
\usepackage{authblk}

\newbox{\orcid}\sbox{\orcid}{\includegraphics[scale=0.06]{orcid.pdf}}
\author[1]{%
	\href{https://orcid.org/0000-0000-0000-0000}{\usebox{\orcid}\hspace{1mm}David S.~Hippocampus\thanks{\texttt{hippo@cs.cranberry-lemon.edu}}}%
}
\author[1,2]{%
	\href{https://orcid.org/0000-0000-0000-0000}{\usebox{\orcid}\hspace{1mm}Elias D.~Striatum\thanks{\texttt{stariate@ee.mount-sheikh.edu}}}%
}
\affil[1]{Department of Computer Science, Cranberry-Lemon University, Pittsburgh, PA 15213}
\affil[2]{Department of Electrical Engineering, Mount-Sheikh University, Santa Narimana, Levand}
\fi

\renewcommand{\shorttitle}{\textit{KOS-TL} Knowledge Operation System Type Logic}

\hypersetup{
pdftitle={A template for the arxiv style},
pdfsubject={q-bio.NC, q-bio.QM},
pdfauthor={David S.~Hippocampus, Elias D.~Striatum},
pdfkeywords={First keyword, Second keyword, More},
}

\usepackage{hyperref} 
\usepackage{cleveref} 

\begin{document}
\maketitle

\begin{abstract}
As knowledge representation shifts from static databases to evolving operational systems, traditional logics face limitations in handling event-driven state transitions. This paper proposes Knowledge Operation System Type Logic (KOS-TL), or "Zhi-Xing Logic" By integrating intuitionistic dependent type theory with small-step operational semantics, KOS-TL establishes a unified framework for static knowledge constraints, dynamic state evolution, and physical environment refinement.
\end{abstract}


\section{Introduction}
The core challenge of modern knowledge systems has shifted from conceptual modeling to the execution, traceability, and verification of operations within event-driven environments. Traditional Description Logic (DL) frameworks, rooted in static model-theoretic semantics, struggle to express "how knowledge is updated" and "how operations are executed."

Knowledge representation and reasoning is an important application domain of logic. However, with the proliferation of large-scale data integration and complex decision-making systems, the research objects in knowledge representation and reasoning are gradually shifting from static knowledge bases to continuously evolving knowledge operation systems. In knowledge operations oriented toward continuous development, concept modeling or ontology consistency verification is no longer the core focus; the core challenge has transformed into how to perform executable, traceable, and verifiable operations on knowledge in environments driven by events, state evolution, and strong engineering constraints.

In the face of new demands in the field of knowledge operations, existing mainstream logical frameworks (especially formal systems represented by Description Logic (DL) and its Semantic Web implementations such as OWL) exhibit fundamental mismatches with the aforementioned requirements in terms of theoretical assumptions and semantic structures. This is specifically manifested in the following aspects.
\begin{enumerate}[label=(\arabic*)]
\item Tension between static semantics and dynamic operations\\
Description Logic is based on static model-theoretic semantics, with its core reasoning problems revolving around concept satisfiability, concept inclusion, and instance checking. This paradigm assumes knowledge describes "possible states of the world" rather than "system runtime states." In contrast, the core objects in knowledge operations are events, operations, and state transitions, where the fundamental questions are no longer whether a certain assertion is true in some model, but whether a particular operation can be legally executed and how the system state evolves after its execution.
The model-theoretic semantics of Description Logic centers on static interpretive structures, with the primary goal of characterizing "what the world might logically be like." Concepts are interpreted as subsets of the individual domain, roles as binary relations, and reasoning problems primarily focus on satisfiability, concept inclusion, and instance checking. This semantic structure is naturally suited for taxonomies, ontology engineering, and terminological reasoning.
However, in the application scenarios faced by knowledge operations, knowledge is not a static collection but a state system that evolves continuously over time. The core issues are no longer "whether a certain assertion is true in some model," but "whether a certain event has occurred, whether a certain state has been updated, and which new facts these changes will trigger." Description Logic does not treat events and state transitions as first-class logical objects; its support for dynamic processes can only be achieved indirectly through external mechanisms or reification, which is costly in engineering terms and semantically opaque.
\item Fundamental conflict between open world assumption and operational semantics\\
The open world assumption (Open World Assumption) adhered to by Description Logic fundamentally conflicts with the closed or semi-closed world semantics commonly used in knowledge operations. In engineering practice, missing information is often treated as an abnormal state or basis for operation failure, rather than logical "unknown."
Description Logic adheres to the open world assumption (Open World Assumption, OWA), where unknown does not equate to false. This assumption is reasonable in Semantic Web and open knowledge environments but often becomes an obstacle in knowledge operations. In scenarios such as enterprise governance, risk control, and compliance auditing, "missing records" themselves constitute negative information or abnormal states.
The semantics of knowledge operating systems are closer to the closed world assumption: facts that do not appear are regarded as non-occurring events, and unsatisfied constraints as system errors. This semantic orientation, centered on closed worlds and executable constraints, makes the existential reasoning of Description Logic models difficult to directly serve actual system operations.
\item Differences between conceptual semantics and nominal type semantics\\
The conceptual semantics in Description Logic is extensional, with membership dynamically determined through reasoning; in practical knowledge operations, however, types more often play nominal and constraint roles. Whether an object belongs to a certain type is not derived through logical entailment but is a prerequisite that must be satisfied during data ingestion and operation phases, directly determining whether an operation is legal and whether the process can continue. Type errors manifest as system non-executable states rather than mere reasoning failures. This type semantics is more akin to type systems in programming languages and operating systems than traditional ontology logic.
\item Fundamental shift in reasoning objectives\\
Finally, the reasoning objectives of Description Logic primarily involve proving logical entailment, whereas "reasoning" in knowledge operations is more akin to rule-driven fact materialization—i.e., under the given current data state, which new facts should be immediately generated, stored, and involved in subsequent computations\footnote{This kind of "reasoning" is closer to rule-driven approaches such as Datalog, triggers, and operational semantics, rather than classical logical reasoning centered on proof theory or model theory.}. The results of reasoning are not merely used to answer queries but directly alter the system's observable state and impose constraints on subsequent operations. This reasoning mode lacks direct formal characterization in classical logical frameworks.
\end{enumerate}
In the field of knowledge operations, the issue is not whether logical systems are sufficiently powerful, but whether they can natively support events, time, state changes, and executable rules. Traditional Description Logic remains irreplaceable in static knowledge representation, but its logical assumptions and semantic structures are not suitable for direct use as the kernel of knowledge operating systems.
The aforementioned differences reveal a key theoretical gap: traditional logical systems have yet to provide a unified formal foundation for systems where "knowledge is treated as an operable object." This directly gives rise to the demand for a new logical system: one that takes events and state transitions as core objects, employs type-driven and operational semantics for reasoning, incorporates built-in type constraints to characterize operational legality, supports rule-driven state evolution, and achieves an engineering-feasible balance between reliability and termination.
To address this theoretical gap, we propose and study a new formal logical system—"Zhi-Xing Logic" (Knowledge Operation System Type Logic, abbreviated as KOS-TL). KOS-TL aims to build on intuitionistic type theory, introducing eventification and operational semantics, enabling the logical system to natively characterize the processes of knowledge comprehension, operation, and state updating, thereby providing a verifiable, executable, and extensible logical kernel for knowledge operating systems.

\section{The Architecture of KOS-TL}
\label{sec:headings}

To integrate logical rigor with operational expressiveness, "Zhi-Xing Logic" (KOS-TL) is defined as a \emph{layered formal system}. It consists of three distinct layers of formal definitions: \emph{Core}, \emph{Kernel}, and \emph{Runtime}, each differentiated in terms of logical roles and semantic commitments (as shown in Table~\ref{tab:kos-layers}).
\begin{table}[htbp]
\centering
\caption{Overview of KOS-TL Layered Structure}
\label{tab:kos-layers}
\begin{tabular}{p{1.5cm}p{3cm}p{6cm}p{4cm}p{2cm}}
\toprule
\textbf{Layer} & \textbf{Formal Name} & \textbf{Core Responsibilities} & \textbf{Logical Objects} & \textbf{Decidability} \\
\midrule
L0 Core & Static Truth Layer (Logic) & Defines ``what is valid.'' Establishes type constructions and constraints based on intuitionistic dependent type theory (ITT). & Types $T$, proof terms $p$, propositions $P$ & Strongly decidable \\
L1 Kernel & Dynamic Transition Layer (Dynamics) & Defines ``how to change.'' Introduces small-step operational semantics to handle event-driven state transitions. & States $\sigma$, events $e$, transitions $\to$ & Locally decidable \\
L2 Runtime & Environment Evolution Layer (System) & Defines ``how to run.'' Handles external I/O, timeline mounting, and nondeterministic inputs. & Queues $Q$, external sources $\mathcal{E}_{\mathrm{env}}$ & Semi-decidable \\
\bottomrule
\end{tabular}
\end{table}

The layering principles of "Zhi-Xing" Logic are as follows:
\begin{itemize}
    \item Logical validity is determined solely by the core layer;
    \item Operational correctness is enforced by the kernel layer;
    \item System evolution is realized in the runtime layer.
\end{itemize}
This layered design ensures that the correctness of the core logic is maintained even in open, nondeterministic system evolutions.

\subsection{L0: Core (Core Layer) --- Static Logical Domain}

\subsubsection{Core Layer Requirements Analysis and Logical Construction}
In terms of constructive expression requirements, traditional systems often focus solely on the static storage of data, whereas in high-security and high-trust scenarios, the system must transition to the storage of knowledge.'' Based on Intuitionistic Type Theory \cite{MartinLof1984}, each knowledge item is defined as a dependent pair $\Sigma(d:D).P(d)$, where data $d$ is strongly coupled with the credential $P(d)$ that proves it satisfies the business ontology. This design eliminates rootless data'' at the source, ensuring that all knowledge entering the kernel undergoes constructive verification.
Deep alignment between physics and logic is another key requirement. Through the expressive power of dependent types (Dependent Types), the system can internalize physical laws and compliance constraints from industrial or financial domains, rather than relying on external ad-hoc logical judgments. This paradigm of ``Make Illegal States Unrepresentable'' ensures that any attempt to perform an operation violating axioms fails at the type-checking stage, thereby maintaining the system's steady state.
Furthermore, to address high-frequency compliance requirements, the system introduces computational reflexivity (Computational Reflexivity). By requiring the Core Layer to describe its own reduction rules, the system can automatically synthesize equivalence proofs $\text{Id}(t, t')$ during the execution of every small-step logical evolution ($\beta, \iota, \delta$ reductions). This full-path automated auditing elevates traditional post-hoc investigation to runtime real-time formal verification. Finally, by defining elaboration operator ($\textsf{elab}$) templates, the Core Layer provides a semantic elevation benchmark for runtime (Runtime) signals, establishing a unique mapping path from physical bits (Bits) to logical truths (Truth).
\subsubsection{Overall Architecture Description}
The Core Layer is architected as a strongly verified microkernel based on dependent type theory, primarily driven by the following three functional modules: First, the type constructor and ontology manager act as the system's legislators,responsible for transforming business domain ontologies into sorts (Sorts) and dependent type structures in the type system, clearly defining the boundaries of legal objects. Second, the reduction engine serves as the system's reasoning machine,handling fine-grained evolution of logical terms by executing function applications and structural decompositions to compute the logical steady state after knowledge evolution. Third, the type checker acts as the system's gatekeeper,'' performing bidirectional type checking (Bidirectional Type Checking) to ensure that all operations entering the Kernel Layer satisfy pre-verified correctness (Correct-by-construction).

\subsubsection{Key Design Decisions}

The Core Layer's technology selections and architectural decisions reflect a balance between logical rigor and engineering feasibility. \begin{enumerate}[label=(\arabic*)]
\item Replacement of First-Order Logic (FOL) with Dependent Type Theory (MLTT)

Traditional knowledge bases rely on FOL or description logics (DL), leading to a disconnect between logical assertions and concrete data. The Core Layer selects dependent type theory, utilizing $ \Sigma $   types to achieve atomic encapsulation of data and constraints. This decision resolves the persistent issue of evidence absence in knowledge operations, enforcing physical constraints at the architectural level.
\item Introduction of Computational Reflexivity and Endogenous Auditing

To address the risks of traditional auditing being delayed and easily tampered with, the Core Layer models reduction rules as logical processing objects. Whenever a state changes, the system automatically synthesizes an identity proof (Identity Proof). This design transforms auditing behavior into an automated type-checking process; as long as the proof chain is complete, the system behavior achieves absolute compliance.
\item Dual-Universe System and Proposition Shrinking (Prop-Shrinking)

To solve the computational overhead problem posed by formal proofs, the Core Layer designs multi-level universes ($\mathcal{U}_i$) for complex modeling and introduces a dedicated proof space (Prop). Based on the proof irrelevance (Proof Irrelevance) principle, the system contracts proof details through type erasure (Erasure) techniques after completing rigorous verification. This decision achieves a balance between logical depth and engineering efficiency, supporting efficient processing of large-scale real-time knowledge streams.
\item Paradigm Shift from Denotational Semantics to Operational Semantics

Traditional logic emphasizes static truth values, whereas the KOS-TL Core Layer defines small-step reduction (Small-step reduction) rules for logical terms, defining knowledge operations as a reduction process. This decision ensures that the transformation from the Core Layer to the Kernel Layer is lossless and deterministic, achieving high isomorphism between logical inference steps and physical computation steps.
\end{enumerate}
The comparison between the Core Layer's decisions and traditional knowledge base architectures is shown in Table~\ref{tab:KosCorecomparison}.
\begin{table}[h]
\caption{Comparison of Traditional Architecture and KOS-TL Core Layer Decisions}
\label{tab:KosCorecomparison}
\centering
\begin{tabular}{p{3cm}|p{6cm}|p{6cm}}
\hline
\textbf{Design Dimension} & \textbf{Traditional Architecture Decision} & \textbf{KOS-TL Core Layer Decision} \\
\hline
Knowledge Carrier & Database Records + External Validation & Logical Terms (Terms) with Dependent Proofs \\
\hline
Constraint Trigger & Runtime Interception / Business Code if-else & Type Checking  \\
\hline
Evolution Driver & Database Transactions  & Logical Reduction \\
\hline
Trust Root & System Administrator Permissions / Log Records & Immutable Mathematical Proof Chain \\
\hline
\end{tabular}
\end{table}

\subsection{L1: Kernel (Kernel Layer) — Operational Semantics Domain}

In the layered architecture of the KOS-TL system, the Kernel Layer undertakes the central function of transforming the static truths defined by the Core Layer into dynamic evolutionary drivers. If the Core Layer is likened to the constitution, then the Kernel Layer serves as the administrative hub'' and power engine. Its core objective is to ensure, through formal reduction mechanisms, that the knowledge system maintains logical determinism and evolutionary continuity when handling high-frequency business streams and time streams.

\subsubsection{Kernel Layer Requirements Modeling: Dynamic Evolution and State Determinism}

For complex knowledge operating systems, the Kernel Layer's design primarily addresses the issues of collapse'' and ``generation'' of knowledge states.
First is the causal traceability requirement: the system mandates that every change in knowledge state must have explicit causal associations, meaning all changes must trace back to specific elaboration events (Event). Second is the state consistency requirement: especially in distributed or concurrent environments, atomicity of state transitions must be guaranteed, eradicating conflicting states at the logical level. Finally, real-time performance and progress guarantees (Progress Guarantee): the system must ensure that, upon receiving valid inputs, it can logically deterministically evolve to the next stable state, avoiding undefined behaviors or logical deadlocks caused by non-determinism.

\subsubsection{Design Methodology: Small-Step Operational Semantics and State Machine Model}
The Kernel Layer rejects traditional black-box batch processing and is instead built upon formalized small-step operational semantics (Small-step Operational Semantics).
The core of its methodology lies in the state triple model. The Kernel abstracts the system as $\sigma = \langle \mathcal{K}, \mathcal{TS}, \mathcal{P} \rangle$, where $\mathcal{K}$ represents the current verified set of knowledge truths, $\mathcal{TS}$ is the coupling of logical clocks and physical anchors, and $\mathcal{P}$ is the queue of elaborated but unexecuted events. The Kernel Layer is essentially event-centric, based on the ontological status of events \cite{Davidson1967}, where events are defined as state transition operators (Events as Transitions): each type of event corresponds to an explicit transition operator that, by synchronously invoking the Core Layer's judgment results, drives the monotonic evolution from the old state $\sigma$ to the new state $\sigma'$.

\subsubsection{Overall Architecture Description}
As the "logical router" between the logical foundation (Core) and the physical environment (Runtime), the Kernel Layer collaborates through three core modules:
\begin{itemize}
\item Event Queue Manager: Responsible for receiving event packets $\langle e, p \rangle$ elaborated from the Runtime Layer, and performing strict sequencing (Sequencing) and dependency conflict detection.
\item Evolution Scheduler: Drives the system's core evolution loop. This module employs a Peek-Verify-Reduce-Confirm'' process, i.e., invoking the Core Layer to verify preconditions before consuming events, executing reduction operations, and validating postconditions. Its operating objects are always logical terms (Terms) rather than underlying physical bits. \item State Mirror: Maintains the latest truth view'' at the logical level, providing consistent context for state queries from the Runtime Layer and contextual verification from the Core Layer.
\end{itemize}

\subsubsection{Key Design Decisions}

\begin{enumerate}[label=(\arabic*)]
\item Decoupling of Strong Sequential Commit and Asynchronous Elaboration

To address the contradiction between the high-frequency generation of physical signals and the time-consuming nature of logical verification (deep proofs), the Kernel Layer adopts a decision of asynchronous elaboration, sequential commit.'' The Runtime Layer can execute signal elaboration in parallel, but the Kernel Layer insists on sequential commit (Sequential Commit). This decision ensures the uniqueness of the causal chain, presenting the knowledge evolution trajectory as a deterministic linear path, mitigating risks of complex logical branch backtracking.
\item Closed-Loop Evolution: Precondition Verification and Postcondition Evidence Synthesis

To ensure that the system state after operations continues to conform to the ontology definition, the Kernel Layer establishes a closed-loop mechanism of pre-verification, post-synthesis.'' Before executing an operation, the Kernel forcibly verifies the $Pre(e)$ provided by the Core Layer; after the operation completes, it automatically triggers the Core Layer to synthesize a new proof for $Post(e)$. This decision guarantees that the system always transitions between states proven to be true,'' eliminating logical vacuum periods.
\item Loose Coupling Mechanism Between Logical Time and Physical Time

To tackle the pain point of physical time precedence inconsistencies with causal logic in distributed networks, the Kernel Layer maintains an independent logical order. Physical timestamps serve only as evidence anchors, with synchronization occurring only during the materialization phase (Runtime). This decision leverages time evidence in logical proofs for reordering, fundamentally solving the phantom timing'' problem and ensuring absolute fidelity of the \textbf{Causality Order}.
\item Deterministic Reduction

The Kernel Layer strictly prohibits non-deterministic (Non-deterministic) choices. If an event points to multiple possible evolution branches during reduction, it will be deemed a type error in the Core Layer. This decision greatly enhances the system's predictability: under the same initial state and event sequence, the Kernel Layer will produce a mathematically unique, reproducible knowledge view.
\end{enumerate}

\subsection{L2: Runtime Layer — System Evolution Domain}

In the architecture of KOS-TL (ZhiXing Logic), the Runtime Layer is defined as the key anchor point between the logical world and the non-deterministic physical world. If the Core Layer is regarded as the system's constitution and the Kernel Layer as the administrative hub,'' then the Runtime Layer serves as the system's senses, limbs, and physical carrier.'' Its core mission is to solve the ``last mile'' problem of formal logic implementation in complex engineering environments, achieving seamless integration between logical semantics and underlying physical signals.

\subsubsection{Requirements Modeling: Physical Fidelity and Environmental Elaboration}
The design of the Runtime Layer aims to address the fundamental tension between the physical world and logical abstractions.
First is the signal elevation requirement: raw bit streams (Raw Bits) generated by the physical world lack intrinsic semantics, and the Runtime Layer needs to elevate'' them into event objects with logical connotations. Second is the resource mapping requirement: abstract state updates output from the logical layer must ultimately be precisely implemented in disk bits, memory entries, or hardware controller voltage states. Finally, real-time and concurrency requirements: when handling high-frequency sensor signals, the system must ensure low-latency ingestion without disrupting the causal order consistency maintained by the Kernel Layer.

\subsubsection{Design Methodology: Elaboration and Materialization in Bidirectional Mapping}

The methodological core of the Runtime Layer lies in establishing a bidirectional elaboration relationship between logical semantics and physical resources.
\begin{enumerate}[label=(\arabic*)]
\item Elaboration (Refinement/Elaboration) Methodology

The system introduces the $ \textsf{elab} $ operator to achieve signal elevation. Unlike traditional syntactic parsing (Parsing), $ \textsf{elab} $ is a proof construction process: it references the Core Layer's ontology templates to find logical evidence $ p $ for the raw signal $ s $, thereby converting it into a proven event $ \langle e, p \rangle $ that conforms to the Kernel interface. This process ensures that all inputs entering the system have a formalized basis for legitimacy.
\item Materialization Methodology

Through the $ \mathcal{M} $ operator, the Runtime Layer is responsible for degrading the abstract knowledge state $\mathcal{K}$ in the Kernel Layer into concrete physical forms. For example, materializing a logical transfer successful assertion into an ACID transaction in a database or a transaction entry on a blockchain. The materialization mechanism ensures the faithful execution of logical conclusions in the physical world.
\end{enumerate}

\subsubsection{Overall Architecture Description}
As the hub responsible for cross-border interaction,the Runtime Layer consists of the following key components:
\begin{itemize}
\item Elaboration Engine (Elaborator): Connects to physical I/O devices, responsible for monitoring external interrupts and sensor data streams, and performing bidirectional elaboration to anchor non-deterministic signals into deterministic logical events.
\item Physical Storage Manager (Physical Storage Manager): Abstracts underlying media differences, manages databases and memory mappings, ensuring the atomicity and persistence of materialization operations.
\item Scheduler Relay (Scheduler Relay): Acts as a buffer for the Kernel's sequential evolution, responsible for concurrent ingestion of multi-threaded signals and ordered queuing.
\end{itemize}
\subsubsection{Key Design Decisions}
\begin{enumerate}[label=(\arabic*)]
\item Input Filtering Decision Based on Elaboration

 To address the problem of traditional systems directly reading variables being susceptible to environmental interference (dirty data), the Runtime Layer mandates that all inputs must pass through the $ \textsf{elab} $operator. If a signal cannot construct a valid proof $ p $ in the Core Layer, it is deemed invalid input and immediately discarded. This decision establishes a logical firewall, ensuring the Kernel Layer is not polluted by unexpected signals.
 \item Atomic Commit Fence

 To solve the knowledge-action inconsistency problem (i.e., logical update succeeds but physical write fails), the Runtime Layer introduces a fence mechanism similar to two-phase commit. Only after the physical layer returns a write acknowledgment (ACK) does the logical clock $\mathcal{TS}$ formally advance. This decision achieves precise synchronization between physical storage and logical truths.
\item Resource Abstraction and Multi-Backend Plug-and-Play Support
To adapt to diverse environments from embedded devices to cloud clusters, the Runtime Layer designs the materialization operator $\mathcal{M}$ as a pluggable backend. This decision realizes logical portability: the logical axioms of the Core Layer and Kernel Layer remain unchanged, and cross-platform semantic consistency can be achieved by simply replacing the Runtime Layer backend.
\item Deterministic Trajectory Replay and Post-Disaster Self-Healing

Based on the unreliability of the physical environment, the Runtime Layer fully records the original trajectories of all elaboration events. Leveraging the deterministic reduction characteristics of the Kernel Layer, the system can, after a failure, replay the event stream to logically re-derive and repair the physical configuration. This decision provides the system with extremely strong logical robustness and self-healing capabilities.
\end{enumerate}
The hierarchical structure and stability guarantees of the entire KOS-TL system are shown in Table~\ref{tab:KOSHiera}.
\begin{table}[htbp]
\centering
\caption{System Hierarchical Structure and Stability Guarantees}
\label{tab:KOSHiera}
\begin{tabular}{llll} 
\toprule 
Layer & Core Focus & Core Operators & Stability Guarantee \\
\midrule 
L0: Core & Static Structure of Knowledge & $\Pi, \Sigma, \textsf{Prop}$ & Type Checking (Type Checking) \\
L1: Kernel & Event-Based State Transitions & $\textsf{STEP}, \textsf{Ev}$ & Proof Verification (Proof Verification) \\
L2: Runtime & Mapping Between Reality and Logic & $\textsf{elaborate}, \mathcal{M} $ & Transactional Consistency (Transactional Consistency) \\
\bottomrule 
\end{tabular}
\end{table}

\section{KOS-TL Core Layer: Static Logical Foundation}

The Core Layer is the "formal constitution" of the KOS-TL system, based on Intuitive Dependent Type Theory . Its core task is to define the static structure of knowledge, logical constraints, and their validity proofs, providing an unalterable logical foundation for upper-layer execution. It does not change with time and is only responsible for defining what constitutes a "legal construction".

\subsection{Syntax}

\subsubsection{Domain ($\mathcal{D}_{Core}$)}

The domain ($\mathcal{D}_{Core}$) world of the Core Layer consists of a dual-axis structure: one is the data axis, and the other is the logical axis.

\begin{itemize}
    \item \textbf{Dual-Axis Universes}:
    \begin{itemize}
        \item \textbf{Computational Axis ($\mathcal{U}_i$)}: Follows predicativity. $\mathcal{U}_0$ contains base sorts (Sorts), and $\mathcal{U}_{i+1}$ contains $\mathcal{U}_i$ as an element. Used for modeling data with physical effects.
        \item \textbf{Logical Axis ($\textsf{Type}_i$)}: Follows predicativity, but its base $\textsf{Prop} : \textsf{Type}_1$ has impredicativity. $\textsf{Type}_i$ is used for modeling logical predicate spaces and metalogical rules.
        \item \textbf{Hierarchical Relations}:
        \begin{itemize}
            \item $\textsf{Prop} : \textsf{Type}_1, \quad \textsf{Type}_i : \textsf{Type}_{i+1}$
            \item $\mathcal{U}_i : \mathcal{U}_{i+1}, \quad \mathcal{U}_i : \textsf{Type}_{i+1}$ (Computational universes can be objects of logical discussion)
            \item $\textsf{Prop} \hookrightarrow \mathcal{U}_1$ (Propositions can be embedded into the data axis)
        \end{itemize}
    \end{itemize}
    \item \textbf{Base Sorts }: $\textsf{Val}$ (atomic values), $\textsf{Time}$ (time point scalars), $\textsf{ID}$ (unique identifiers).
    \item \textbf{Knowledge Objects }: Instances of all dependent record types.
\end{itemize}

\subsubsection{Syntax}

\begin{itemize}
\item \textbf{Type Constructions (Types)}:
$$\begin{aligned}
A, B ::= & \ \textsf{Prop} \mid \textsf{Type}_i \mid \mathcal{U}_i & (\text{Universes}) \\
& \mid \textsf{Val} \mid \textsf{Time} \mid \textsf{ID} & (\text{Base Sorts}) \\
& \mid \Pi(x:A).B \mid \Sigma(x:A).B \mid A + B \mid \textsf{Id}_A(a, b) & (\text{Constructors})
\end{aligned}$$
\item \textbf{Term Constructions (Terms)}:
$$
\begin{aligned}
t, u ::= & \ x \mid \lambda x.t \mid t\,u & (\Pi\text{ Intro/Elim}) \\
& \mid \langle t, u \rangle \mid \textsf{split}(t, x.y.u) & (\Sigma\text{ Intro/Elim}) \\
& \mid \textsf{inl}(t) \mid \textsf{inr}(t) \mid \textsf{case}(t, x.u, y.v) & (+\text{ Intro/Elim}) \\
& \mid \textsf{refl} & (\text{Id Intro})
\end{aligned}
$$
\item \textbf{Judgments}:
\begin{itemize}
    \item $\Gamma \vdash A : \mathcal{S}$: Indicates that $A$ is a valid type, where $\mathcal{S} \in \{ \textsf{Type}_i, \mathcal{U}_i \}$.
    \item $\Gamma \vdash t : A$: Indicates that $t$ is a valid instance of type $A$ (for the data axis) or a valid proof (for the logical axis).
    \item $\Gamma \vdash A \equiv B$ and $\Gamma \vdash t \equiv u$: Indicates that types or terms are computationally equivalent (conversion rules).
\end{itemize}
\end{itemize}

To maintain the simplicity of the Core Layer, logical operations in $\textsf{Prop}$ are implemented through type construction operators in $\mathcal{U}$. The specific semantic mapping relations are shown in Table~\ref{tab:correspondenceofpro2type}.

Regarding the special nature of $\textsf{Prop}$. The KOS-TL Core follows the \textbf{Proof Irrelevance} principle: for any $P : \textsf{Prop}$, if $p, q : P$, then semantically $\llbracket p \rrbracket = \llbracket q \rrbracket$.

\begin{table}[h]
\centering
\caption{Isomorphism Between Logical Propositions and Types}
\label{tab:correspondenceofpro2type}
\begin{tabular}{|l|l|l|}
\hline
\textbf{Logical Proposition ($\textsf{Prop}$)} & \textbf{Type Construction ($\mathcal{T}$)} & \textbf{Term Construction (Terms)} \\
\hline
Universal Quantifier $\forall x:A.\, P(x)$ & Dependent Product $\Pi(x:A).P$ & $\lambda x.\, p$ \\
\hline
Existential Quantifier $\exists x:A.\, P(x)$ & Dependent Sum $\Sigma(x:A).P$ & $\langle a, p \rangle$ \\
\hline
Logical Implication $P \to Q$ & Function Space $P \to Q$ & $\lambda p.\, q$ \\
\hline
Logical Conjunction $P \wedge Q$ & Product Type $P \times Q$ & $\langle p, q \rangle$ \\
\hline
Logical Disjunction $P \vee Q$ & Sum Type $P + Q$ & $\textsf{inl}(p) / \textsf{inr}(q)$ \\
\hline
\end{tabular}
\end{table}

According to the construction rules for terms and types, the type set $\mathcal{T}$ of the KOS-TL Core Layer is inductively defined by the following rules.

\begin{enumerate}[label=(\arabic*)]
    \item Base Rules
    $$\begin{aligned}
         \textsf{Prop} &: \textsf{Type}_1 \quad (\text{Logical axis starting point}) \\
         \textsf{Type}_i &: \textsf{Type}_{i+1} \quad (\text{Logical universe accumulation}) \\
         \mathcal{U}_i &: \mathcal{U}_{i+1} \quad (\text{Data universe accumulation}) \\
         \textsf{Prop} & \hookrightarrow \mathcal{U}_1 \quad (\text{Lifting rule: Propositions can be treated as data processing, which is an implicit coercion})
    \end{aligned}$$
    $\textsf{Prop} \hookrightarrow \mathcal{U}_1$ is a one-way embedding, allowing proofs to be embedded as objects into data records (such as $\Sigma$ types), but ordinary data in $\mathcal{U}_i$ cannot be directly used as propositions for logical derivation.
    Atomic types: $\textsf{Val} \in \mathcal{T}, \textsf{Time} \in \mathcal{T}, \textsf{ID} \in \mathcal{T}$.
    \item Dependent Product Construction

    \begin{itemize}
        \item Logical/Computational Hybrid Rule:
        $$\frac{\Gamma \vdash A : \textsf{Type}_i/\mathcal{U}_i \quad \Gamma, x:A \vdash B : \textsf{Prop}}{\Gamma \vdash \Pi(x:A).B : \textsf{Prop}} (\text{Impredicative})$$
        \item Pure Universe Rule:
        $$\frac{\Gamma \vdash A : \textsf{Type}_i \quad \Gamma, x:A \vdash B : \textsf{Type}_j}{\Gamma \vdash \Pi(x:A).B : \textsf{Type}_{\max(i, j)}} (\text{Predicative})$$
    \end{itemize}
    $\textsf{Prop}$ has a special property called impredicativity (Impredicativity). Regardless of how high the level of $A$ is, as long as $B$ belongs to $\textsf{Prop}$, then $\Pi(x:A).B$ usually still belongs to $\textsf{Prop}$. Meaning: This allows us to perform logical judgments on ``infinite objects.'' For example, ``for all types in $\mathcal{U}_1$, they all satisfy the safety property $P$,'' this judgment itself is still just a simple $\textsf{Prop}$ (true or false), without exploding into a super-complex type.
    \item Dependent Sum Construction

    To prevent logical paradoxes (similar to Girard's Paradox), $\Sigma$ types in KOS-TL must be predicative. If $A$ is a type, and under the assumption of variable $x:A$, $B$ is a type, then:
    $$\frac{\Gamma \vdash A : \mathcal{U}_i \quad \Gamma, x:A \vdash B : \mathcal{U}_j}{\Gamma \vdash \Sigma(x:A).B : \mathcal{U}_{\max(i, j)}}$$
    Note: If $A, B \in \mathcal{U}$, the result is in $\mathcal{U}$; if proof extraction is involved, the highest level is constrained by the logical axis Universe.
    The dependent sum construction models knowledge objects. The $\Sigma$ type is the core of KOS-TL, forcing data $x$ to be associated with a proof term $p:B(x)$.
    \item Sum Type Construction

    If $A$ and $B$ are valid types respectively, then their disjunctive sum (disjoint union) is also a type:
$$\frac{\Gamma \vdash A : \mathcal{U} \quad \Gamma \vdash B : \mathcal{U}}{\Gamma \vdash A + B : \mathcal{U}}$$
  The sum type models the ``disjunction ($\vee$)'' relation in logic. In manufacturing scenarios, it is used to model ``mutually exclusive states'' or ``alternative paths.'' For example, a task's state is either $\textsf{Success}$ or $\textsf{Failure}$.
    \item Identity Type Construction

    If $A$ is a type, and $u, v$ are two terms of type $A$, then:
    $$\frac{\Gamma \vdash A : \mathcal{U} \quad \Gamma \vdash u:A \quad \Gamma \vdash v:A}{\Gamma \vdash \textsf{Id}_A(u, v) : \textsf{Prop}}$$
    The identity type construction models the equivalence of knowledge, serving as the logical foundation for judging whether two facts are consistent during "causal tracing" and "state rollback".
\end{enumerate}

In the KOS-TL Core, $\textsf{Prop}$ is a special domain dedicated to handling logical assertions. Unlike ordinary $\mathcal{U}_i$, it exhibits impredicativity under the $\Pi$ construction.

\begin{Definition}{Impredicative $\Pi$ Construction Rule}

For any level type $A : \mathcal{U}_i$, if under the assumption $x:A$, $B$ is a proposition, then its universal quantifier (or function space) still maps back to the smallest proposition world:
$$\frac{\Gamma \vdash A : \mathcal{U}_i \quad \Gamma, x:A \vdash B : \textsf{Prop}}{\Gamma \vdash \Pi(x:A).B : \textsf{Prop}}$$
\textbf{Logical Closure Point}: This means that the complexity of propositions does not increase in level with the expansion of their quantifier scope. This property allows us to make consistency assertions over full data (even objects at the $\mathcal{U}_k$ level) without triggering Universe explosion.
\end{Definition}

\begin{Definition}{Universe Lifting and Inclusion Rules}

To support the closure of dual-axis semantics, the system introduces the following implicit conversions:
\begin{itemize}
    \item Observation from Computation to Logic:
    $$\frac{\Gamma \vdash A : \mathcal{U}_i}{\Gamma \vdash A : \textsf{Type}_{i+1}}$$
    This means any computational type can be treated as an object of discussion for logical propositions (e.g., discussing the algebraic properties of SensorData at the $\textsf{Type}$ level).
    \item Computational Embedding of Propositions:
    $$\frac{\Gamma \vdash P : \textsf{Prop}}{\Gamma \vdash P : \mathcal{U}_1}$$
    This allows logical proof terms $p:P$ to be packaged into $\Sigma$ records as inputs to real-time computational systems (i.e., ``data packets with proofs'').
\end{itemize}
\end{Definition}

\subsubsection{Judgmental Rules}

To ensure the above constructions are logically well-formed, the KOS-TL Core follows the following derivation rules.

\begin{enumerate}[label=(\arabic*)]
    \item Dependent Product ($\Pi$-Types)
    \begin{itemize}
        \item Introduction Rules

        For $\Pi$ types, the construction (introduction) rule is:
        $$\frac{\Gamma, x:A \vdash t : B}{\Gamma \vdash \lambda x:A. t : \Pi(x:A).B}$$
        In the current context $\Gamma$, if we assume a variable $x$ of type $A$ and can construct a term $t$ of type $B$, then we can construct a $\lambda$ abstraction (i.e., function) whose type is $\Pi(x:A).B$.
        \item Elimination Rules

        For $\Pi$ types, there is a general rule $f$ (of type $\Pi(x:A).B$) and a concrete object $a$ (of type $A$). Applying $f$ to $a$ (denoted $f\,a$) yields a result of type $B[a/x]$. In the result type, all $x$ are replaced by the concrete value $a$.
        $$\frac{\Gamma \vdash f : \Pi(x:A).B \quad \Gamma \vdash a : A}{\Gamma \vdash f\,a : B[a/x]}$$
    \end{itemize}
    \item Dependent Sum ($\Sigma$-Types)
    \begin{itemize}
        \item Introduction Rules

        For $\Sigma$ types, the construction (introduction) rule is:
        $$\frac{\Gamma \vdash a : A \quad \Gamma \vdash b : B[a/x]}{\Gamma \vdash \langle a, b \rangle : \Sigma(x:A).B}$$
        The introduction rule embodies ``construction as proof'': only when you can provide evidence $b$ satisfying $B(a)$ can the knowledge object be created.
        \item Elimination Rules

        For $\Sigma$ types, we define a general dependent elimination operator $\textsf{split}$. It allows constructing a target term dependent on the overall pair by pattern matching on the pair structure.
        $$\frac{\Gamma \vdash p : \Sigma(x:A).B \quad \Gamma, x:A, y:B \vdash t : C[\langle x, y \rangle / z]}{\Gamma \vdash \textsf{split}(p, x.y.t) : C[p/z]}$$
        Traditional projection operators can be defined as special cases of $\textsf{split}$:
        \begin{align*}
            \textsf{proj}_1(p)\text{ (left projection)} & \equiv \textsf{split}(p, x.y.x) \\
            \textsf{proj}_2(p)\text{ (right projection)} &\equiv \textsf{split}(p, x.y.y)
        \end{align*}
    \end{itemize}
    \item Sum Types ($A + B$)
    \begin{itemize}
        \item Introduction Rules

        The introduction rules define how to create an object of type $A+B$. It has two branches, corresponding to "left choice" and "right choice".
        $$\frac{\Gamma \vdash a : A}{\Gamma \vdash \textsf{inl}(a) : A + B} \quad \quad \quad \frac{\Gamma \vdash b : B}{\Gamma \vdash \textsf{inr}(b) : A + B}$$
        If there is evidence $a$ of type $A$, it can be wrapped into type $A+B$ via the label $\textsf{inl}$ (In-Left). Similarly, if there is evidence $b$ of type $B$, it can be wrapped into $A+B$ via $\textsf{inr}$ (In-Right).
        \item Elimination Rules

        The elimination rule defines how to safely use an object of type $A+B$. Since it is unknown whether the interior is $A$ or $B$, two schemes must be prepared.
        $$\frac{\Gamma \vdash s : A + B \quad \Gamma, x:A \vdash t:C \quad \Gamma, y:B \vdash u:C}{\Gamma \vdash \textsf{case}(s, x.t, y.u) : C}$$
        $s : A + B$ is the input. $\Gamma, x:A \vdash t:C$ is scheme one. If $s$ is ultimately proven to be of type $A$, extract the data inside and give it to $x$, then compute a result according to logic $t$, with result type $C$. $\Gamma, y:B \vdash u:C$ is scheme two. If $s$ is of type $B$, give the data to $y$, and compute a result of type $C$ according to logic $u$. $\textsf{case}(s, x.t, y.u) : C$ indicates that regardless of which path $s$ takes, a deterministic result of type $C$ can ultimately be obtained.
    \end{itemize}
    \item Conversion Rule

    $$\frac{\Gamma \vdash t : A \quad \Gamma \vdash A \equiv B \quad \Gamma \vdash B : \mathcal{S}}{\Gamma \vdash t : B}$$
    This rule ensures that if term $t$ is valid in $\mathcal{U}_1$, and $\mathcal{U}_1 \hookrightarrow \textsf{Type}_2$ holds, then $t$ automatically has validity for observation at higher levels.
\end{enumerate}

\subsubsection{Reduction Rules}

We use $\to$ to denote one-step reduction (One-step reduction) and $\twoheadrightarrow$ to denote multi-step reduction (computational closure).
\begin{enumerate}[label=(\arabic*)]
\item Function Reduction ($\beta$-reduction)

For $\Pi$ type constructions ($\lambda$ abstraction):
$$\frac{}{\Gamma \vdash (\lambda x:A.t)\,u \to t[u/x]}$$
\item Dependent Record Reduction ($\iota$-reduction)

For structured elimination of $\Sigma$ types. This is the core correction point, directly deconstructing pairs via the \textsf{split} operator:
$$\frac{}{\Gamma \vdash \textsf{split}(\langle u, v \rangle, x.y.t) \to t[u/x, v/y]}$$
Under this definition, traditional projection reductions can be naturally derived as special cases:
\begin{itemize}
    \item Left Projection:
     $$\textsf{proj}_1(\langle u, v \rangle) \equiv \textsf{split}(\langle u, v \rangle, x.y.x) \to u$$
    \item Right Projection:
     $$\textsf{proj}_2(\langle u, v \rangle) \equiv \textsf{split}(\langle u, v \rangle, x.y.y) \to v$$
\end{itemize}
\item Sum Type Reduction ($\iota$-reduction)

For branch judgment of $+$ types. Matching labels via the $\textsf{case}$ operator:
$$\frac{}{\Gamma \vdash \textsf{case}(\textsf{inl}(u), x.t, y.v) \to t[u/x]} \quad \quad \quad \frac{}{\Gamma \vdash \textsf{case}(\textsf{inr}(w), x.t, y.v) \to v[w/y]}$$
\item Identity Term Reduction ($\iota$-reduction)

When the judgment term has been reduced to $\textsf{refl}$, the identity judgment automatically resolves.
\end{enumerate}

To support modular definitions and local variables in engineering practice, the system defines the following auxiliary conversion rules. Unlike core reductions ($\beta, \iota$), these conversions are typically triggered on demand during the type checker's \textbf{equivalence judgment (Conversion Check)} phase and are not counted in core logical steps.

\begin{itemize}
\item Global Unfolding ($\delta$-conversion):

  If term $c$ is defined as $c := t:A$ in context $\Gamma$, then unfolding is allowed during equivalence judgment: \\
  \[
  \frac{(c := t:A) \in \Gamma}{\Gamma \vdash c \equiv_\delta t}
  \]\\
  Meaning: Allows the system to recognize that aliases (Alias) and their original definitions are logically the same object.
\item Local Binding Unfolding ($\zeta$-conversion):

  For local definitions of the \texttt{let} structure, its semantics is equivalent to immediate substitution: \\
  \[
  \Gamma \vdash (\textsf{let } x = u \textsf{ in } t) \equiv_\zeta t[u/x]
  \]\\
  Meaning: Supports local reuse of terms without increasing the overhead of function calls ($\beta$).
\item Extensional Equivalence ($\eta$-conversion):

  To ensure function consistency, the system supports functional extensionality judgment: \\
  \[
  \Gamma \vdash \lambda x:A.(f\,x) \equiv_\eta f \quad (x \notin \text{FV}(f))
  \]\\
  Meaning: Ensures behavioral consistency between function abstractions and direct references, supporting functional programming paradigms.
\end{itemize}

$\delta$ and $\zeta$ ensure that the system has definitional transparency (Definitional Transparency), meaning that referencing names does not change the semantic essence of logical terms.

\begin{Definition}{Definitional Equality}

The judgmental equivalence relation $\equiv$ in KOS-TL is the smallest equivalence relation generated by all the above reductions ($\beta, \iota$) and conversions ($\delta, \zeta, \eta$) (satisfying symmetry, transitivity, and congruence).
\end{Definition}

\subsection{Logical Properties of the Core Layer}

\begin{Definition}{Normal Form}

A term $t$ is said to be in \textbf{normal form} (denoted $t \in \textsf{NF}$), if and only if it contains no \textbf{Redex} (reducible expressions). That is, for the reduction relation $\to$ defined in KOS-TL Core, there does not exist a term $t'$ such that:
\[
t \to t'
\]
\end{Definition}

\begin{Definition}{Strong Normalization ($\textsf{SN}$)}

A term $t$ is \textbf{strongly normalizing} if and only if there is no infinite reduction sequence starting from $t$. That is, all possible reduction paths $t \to t_1 \to t_2 \dots$ terminate in a finite number of steps at some normal form.
\end{Definition}

\begin{Definition}{Reducibility Candidate Set ( $\textsf{Red}_A$)}

For any type $A$ and term $t : A$, the reducibility candidate set $\textsf{Red}_A \subseteq \textsf{Val}_A$ is defined inductively on the structure of $A$ as follows:
\begin{enumerate}[label=(\arabic*)]
    \item Universe Type Cases
    \begin{itemize}
        \item $A \equiv \mathcal{U}_i$

        \[ t \in \textsf{Red}_{\mathcal{U}_i} \iff t \in \mathcal{RC} \land \textsf{level}(t) < i \]
        where $\mathcal{RC}$ is the set of all terms satisfying the CR properties (SN, stability, neutral term construction).
        \item $A \equiv \textsf{Type}_i$

        \[ t \in \textsf{Red}_{\textsf{Type}_i} \iff t \in \mathcal{RC} \land \textsf{level}(t) < i \]
        \item \textbf{Cross-Axis Constraints}:

        Since $\textsf{Prop} : \textsf{Type}_1$, then $\textsf{Prop} \in \textsf{Red}_{\textsf{Type}_1}$.

        Since $\textsf{Val} : \mathcal{U}_0$, then $\textsf{Val} \in \textsf{Red}_{\mathcal{U}_1}$.
    \end{itemize}
    \item Base Type Cases
    \begin{itemize}
        \item $A \equiv \textsf{Val}$

    $$t \in \textsf{Red}_A \iff t \in \textsf{SN} \land \exists c \in \textsf{Const}_\mathbb{N}.\ t \to^* c$$
        \item $A \equiv \textsf{Time}$

    $t \in \textsf{Red}_A \iff t \in \textsf{SN} \land t$ represents a valid timestamp or duration \\ (i.e., $t \to^* \mathsf{timestamp}(n)$ or $t \to^* \mathsf{duration}(n)$ for $n \in \mathbb{N}$)
    \end{itemize}
    \item Constructed Type Cases
    \begin{itemize}
        \item $A \equiv \Pi(x : B).C$

     $$t \in \textsf{Red}_A \iff \forall u \in \textsf{Red}_B.\ (t \ u) \in \textsf{Red}_{C[u/x]}$$
        \item $A \equiv \Sigma(x : B).C$

    $$t \in \textsf{Red}_{\Sigma(x:B).C} \iff t \twoheadrightarrow \langle u, v \rangle \land u \in \textsf{Red}_B \land v \in \textsf{Red}_{C[u/x]}$$
    $$t \in \textsf{Red}_A \iff t \in \textsf{SN} \land \left( t \to^* \mathsf{inl}(u) \implies u \in \textsf{Red}_B \right) \land \left( t \to^* \mathsf{inr}(v) \implies v \in \textsf{Red}_D \right)$$
        \item $A \equiv \textsf{Id}_B(a, b)$

     $$t \in \textsf{Red}_A \iff t \in \textsf{SN} \land \left( t \to^* \mathsf{refl}(w) \implies w \in \textsf{Red}_B \land a \equiv_B w \land b \equiv_B w \right)$$
    \end{itemize}
\end{enumerate}
where:
\begin{itemize}
\item $\textsf{Const}_\mathbb{N}$ is the set of natural number constants (numeric constants).
\item $\to^*$ denotes the multi-step closure of $\beta\eta$-reduction.
\item $\equiv_B$ denotes reducibility equivalence under type $B$: $p \equiv_B q \iff \exists v \in \textsf{Red}_B.\ (p \to^* v \land q \to^* v)$.
\item For $\textsf{Time}$, the precise semantics of "valid timestamp or duration" must be predefined (e.g., $\mathsf{timestamp}(n)$ represents Unix timestamp $n$, $\mathsf{duration}(n)$ represents $n$ milliseconds duration) to ensure formalization.
\end{itemize}
It should be noted that the interpretation of $\textsf{Prop}$ does not depend on its universe level.
\end{Definition}

\begin{Definition}{Neutral Term}
In the KOS-TL Core Layer, a term $t$ is a \textbf{neutral term} if and only if it satisfies one of the following two conditions:
\begin{enumerate}[label=(\arabic*)]
\item $t$ is a variable $x$.
\item The head of $t$ (Head) is a variable, and the term is in the process of being eliminated (Elimination) but cannot be further reduced.
\end{enumerate}
\end{Definition}

The reducibility set $\textsf{Red}_A$ of type $A$ must satisfy three key saturation properties (Saturation Properties), commonly known as $CR1, CR2, CR3$.

\begin{Definition}{Saturation Properties of Reducibility Candidate Set $\textsf{Red}_A$}
\begin{enumerate}[label=(\arabic*)]
    \item \textbf{CR 1 (Inclusivity)}

    If $t \in \textsf{Red}_A$, then $t \in \textsf{SN}$ (i.e., $t$ must first be strongly normalizing).
    \item \textbf{CR 2 (Stability)}

    If $t \in \textsf{Red}_A$ and $t \to t'$, then $t' \in \textsf{Red}_A$.
    \item \textbf{CR 3 (Neutral Term Construction)}

    If $t$ is a neutral term, and all one-step reduction terms $t'$ of $t$ are in $\textsf{Red}_A$, then $t \in \textsf{Red}_A$. All variables are neutral terms.
\end{enumerate}
\end{Definition}

\begin{Lemma}{Substitution Lemma}

\label{lemma:KOSsubstitution}
If $\Gamma, x:B, \Delta \vdash t : A$ and $\Gamma \vdash u : B$, then $\Gamma, \Delta[u/x] \vdash t[u/x] : A[u/x]$.
Here, $\Delta$ is a general context to handle variables defined after $x$ that depend on $x$. In simple cases, $\Delta$ is empty.
\end{Lemma}

\begin{proof}
Prove by induction on the structure of the type derivation tree. Discuss the following core cases based on the construction rule of $t$:
\begin{enumerate}[label=(\arabic*)]
\item \textbf{Variable Case (Variable)}\\
Assume $t$ is a variable $y$.
\begin{itemize}
\item \textbf{Subcase 1:} $y = x$\\
  According to the derivation rule, $A = B$ at this point.\\
  We need to prove $\Gamma, \Delta[u/x] \vdash x[u/x] : B[u/x]$.\\
  Since $x[u/x] = u$, and the premise is $\Gamma \vdash u : B$. According to the context weakening rule (Weakening), $\Delta[u/x]$ can be added after $\Gamma$, so the conclusion holds.
\item \textbf{Subcase 2:} $y \neq x$\\
  At this point, $y$ must be defined in $\Gamma$ or $\Delta$.\\
  If $y \in \Gamma$, then $y[u/x] = y$ and $A[u/x] = A$ (since types in $\Gamma$ do not depend on $x$), the conclusion is obvious.\\
  If $y \in \Delta$, then $y[u/x] = y$, its type is $A[u/x]$, which is exactly the corresponding declaration in $\Delta[u/x]$.
\end{itemize}
\item \textbf{$\Pi$-Type Introduction ($\lambda$-Abstraction)}\\
Assume $t = \lambda y:C. M$, and $A = \Pi(y:C). D$.\\
The last derivation step is
\[
\frac{\Gamma, x:B, \Delta, y:C \vdash M : D}{\Gamma, x:B, \Delta \vdash \lambda y:C. M : \Pi(y:C). D}
\]
\begin{itemize}
\item Apply the induction hypothesis: Use the induction hypothesis on $M$ (current context is $\Delta, y:C$):
  \[
  \Gamma, \Delta[u/x], y:C[u/x] \vdash M[u/x] : D[u/x]
  \]
\item Construct the conclusion: Apply the $\Pi$-introduction rule:
  \[
  \Gamma, \Delta[u/x] \vdash \lambda y:C[u/x]. M[u/x] : \Pi(y:C[u/x]). D[u/x]
  \]
  This is equivalent to $(\lambda y:C. M)[u/x] : (\Pi(y:C). D)[u/x]$.
\end{itemize}
\item \textbf{$\Pi$-Type Elimination (Application)}\\
Assume $t = (f \, v)$, where $\Gamma, x:B, \Delta \vdash f : \Pi(y:C). D$ and $\Gamma, x:B, \Delta \vdash v : C$.
\begin{itemize}
\item Induction Hypothesis 1: $\Gamma, \Delta[u/x] \vdash f[u/x] : (\Pi(y:C). D)[u/x]$.
\item Induction Hypothesis 2: $\Gamma, \Delta[u/x] \vdash v[u/x] : C[u/x]$.
\item Combination: Apply the elimination rule:
  \[
  (f[u/x] \, v[u/x]) : D[u/x][v[u/x]/y]
  \]
  According to the commutativity of substitution, the above type is equivalent to $D[v/y][u/x]$, i.e., $A[u/x]$.
\end{itemize}
\item \textbf{$\Sigma$-Type Construction (Pairing)}
Assume $t = \langle t_1, t_2 \rangle$, and $A = \Sigma(y:C). D$.
\begin{itemize}
\item By the induction hypothesis, $t_1[u/x] : C[u/x]$.
\item By the induction hypothesis, $t_2[u/x] : D[t_1/y][u/x]$.
\item According to the rule, construct $\langle t_1[u/x], t_2[u/x] \rangle : (\Sigma(y:C). D)[u/x]$.
\end{itemize}
\item \textbf{$\Sigma$-Type Elimination (Structured Elimination)}
Assume $t = \textsf{split}(p, x.y.u)$, and the last derivation step is:
$$\frac{\Gamma, z:B, \Delta \vdash p : \Sigma(x:A).B \quad \Gamma, z:B, \Delta, x:A, y:B \vdash u : C(\langle x, y \rangle)}{\Gamma, z:B, \Delta \vdash \textsf{split}(p, x.y.u) : C(p)}$$
(Here $z:B$ is the variable we are substituting)
\begin{itemize}
    \item \textbf{Induction Hypothesis 1}: Apply the induction hypothesis to $p$, yielding $\Gamma, \Delta[v/z] \vdash p[v/z] : (\Sigma(x:A).B)[v/z]$.
    \item \textbf{Induction Hypothesis 2}: Apply the induction hypothesis to the elimination body $u$ (now with added context $x, y$):
    $$ \Gamma, \Delta[v/z], x:A[v/z], y:B[v/z] \vdash u[v/z] : C(\langle x, y \rangle)[v/z] $$
    \item \textbf{Construct the Conclusion}: Reapply the $\Sigma$-elimination rule ($\textsf{split}$ rule):
    $$ \Gamma, \Delta[v/z] \vdash \textsf{split}(p[v/z], x.y.u[v/z]) : C(p)[v/z] $$
    Since $\textsf{split}(p, x.y.u)[v/z] = \textsf{split}(p[v/z], x.y.u[v/z])$, the conclusion holds.
\end{itemize}
\item \textbf{$+$-Type Introduction (Injection)} \\
Assume $t = \textsf{inl}_D(s)$, and $A = C + D$ ($\textsf{inr}$ case symmetric).
\begin{itemize}
    \item \textbf{Premise}: Known that $\Gamma, x:B, \Delta \vdash s : C$ and $\Gamma, x:B, \Delta \vdash D : \mathcal{U}$.
    \item \textbf{Induction Hypothesis}: For $s$, $s[u/x] : C[u/x]$; for type $D$, $D[u/x] : \mathcal{U}$.
    \item \textbf{Construct the Conclusion}: Apply the $+$-introduction rule:
    $$ \Gamma, \Delta[u/x] \vdash \textsf{inl}_{D[u/x]}(s[u/x]) : C[u/x] + D[u/x] $$
    i.e., $(\textsf{inl}_D(s))[u/x] : (C + D)[u/x]$.
\end{itemize}
\item \textbf{$+$-Type Elimination (Branch Judgment)}\\
Assume $t = \textsf{case}(s, y.t_1, z.t_2)$, and the last derivation step is:
$$\frac{\Gamma', s:C+D \quad \Gamma', y:C \vdash t_1 : A \quad \Gamma', z:D \vdash t_2 : A}{\Gamma' \vdash \textsf{case}(s, y.t_1, z.t_2) : A}$$
(where $\Gamma'$ is abbreviated as $\Gamma, x:B, \Delta$)
\begin{itemize}
    \item \textbf{Induction Hypothesis 1}: For the judgment term $s$, $\Gamma, \Delta[u/x] \vdash s[u/x] : C[u/x] + D[u/x]$.
    \item \textbf{Induction Hypothesis 2}: For the left branch $t_1$ (now with added context $y:C$), $\Gamma, \Delta[u/x], y:C[u/x] \vdash t_1[u/x] : A[u/x]$.
    \item \textbf{Induction Hypothesis 3}: For the right branch $t_2$ (now with added context $z:D$), $\Gamma, \Delta[u/x], z:D[u/x] \vdash t_2[u/x] : A[u/x]$.
    \item \textbf{Combination}: Apply the $+$-elimination rule ($\textsf{case}$ rule):
    $$ \Gamma, \Delta[u/x] \vdash \textsf{case}(s[u/x], y.t_1[u/x], z.t_2[u/x]) : A[u/x] $$
    The conclusion holds.
\end{itemize}
\item \textbf{Identity Type (Identity)}\\
If $t = \textsf{refl}_a$, then $A = \textsf{Id}_C(a, a)$.
\begin{itemize}
\item By the induction hypothesis, $a[u/x] : C[u/x]$.
\item Directly apply the construction rule to obtain $\textsf{refl}_{a[u/x]} : \textsf{Id}_{C[u/x]}(a[u/x], a[u/x])$, i.e., $A[u/x]$.
\end{itemize}
\end{enumerate}
\end{proof}

\begin{Lemma}{Fundamental Lemma of Reducibility}
\label{lemma:KOSreduction}

Let $\Gamma = \{x_1:A_1, \dots, x_n:A_n\}$ be a well-formed context. If $\Gamma \vdash t : C$, and there exists a reducible substitution $\gamma = [u_1/x_1, \dots, u_n/x_n]$ such that for all $i$, $u_i \in \textsf{Red}_{A_i}$.

Then the substituted term $t[\gamma]$ must satisfy:
\[
t[\gamma] \in \textsf{Red}_C
\]
\end{Lemma}

\begin{proof}
Prove by induction on the structure of the type derivation tree.
\begin{enumerate}[label=(\arabic*)]
    \item \textbf{Variable Rules (Variables)}

    If the derivation is $\Gamma \vdash x_i : A_i$.
    According to the substitution definition, $x_i[\gamma] = u_i$. By the premise $u_i \in \textsf{Red}_{A_i}$, the proposition obviously holds.
    \item \textbf{$\Pi$-Type Introduction ($\lambda$-Abstraction)}
    If $\Gamma \vdash \lambda x:A.M : \Pi(x:A).B$.\\
    Need to prove: For any $u \in \textsf{Red}_A$, $((\lambda x:A.M)[\gamma] \, u) \in \textsf{Red}_{B[u/x]}$.
    \begin{enumerate}[label=(\alph*)]
        \item This term $\beta$-reduces to $M[\gamma, u/x]$.
        \item Since $u \in \textsf{Red}_A$ and $\gamma \in \textsf{Red}_\Gamma$, then $(\gamma, u/x)$ is a reducible substitution under the context $(\Gamma, x:A)$.
        \item By the induction hypothesis, $M[\gamma, u/x] \in \textsf{Red}_B$.
        \item Since the reducibility set is closed under reverse reduction, $(\lambda x:A.M)[\gamma] \in \textsf{Red}_{\Pi(x:A).B}$.
    \end{enumerate}
    \item \textbf{$\Pi$-Type Elimination (Application)}

    If $\Gamma \vdash (f \, v) : C$, where $C = B[v/x]$ and $\Gamma \vdash f : \Pi(x:A).B, \Gamma \vdash v : A$.
    \begin{enumerate}[label=(\alph*)]
        \item By the induction hypothesis, $f[\gamma] \in \textsf{Red}_{\Pi(x:A).B}$.
        \item By the induction hypothesis, $v[\gamma] \in \textsf{Red}_A$.
        \item According to the definition of $\textsf{Red}_{\Pi}$: If a term belongs to the reducibility set of a $\Pi$ type, then its application to any term in the reducibility set of the parameter type must belong to the reducibility set of the result type.
        \item Therefore, $(f[\gamma] \, v[\gamma]) \in \textsf{Red}_{B[v[\gamma]/x]}$.
        \item Since $(f[\gamma] \, v[\gamma]) = (f \, v)[\gamma]$ and $B[v[\gamma]/x] = C[\gamma]$, the conclusion holds.
    \end{enumerate}
    \item \textbf{$\Sigma$-Type Introduction (Pairing)}

    If $\Gamma \vdash \langle a, b \rangle : \Sigma(x:A).B$.
    \begin{enumerate}[label=(\alph*)]
        \item By the induction hypothesis, $a[\gamma] \in \textsf{Red}_A$.
        \item By the induction hypothesis, $b[\gamma] \in \textsf{Red}_{B[a[\gamma]/x]}$.
        \item According to the definition of $\textsf{Red}_{\Sigma}$, if both components are reducible, then $\langle a, b \rangle[\gamma] \in \textsf{Red}_{\Sigma(x:A).B}$.
    \end{enumerate}
    \item \textbf{$\Sigma$-Type Elimination (split Operator)}

    If $\Gamma \vdash \textsf{split}(p, x.y.t) : C$, where $\Gamma \vdash p : \Sigma(x:A).B$.
    \begin{enumerate}[label=(\alph*)]
        \item \textbf{Induction Hypothesis}: By the induction hypothesis, $p[\gamma] \in \textsf{Red}_{\Sigma(x:A).B}$. This means $p[\gamma]$ is strongly normalizing and ultimately reduces to some pair $\langle u, v \rangle$, where $u \in \textsf{Red}_A, v \in \textsf{Red}_{B[u/x]}$.
        \item \textbf{Reduction Analysis}: According to $\iota$-reduction, $\textsf{split}(p[\gamma], x.y.t[\gamma]) \twoheadrightarrow t[\gamma, u/x, v/y]$.
        \item \textbf{Apply Induction Hypothesis}: Since $(\gamma, u/x, v/y)$ is a valid reducible substitution under the context $(\Gamma, x:A, y:B)$, applying the induction hypothesis to $t$ yields: $$ t[\gamma, u/x, v/y] \in \textsf{Red}_C $$
        \item \textbf{Closure}: Using the closure of the reducibility set under reverse $\iota$-reduction (a corollary of CR 3 property), the original term $\textsf{split}(p, x.y.t)[\gamma]$ also belongs to $\textsf{Red}_C$.
    \end{enumerate}
    \item \textbf{$+$-Type Introduction (Injection)}

    If $\Gamma \vdash \textsf{inl}(a) : A + B$ ($\textsf{inr}$ similarly).
    \begin{enumerate}[label=(\alph*)]
        \item By the induction hypothesis, $a[\gamma] \in \textsf{Red}_A$.
        \item According to the definition of $\textsf{Red}_{A+B}$ (typically defined by neutral terms and injection properties): Since $a[\gamma]$ is reducible, its injection term $\textsf{inl}(a[\gamma])$ is also reducible in the $+$-type reducibility system (using reverse reduction closure).
        \item Thus, $\textsf{inl}(a)[\gamma] \in \textsf{Red}_{A+B}$.
    \end{enumerate}
    \item \textbf{Sum Type Elimination (Case Analysis)}

    If $\Gamma \vdash \textsf{case}(t, x.M, y.N) : C$.
    \begin{enumerate}[label=(\alph*)]
        \item By the induction hypothesis, $t[\gamma] \in \textsf{Red}_{A+B}$.
        \item $t[\gamma]$ reduces to $\textsf{inl}(u)$ or $\textsf{inr}(v)$. Assume $\textsf{inl}(u)$, then $u \in \textsf{Red}_A$.
        \item At this point, the $\textsf{case}$ term reduces to $M[\gamma, u/x]$.
        \item By the induction hypothesis, $M[\gamma, u/x] \in \textsf{Red}_C$. Similarly for the $\textsf{inr}$ case.
    \end{enumerate}
    \item \textbf{Identity Type Introduction ($\textsf{refl}$)}

    If $\Gamma \vdash \textsf{refl}_a : \textsf{Id}_A(a, a)$.
    \begin{enumerate}
        \item By the induction hypothesis, $a[\gamma] \in \textsf{Red}_A$.
        \item Obviously $a[\gamma] \cong a[\gamma]$ and $\textsf{refl} \in \textsf{SN}$.
        \item Thus, $\textsf{refl}_{a[\gamma]} \in \textsf{Red}_{\textsf{Id}_A(a[\gamma], a[\gamma])}$.
    \end{enumerate}
    \item \textbf{$\delta$-Reduction and Local Definition ($\textsf{let}$)}

    If $\Gamma \vdash \textsf{let } x=u \textsf{ in } t : C$.
    \begin{enumerate}[label=(\alph*)]
        \item By the induction hypothesis, $u[\gamma] \in \textsf{Red}_A$.
        \item Construct the extended substitution $\gamma' = [\gamma, u[\gamma]/x]$. Since $u[\gamma]$ is reducible, $\gamma'$ is a reducible substitution under the well-formed context.
        \item By the induction hypothesis, $t[\gamma'] \in \textsf{Red}_C$.
        \item Since $(\textsf{let } x=u \textsf{ in } t)[\gamma] \to t[\gamma, u[\gamma]/x]$, according to the closure of the reducibility set under reverse $\zeta$-reduction, the conclusion holds.
    \end{enumerate}
\end{enumerate}
\end{proof}

\begin{Theorem}{Strong Normalization for KOS-TL Core}

Let $\Gamma$ be a well-formed context. If term $t$ satisfies $\Gamma \vdash t : A$, then $t$ is strongly normalizing (i.e., $t \in \textsf{SN}$).
This means any reduction sequence starting from $t$, $t \to t_1 \to t_2 \dots$, is finite.
\end{Theorem}

\begin{proof}
The proof references the Tait-Girard method~\cite{Girard1989}. The core logic is to use the fundamental lemma to convert "type validity" to "reducibility", and then leverage the property that reducible terms are strongly normalizing.

\textbf{Step 1: Introduce Identity Substitution}

For the context $\Gamma = \{x_1:A_1, \dots, x_n:A_n\}$, we construct a special substitution $\gamma_{id}$:
\[
\gamma_{id} = [x_1/x_1, \dots, x_n/x_n]
\]
To apply the fundamental lemma, we need to prove that $\gamma_{id}$ is a reducible substitution. This means that for every variable $x_i$, it must be proven to belong to the reducibility set, i.e., $x_i \in \textsf{Red}_{A_i}$.

\textbf{Step 2: Reducibility of Variables}

According to the properties of the reducibility candidate set (Girard's $\mathcal{RC}$), all $\textsf{Red}_A$ sets satisfy the following two key properties:
\begin{enumerate}
\item CR 1 (SN Inclusivity): If $t \in \textsf{Red}_A$, then $t \in \textsf{SN}$.
\item CR 3 (Neutral Term Property): If $t$ is a neutral term (i.e., a variable or application of a variable that cannot be further reduced) and all its one-step reduction terms are in $\textsf{Red}_A$, then $t \in \textsf{Red}_A$.
\end{enumerate}
Since the variable $x_i$ is a basic neutral term and has no reduction forms, by CR 3, it follows that $x_i \in \textsf{Red}_{A_i}$.
Thus, $\gamma_{id}$ satisfies the premise of the fundamental lemma.

\textbf{Step 3: Application of Fundamental Lemma}

Since the premise $\Gamma \vdash t : A$ holds, and $\gamma_{id}$ is a reducible substitution, by the fundamental lemma~\ref{lemma:KOSreduction}:
\[
t[\gamma_{id}] \in \textsf{Red}_A
\]
Since $t[\gamma_{id}]$ is syntactically equivalent to the term $t$ itself, we obtain:
\[
t \in \textsf{Red}_A
\]

\textbf{Step 4: Conclusion Derivation}

According to the property CR 1 of the reducibility set (all terms belonging to the reducibility set are strongly normalizing):
\[
t \in \textsf{Red}_A \implies t \in \textsf{SN}
\]
This completes the proof.
\end{proof}

\begin{Theorem}{Subject Reduction}

In KOS-TL Core, reduction operations do not change the type of a term. If $\Gamma \vdash t : A$ and $t \to t'$, then $\Gamma \vdash t' : A$.
\end{Theorem}

\begin{proof}
Based on the substitution lemma~\ref{lemma:KOSsubstitution}, the proof proceeds by structural induction as follows.
\begin{enumerate}[label=(\arabic*)]
    \item \textbf{Main Reduction Case: $\beta$-reduction}

    Consider the most basic reduction step $(\lambda x:B.t)u \to t[u/x]$:
    \begin{itemize}
        \item From the premise $\Gamma \vdash (\lambda x:B.t)u : A$, there must exist a type $B$ such that $\Gamma \vdash \lambda x:B.t : \Pi(x:B).A'$ and $\Gamma \vdash u : B$.
        \item According to the inversion of the $\Pi$-introduction rule, $\Gamma, x:B \vdash t : A'$.
        \item Applying the substitution lemma directly yields $\Gamma \vdash t[u/x] : A'[u/x]$.
    \end{itemize}
    \item \textbf{$\Sigma$-Type Reduction: $\iota$-reduction}

    Consider the case $\textsf{split}(\langle u, v \rangle, x.y.t) \to t[u/x, v/y]$:
    \begin{itemize}
        \item \textbf{Premise}: From $\Gamma \vdash \textsf{split}(\langle u, v \rangle, x.y.t) : C$, it follows that:
        \begin{enumerate}[label=(\alph*)]
            \item $\Gamma \vdash \langle u, v \rangle : \Sigma(x:A).B$
            \item $\Gamma, x:A, y:B \vdash t : C'$ (where $C$ is actually $C'[\langle u, v \rangle / z]$)
        \end{enumerate}
        \item \textbf{Derivation}: From (a), by inversion of the introduction rule, $\Gamma \vdash u : A$ and $\Gamma \vdash v : B[u/x]$.
        \item \textbf{Application}: Continuously apply the substitution lemma twice to $t$ (first substitute $x$, then $y$), directly obtaining: $$ \Gamma \vdash t[u/x, v/y] : C'[\langle u, v \rangle / z] $$ The type is preserved consistently, and the conclusion holds.
    \end{itemize}
    \item \textbf{Sum Type Reduction: $\iota$-reduction}

    Consider the case $\textsf{case}(\textsf{inl}(u), x.M, y.N) \to M[u/x]$:
    \begin{itemize}
    \item From the premise $\Gamma \vdash \textsf{case}(\textsf{inl}(u), x.M, y.N) : A$, it follows that:
    \begin{enumerate}[label=(\alph*)]
        \item $\Gamma \vdash \textsf{inl}(u) : B + C$
        \item $\Gamma, x:B \vdash M : A$ and $\Gamma, y:C \vdash N : A$.
    \end{enumerate}
    \item From (a), by inversion of the $+$-introduction rule, $\Gamma \vdash u : B$.
    \item Applying the substitution lemma (lemma~\ref{lemma:KOSsubstitution}), from $\Gamma, x:B \vdash M : A$ and $\Gamma \vdash u : B$, we obtain $\Gamma \vdash M[u/x] : A[u/x]$.
    \item If $A$ does not depend on the judgment term, then $A[u/x] = A$, and the conclusion holds. (For dependent type cases, substitution similarly preserves type consistency).
    \end{itemize}
    \item \textbf{Local Definition Reduction: $\zeta$-reduction}

    Consider the case $\textsf{let } x = u \textsf{ in } t \to t[u/x]$:
    \begin{itemize}
    \item From the premise $\Gamma \vdash \textsf{let } x = u \textsf{ in } t : A$, there must exist a type $B$ such that $\Gamma \vdash u : B$ and $\Gamma, x:B \vdash t : A$.
    \item This is a standard application scenario of the substitution lemma. According to lemma~\ref{lemma:KOSsubstitution}, we directly derive $\Gamma \vdash t[u/x] : A[u/x]$.
    \end{itemize}
    \item \textbf{Definition Unfolding Reduction: $\delta$-reduction}

    Consider the case $c \to t$, where $(c := t:A) \in \Gamma$:
    \begin{itemize}
    \item From the premise $\Gamma \vdash c : A$, $c$ is a constant declared in the context.
    \item According to the definition of $\delta$-reduction, the type of the identifier $c$ is completely consistent with the type of its definition body $t$ in $\Gamma$.
    \item Thus, by the well-formedness of $\Gamma$, we directly obtain $\Gamma \vdash t : A$.
    \end{itemize}
    \item \textbf{Congruence Cases}

    If the reduction occurs in a subterm, e.g., $t = f\,u \to f'\,u$ (where $f \to f'$):
    \begin{itemize}
    \item By the induction hypothesis, $f'$ preserves the type $\Pi(x:B).A$ of $f$.
    \item Reapply the $\Pi$-elimination rule; the overall term's type remains $A[u/x]$.
    \item Similarly, all other constructions (pairing, projection, injection, etc.) preserve types under congruence reductions.
\end{itemize}
\end{enumerate}
Since all basic computational reductions ($\beta, \iota$, etc.) satisfy type preservation, and the reduction relation $\to$ is closed under context construction, by structural induction on the reduction relation, the theorem holds for all reduction steps.
\end{proof}

We interpret the type $A$ as a set of terms $\llbracket A \rrbracket$. These sets must satisfy the reducibility candidate (CR) properties mentioned earlier.

\begin{Definition}{Type Semantics}
\begin{itemize}
\item $\llbracket \textsf{Val} \rrbracket = \{ t \mid t \in \textsf{SN} \land t$ ultimately reduces to a numeric constant $\}$
\item $\llbracket \textsf{Time} \rrbracket = \{ t \mid t \in \textsf{SN} \land t$ ultimately reduces to a valid timestamp $\}$
\item $\llbracket \Pi(x:A).B \rrbracket = \{ f \mid \forall u \in \llbracket A \rrbracket, (f \, u) \in \llbracket B \rrbracket [u/x] \}$
\item $\llbracket \Sigma(x:A).B \rrbracket = \{ p \mid p \twoheadrightarrow \langle u, v \rangle \land u \in \llbracket A \rrbracket \land v \in \llbracket B \rrbracket[u/x] \}$
\item $\llbracket A + B \rrbracket = \{ t \mid t \in \textsf{SN} \land (t \twoheadrightarrow \textsf{inl}(u) \Rightarrow u \in \llbracket A \rrbracket) \land (t \twoheadrightarrow \textsf{inr}(v) \Rightarrow v \in \llbracket B \rrbracket) \}$
\item $\llbracket \textsf{Id}_A(a, b) \rrbracket = \{ \textsf{refl} \mid a, b \in \llbracket A \rrbracket \land a \simeq_{Red} b \}$. Where $\simeq_{Red}$ indicates that they reduce to the same normal form.
\end{itemize}
\end{Definition}

The interpretation function $\llbracket t \rrbracket_\rho$ is responsible for converting syntactic terms with variables into their corresponding semantic values. In strong normalization proofs, this ``interpretation'' is typically the substitution (Substitution) operation.

\begin{Definition}{Semantic Interpretation of Terms}

Let $\rho$ be a mapping from variables to semantic values (assignment).
\begin{itemize}
\item $\llbracket x \rrbracket_\rho = \rho(x)$ (directly read the assignment from the environment)
\item $\llbracket \lambda x:A. M \rrbracket_\rho = \text{a function } v \mapsto \llbracket M \rrbracket_{\rho[x \mapsto v]}$
\item $\llbracket f \, a \rrbracket_\rho = \llbracket f \rrbracket_\rho (\llbracket a \rrbracket_\rho)$ (function application)
\item $\llbracket \langle a, b \rangle \rrbracket_\rho = (\llbracket a \rrbracket_\rho, \llbracket b \rrbracket_\rho)$ (semantic pairing)
\item $\llbracket \textsf{refl} \rrbracket_\rho = \textsf{refl}$ (constant interpreted as itself)
\end{itemize}
\end{Definition}

To ensure $\llbracket t \rrbracket_\rho \in \llbracket A \rrbracket$, the assignment $\rho$ must be ``valid.''

\begin{Definition}{Logical Closure of Prop Semantics}
In the semantic model $\mathcal{M}$, the interpretation $\llbracket \textsf{Prop} \rrbracket$ of $\textsf{Prop}$ is defined as the collection of all term sets $S$ satisfying the following properties:
\begin{enumerate}[label=(\arabic*)]
    \item \textbf{SN Property}: $S \subseteq \textsf{SN}$.
    \item \textbf{CR Property}: $S$ is closed under reduction and reverse reduction, and contains all neutral terms.
    \item \textbf{Impredicative Semantic Operator}: For any set $X$ and function $F: X \to \llbracket \textsf{Prop} \rrbracket$, the intersection operation is defined as:
    $$ \llbracket \Pi(x:A).B \rrbracket_\rho = \bigcap_{u \in \llbracket A \rrbracket_\rho} \{ f \mid (f \, u) \in \llbracket B \rrbracket_{\rho[x \mapsto u]} \} $$
\end{enumerate}
\textbf{Closure Correction}: Since $\llbracket \textsf{Prop} \rrbracket$ includes all $\Pi$-type interpretations it constructs itself (via intersection operations on the reducibility candidate $\mathcal{RC}$), this guarantees that even if $A$ is an infinitely large Universe, the mapped result remains within the predefined set of $\llbracket \textsf{Prop} \rrbracket$.
\end{Definition}

\begin{Definition}{Valid Assignment}

An assignment $\rho$ is said to satisfy the context $\Gamma$ (denoted $\rho \models \Gamma$), if and only if for every binding $(x:A)$ in $\Gamma$,
\[
\rho(x) \in \llbracket A \rrbracket_\rho
\]
\end{Definition}

\begin{Theorem}{Semantic Soundness}
\label{Theorem:KosSoundness}

\centerline{If $\Gamma \vdash t : A$ and $\rho \models \Gamma$, then $\llbracket t \rrbracket_\rho \in \llbracket A \rrbracket_\rho$.}
\end{Theorem}

\begin{proof}
Proceed by induction on the structure of the derivation tree $\Gamma \vdash t : A$. Due to the impredicativity of $\textsf{Prop}$, the construction of its reducibility candidates ($\mathcal{RC}$) is based on Girard's stratified candidate set method, rather than simple Tarski-style semantics.

\textbf{A. Variable Rule (Variable)}

Assume the derivation is
\[
\frac{(x:A) \in \Gamma}{\Gamma \vdash x : A}
\]
\begin{itemize}
\item Proof: From the premise $\rho \models \Gamma$, the assignment for each bound variable in the environment must belong to the interpretation of that type. Thus, $\rho(x) \in \llbracket A \rrbracket_\rho$.
\item Since $\llbracket x \rrbracket_\rho = \rho(x)$, the conclusion $\llbracket x \rrbracket_\rho \in \llbracket A \rrbracket_\rho$ holds.
\end{itemize}

\textbf{B. $\Pi$-Type Introduction ($\lambda$-Abstraction)}

Assume the last derivation step is
\[
\frac{\Gamma, x:A \vdash M : B}{\Gamma \vdash \lambda x:A. M : \Pi(x:A). B}
\]
\begin{itemize}
\item Goal: Prove $\llbracket \lambda x:A. M \rrbracket_\rho \in \llbracket \Pi(x:A). B \rrbracket_\rho$.
\item According to the semantics of $\Pi$, need to prove: For any $u \in \llbracket A \rrbracket_\rho$, $(\llbracket \lambda x:A. M \rrbracket_\rho \cdot u) \in \llbracket B \rrbracket_{\rho[x \mapsto u]}$.
\item By the term interpretation definition, $(\llbracket \lambda x:A. M \rrbracket_\rho \cdot u) \to_\beta \llbracket M \rrbracket_{\rho[x \mapsto u]}$.
\item Since $u \in \llbracket A \rrbracket_\rho$ and $\rho \models \Gamma$, the new assignment $\rho' = \rho[x \mapsto u]$ satisfies $\rho' \models (\Gamma, x:A)$.
\item Apply the induction hypothesis: $\llbracket M \rrbracket_{\rho'} \in \llbracket B \rrbracket_{\rho'}$.
\item Since the reducibility set is closed under reverse reduction (CR property), the original application term also belongs to that set.
\end{itemize}

\textbf{C. $\Pi$-Type Elimination (Application)}

Assume the last derivation step is the application rule:
\[
\frac{\Gamma \vdash f : \Pi(x:A).B \quad \Gamma \vdash u : A}{\Gamma \vdash f \, u : B[u/x]}
\]
\begin{itemize}
\item Proof: By the induction hypothesis, $\llbracket f \rrbracket_\rho \in \llbracket \Pi(x:A).B \rrbracket_\rho$ and $\llbracket u \rrbracket_\rho \in \llbracket A \rrbracket_\rho$.
\item According to the semantics of the $\Pi$ type, the function $\llbracket f \rrbracket_\rho$ applied to any element in $\llbracket A \rrbracket_\rho$ must belong to the interpretation of $B$.
\item Therefore, $\llbracket f \rrbracket_\rho(\llbracket u \rrbracket_\rho) \in \llbracket B \rrbracket_{\rho[x \mapsto \llbracket u \rrbracket_\rho]}$.
\item By the substitution lemma, this set is $\llbracket B[u/x] \rrbracket_\rho$.
\end{itemize}

\textbf{D. $\Sigma$-Type Introduction (Pairing)}

Assume the derivation is:
\[
\frac{\Gamma \vdash a : A \quad \Gamma \vdash b : B[a/x]}{\Gamma \vdash \langle a, b \rangle : \Sigma(x:A). B}
\]
\begin{enumerate}[label=(\arabic*)]
\item By the induction hypothesis, $\llbracket a \rrbracket_\rho \in \llbracket A \rrbracket_\rho$.
\item By the induction hypothesis, $\llbracket b \rrbracket_\rho \in \llbracket B[a/x] \rrbracket_\rho$.
\item According to the semantics of $\Sigma$: $\llbracket \Sigma(x:A). B \rrbracket_\rho = \{ (u, v) \mid u \in \llbracket A \rrbracket_\rho, v \in \llbracket B \rrbracket_{\rho[x \mapsto u]} \}$.
\item Combining with the substitution lemma $\llbracket B[a/x] \rrbracket_\rho = \llbracket B \rrbracket_{\rho[x \mapsto \llbracket a \rrbracket_\rho]}$, the two components of $\llbracket \langle a, b \rangle \rrbracket_\rho$ fully conform to the definition.
\end{enumerate}

\textbf{E. Dependent Sum Elimination ($\Sigma$-Elimination / split)}

Assume the last derivation step applies the $\Sigma$ elimination rule:
\[
\frac{\Gamma \vdash p : \Sigma(x:A).B \quad \Gamma, x:A, y:B \vdash t : C(\langle x, y \rangle)}{\Gamma \vdash \textsf{split}(p, x.y.t) : C(p)}
\]
We need to prove: If $\rho \models \Gamma$, then $\llbracket \textsf{split}(p, x.y.t) \rrbracket_\rho \in \llbracket C(p) \rrbracket_\rho$.
\begin{enumerate}[label=(\arabic*)]
\item \textbf{Semantic Premise Derivation} \\
By the induction hypothesis (Inductive Hypothesis):\\
\begin{itemize}
\item For term $p$, $\llbracket p \rrbracket_\rho \in \llbracket \Sigma(x:A).B \rrbracket_\rho$.
\item According to the semantics of the $\Sigma$ type, there exist $u \in \llbracket A \rrbracket_\rho$ and $v \in \llbracket B \rrbracket_{\rho[x \mapsto u]}$, such that $\llbracket p \rrbracket_\rho$ is semantically equivalent to the pair $(u, v)$.
\end{itemize}
\item \textbf{Construct Valid Assignment} \\
Define a new assignment $\rho' = \rho[x \mapsto u, y \mapsto v]$.\\
\begin{itemize}
\item Since $u \in \llbracket A \rrbracket_\rho$, $\rho[x \mapsto u] \models (\Gamma, x:A)$.
\item Since $v \in \llbracket B \rrbracket_{\rho[x \mapsto u]}$, and $\rho[x \mapsto u]$ satisfies the preceding context, $\rho' \models (\Gamma, x:A, y:B)$.
\end{itemize}
\item \textbf{Apply Induction Hypothesis} \\
Apply the induction hypothesis to the elimination body $t$:
\[
\llbracket t \rrbracket_{\rho'} \in \llbracket C(\langle x, y \rangle) \rrbracket_{\rho'}
\] \\
According to the term semantic interpretation definition:
\[
\llbracket \textsf{split}(p, x.y.t) \rrbracket_\rho = \llbracket t \rrbracket_{\rho[x \mapsto \pi_1(\llbracket p \rrbracket_\rho), y \mapsto \pi_2(\llbracket p \rrbracket_\rho)]}
\] \\
Substituting the components $u$ and $v$ of $\llbracket p \rrbracket_\rho$ yields:
\[
\llbracket \textsf{split}(p, x.y.t) \rrbracket_\rho = \llbracket t \rrbracket_{\rho'}
\]
\item \textbf{Type Consistency (Conversion)} \\
To complete the proof, ensure the result belongs to the consistent set.\\
According to the substitution property of dependent types:
\[
\llbracket C(\langle x, y \rangle) \rrbracket_{\rho'} = \llbracket C \rrbracket_{\rho[p' \mapsto (u, v)]}
\] \\
Since $\llbracket p \rrbracket_\rho \simeq (u, v)$ (in the sense of reducibility equivalence), and the $\textsf{Red}$ set is closed under computational equivalence:
\[
\llbracket C \rrbracket_{\rho[p' \mapsto \llbracket p \rrbracket_\rho]} = \llbracket C(p) \rrbracket_\rho
\] \\
Therefore, $\llbracket t \rrbracket_{\rho'} \in \llbracket C(p) \rrbracket_\rho$, and the conclusion holds.
\end{enumerate}

\textbf{F. Semantic Preservation of $\iota$-reduction (Reduction Invariance)}

To support the above steps, we need to prove that $\iota$-reduction steps do not change semantic properties. For the $\Sigma$ type:
\[
\textsf{split}(\langle u, v \rangle, x.y.t) \to_\iota t[u/x, v/y]
\]
\begin{enumerate}[label=(\arabic*)]
\item \textbf{Semantic Consistency}: By definition, the interpretation of the left side $\llbracket \textsf{split}(\langle u, v \rangle, x.y.t) \rrbracket_\rho$ expands to $\llbracket t \rrbracket_{\rho[x \mapsto \llbracket u \rrbracket_\rho, y \mapsto \llbracket v \rrbracket_\rho]}$.
\item \textbf{Substitution Lemma}: According to the substitution lemma (Lemma~\ref{lemma:KOSsubstitution}), the interpretation of the right side $\llbracket t[u/x, v/y] \rrbracket_\rho$ is semantically identical to the above expansion.
\item \textbf{Conclusion}: Since the semantic interpretations of both are the same element in the set-theoretic sense, and $\textsf{Red}_C$ satisfies CR 2 (stability), the reduced term remains in the corresponding reducibility set.
\end{enumerate}

\textbf{G. Sum Type Introduction (Injection)}

Assume the last derivation step is the left injection rule (right injection $\textsf{inr}$ similarly):
$$\frac{\Gamma \vdash a : A \quad \Gamma \vdash B : \mathcal{U}}{\Gamma \vdash \textsf{inl}_B(a) : A + B}$$
\begin{itemize}
    \item \textbf{Proof}: By the induction hypothesis, $\llbracket a \rrbracket_\rho \in \llbracket A \rrbracket_\rho$.
    \item According to the semantics of sum types: $\llbracket A + B \rrbracket_\rho = \{ \textsf{inl}(u) \mid u \in \llbracket A \rrbracket_\rho \} \cup \{ \textsf{inr}(v) \mid v \in \llbracket B \rrbracket_\rho \} \cup \textsf{Neutral}$.
    \item Since $\llbracket \textsf{inl}_B(a) \rrbracket_\rho = \textsf{inl}(\llbracket a \rrbracket_\rho)$, and $\llbracket a \rrbracket_\rho \in \llbracket A \rrbracket_\rho$.
    \item By set construction, $\textsf{inl}(\llbracket a \rrbracket_\rho)$ obviously belongs to the left branch definition part of $\llbracket A + B \rrbracket_\rho$.
\end{itemize}

\textbf{H. Sum Type Elimination ($+$-Elimination / case)}

Assume the last derivation step applies the $+$ elimination rule:
\[
\frac{\Gamma \vdash s : A + B \quad \Gamma, x:A \vdash t : C \quad \Gamma, y:B \vdash u : C}{\Gamma \vdash \textsf{case}(s, x.t, y.u) : C}
\]
We need to prove: If $\rho \models \Gamma$, then $\llbracket \textsf{case}(s, x.t, y.u) \rrbracket_\rho \in \llbracket C \rrbracket_\rho$.
\begin{enumerate}[label=(\arabic*)]
\item \textbf{Branch Premise Analysis} \\
By the induction hypothesis (IH):\\
\begin{itemize}
\item For the judgment term $s$, $\llbracket s \rrbracket_\rho \in \llbracket A + B \rrbracket_\rho$.
\item According to the semantics of $A+B$, $\llbracket s \rrbracket_\rho$ must be strongly normalizing (SN) and ultimately reduce to the form $\textsf{inl}(a)$ or $\textsf{inr}(b)$.
\end{itemize}
\item \textbf{Branch Discussion (Case Analysis)} \\
We need to discuss two semantic paths:\\
\begin{itemize}
\item \textbf{Path One: Left Injection ($\textsf{inl}$)} \\
  \begin{enumerate}[label=(\arabic*)]
  \item Assume $\llbracket s \rrbracket_\rho \twoheadrightarrow \textsf{inl}(a)$; by definition, $a \in \llbracket A \rrbracket_\rho$.
  \item Construct the assignment $\rho_x = \rho[x \mapsto a]$. Since $a \in \llbracket A \rrbracket_\rho$, $\rho_x \models (\Gamma, x:A)$.
  \item Apply the induction hypothesis to the left branch term $t$: $\llbracket t \rrbracket_{\rho_x} \in \llbracket C \rrbracket_{\rho_x}$.
  \item Since $C$ in this rule does not depend on the specific value of $s$ (simple elimination case), $\llbracket C \rrbracket_{\rho_x} = \llbracket C \rrbracket_\rho$.
  \end{enumerate}
\item \textbf{Path Two: Right Injection ($\textsf{inr}$)} \\
  \begin{enumerate}[label=(\arabic*)]
  \item Assume $\llbracket s \rrbracket_\rho \twoheadrightarrow \textsf{inr}(b)$; by definition, $b \in \llbracket B \rrbracket_\rho$.
  \item Construct the assignment $\rho_y = \rho[y \mapsto b]$. Then $\rho_y \models (\Gamma, y:B)$.
  \item Apply the induction hypothesis to the right branch term $u$: $\llbracket u \rrbracket_{\rho_y} \in \llbracket C \rrbracket_{\rho_y}$.
  \item Similarly, $\llbracket u \rrbracket_{\rho_y} \in \llbracket C \rrbracket_\rho$.
  \end{enumerate}
\end{itemize}
\item \textbf{Unification of Semantic Interpretation} \\
According to the semantic definition of the $\textsf{case}$ operator:
\[
\llbracket \textsf{case}(s, x.t, y.u) \rrbracket_\rho =
\begin{cases}
\llbracket t \rrbracket_{\rho[x \mapsto a]} & \text{if } \llbracket s \rrbracket_\rho \twoheadrightarrow \textsf{inl}(a) \\
\llbracket u \rrbracket_{\rho[y \mapsto b]} & \text{if } \llbracket s \rrbracket_\rho \twoheadrightarrow \textsf{inr}(b)
\end{cases}
\]
Regardless of which branch $\llbracket s \rrbracket_\rho$ collapses to, the result belongs to $\llbracket C \rrbracket_\rho$.
\item \textbf{Reverse Reduction Closure (CR 3 Application)} \\
Since $\textsf{case}(s, x.t, y.u)$ reaches $\llbracket t \rrbracket_{\rho_x}$ or $\llbracket u \rrbracket_{\rho_y}$ via reduction ($\iota$-reduction), according to the CR 3 property of the reducibility set (and closure under reverse reduction), the original $\textsf{case}$ term itself must also belong to the reducibility set $\llbracket C \rrbracket_\rho$. The conclusion holds.
\end{enumerate}
\end{proof}

Soundness means "everything that can be proven is true." In Theorem~\ref{Theorem:KosSoundness}, "what can be proven" is the type judgment $\Gamma \vdash t : A$, and "true" is the term's membership in the semantic model $\llbracket t \rrbracket \in \llbracket A \rrbracket$. It ensures that the syntactic constructions of KOS-TL Core do not deviate from their logical semantics.

\begin{Theorem}{Consistency Theorem}

There does not exist a term $t$ such that $\emptyset \vdash t : \bot$. The system is thus called consistent.
\end{Theorem}

\begin{proof}
The core idea of this proof is that syntactic derivations cannot escape semantic boundaries. We unfold the argument in three stages.

\textbf{Stage 1: Using Strong Normalization (Syntactic Normalization)}

According to the strong normalization theorem of KOS-TL Core, if there exists a term $t$ satisfying $\emptyset \vdash t : \bot$, then $t$ must reduce to a normal form (Normal Form) $t_{nf}$, with the type preserved:
\[
\emptyset \vdash t_{nf} : \bot
\]
In the empty context, the normal form can only be a constructor term (Constructor). However, according to the definition of the $\bot$ type, it has no introduction rules (Introduction Rules), meaning no constructor can produce $\bot$. This implies that, at the syntactic level, $t_{nf}$ does not exist.

\textbf{Stage 2: Semantic Soundness Mapping}

To make the argument mathematically irrefutable, we use the interpretation model $\mathcal{M}$.

\textbf{Step 1: Establish the Mapping}

The interpretation function $\llbracket \cdot \rrbracket$ maps the syntactic world (Types/Terms) to the semantic world (Sets/Elements).
\begin{itemize}
\item For any type $A$, its interpretation $\llbracket A \rrbracket$ is a set.
\item For any term $t : A$, its interpretation $\llbracket t \rrbracket$ must be an element in the set $\llbracket A \rrbracket$.
\end{itemize}

\textbf{Step 2: Special Nature of the Empty Type}

When defining the semantics of $\bot$, we map it to the mathematical absolute empty set:
\[
\llbracket \bot \rrbracket = \emptyset
\]
This is reasonable, as the logical ``false'' corresponds in model theory to a state with no witnesses (Witness).

\textbf{Step 3: Apply Semantic Soundness}

According to the semantic soundness theorem~\ref{Theorem:KosSoundness}:
\[
\text{If } \Gamma \vdash t : A, \text{ then for all valid assignments } \rho, \llbracket t \rrbracket_\rho \in \llbracket A \rrbracket
\]
In the empty context $\emptyset$, no assignment $\rho$ is needed, directly yielding:
\[
\llbracket t \rrbracket \in \llbracket \bot \rrbracket
\]

\textbf{Stage 3: Reduction to Absurdity and Contradiction}

\begin{enumerate}[label=(\arabic*)]
\item From the above derivation: $\llbracket t \rrbracket \in \emptyset$.
\item According to the set theory axioms of extensionality and the empty set: $\forall x, x \notin \emptyset$.
\item Judgment: $\llbracket t \rrbracket \in \emptyset$ and $\forall x, x \notin \emptyset$ constitute a direct logical contradiction.
\item Backtracking: Since the semantic interpretation and set theory axioms are presupposed to be correct, the source of the contradiction can only be the assumption ``existence of term $t$.''
\end{enumerate}
Conclusion: The assumption does not hold, $\neg \exists t, \emptyset \vdash t : \bot$. Consistency is proven.
\end{proof}

In systems like KOS-TL Core based on the Curry-Howard isomorphism, logical consistency is equivalent to proving that the empty type (Empty Type) is uninhabited (Uninhabited).

\textbf{Corollary: Logical Consistency (Consistency)}
Since the KOS-TL core layer satisfies strong normalization (SN) and has type preservation (Subject Reduction), and there is no constructor for the empty type $\bot$ in the system, there does not exist a term $t$ such that $\vdash t : \bot$. This proves that the core layer, as the ``formal constitution,'' is logically contradiction-free.

\begin{Definition}{Confluence}

For any Core layer terms $M, N, P$, if there exist reduction paths such that $M \twoheadrightarrow N$ and $M \twoheadrightarrow P$, then there necessarily exists a term $Q$ such that $N \twoheadrightarrow Q$ and $P \twoheadrightarrow Q$.
\end{Definition}

Let $\to$ be the one-step reduction relation (including $\beta, \delta, \zeta, \eta$ reductions), and $\twoheadrightarrow$ its transitive closure (multi-step reduction). Confluence logically means that all forks eventually converge.

\begin{Theorem}{KOS-TL Core Confluence Theorem}

All well-formed terms in the KOS-TL Core layer satisfy confluence and have a unique normal form (Unique Normal Form).
\end{Theorem}

\begin{proof}
We adopt the Tait-Martin-L\"{o}f parallel reduction method (Parallel Reduction) combined with the strong normalization property for the proof.

\textbf{Step A: Define Parallel Reduction ($\Rightarrow$)}

To handle the ``simultaneous reduction of forks'' that single-step reduction cannot cover, we define a parallel reduction relation $\Rightarrow$, which is an extension of the single-step reduction $\rightarrow$, allowing simultaneous reduction of multiple subparts of a term. Specifically, $\Rightarrow$ is the minimal relation satisfying the following rules:
\begin{enumerate}[label=(\arabic*)]
\item If $M \rightarrow M'$ (single-step $\beta$-reduction), then $M \Rightarrow M'$.
\item For any reduction context $C[\cdot]$, if $M \Rightarrow M'$, then $C[M] \Rightarrow C[M']$.
\item Parallel application of reduction:
$$\frac{M \Rightarrow M' \quad N \Rightarrow N'}{(\lambda x. M) N \Rightarrow M'[x := N']}$$
where $[x := N']$ denotes capture-avoiding substitution.
\item For constructors (e.g., $\Pi x : A. M$), allow simultaneous reduction of domain and body:
$$\frac{A \Rightarrow A' \quad M \Rightarrow M'}{\Pi x : A. M \Rightarrow \Pi x : A'. M'}$$
($\Sigma$ and $\mathsf{Id}$ types are similar).
\end{enumerate}

\textbf{Step B: Prove the Diamond Property}

We prove that the parallel reduction $\Rightarrow$ satisfies the diamond property: If $M \Rightarrow N$ and $M \Rightarrow P$, then there exists $Q$ such that $N \Rightarrow Q$ and $P \Rightarrow Q$.
The proof proceeds by induction on the structure of $M$:
\begin{itemize}
\item \textbf{Base Case}: If $M$ is a variable or atom, then $N = P = M$, take $Q = M$.
\item \textbf{Inductive Case}: Assume it holds for all proper subterms.
\textbf{Case: $M = (\lambda x. M_1) M_2$}: By the definition of $\Rightarrow$, $N$ and $P$ arise from different sub-reduction forks. The induction hypothesis applies to $M_1$ and $M_2$, yielding $Q_1, Q_2$ such that the subterms converge, then $Q = Q_1[x := Q_2]$.
\textbf{Case: Constructor (e.g., $\Pi x : A. M$)}: Constructors in KOS-TL Core are orthogonal (no overlapping reduction rules); induction applies to $A$ and $M$, converging to $\Pi x : A''. M''$. The $\mathsf{Id}$ type is similar, with no internal conflicts.
\end{itemize}
By orthogonality and type preservation, reductions have no critical pairs, ensuring the diamond property.

\textbf{Step C: Deriving Multi-Step Reduction from Parallel Reduction}

The multi-step reduction $\twoheadrightarrow$ is the reflexive-transitive closure of $\Rightarrow$: $M \twoheadrightarrow N$ if and only if there exists a chain $M = M_0 \Rightarrow M_1 \Rightarrow \cdots \Rightarrow M_k = N$.
By Newman's lemma, if $\Rightarrow$ satisfies the diamond property and the system is strongly normalizing (SN, no infinite chains), then $\twoheadrightarrow$ satisfies confluence: For $M \twoheadrightarrow N_1$ and $M \twoheadrightarrow N_2$, there exists $Q$ such that $N_1 \twoheadrightarrow Q$ and $N_2 \twoheadrightarrow Q$. SN ensures weak normalization, and the diamond property implies local confluence.

\textbf{Step D: Unique Normal Form Proof}

Assume $M$ has two normal forms $N_1$ and $N_2$ (irreducible). By confluence, there exists $Q$ such that $N_1 \twoheadrightarrow Q$ and $N_2 \twoheadrightarrow Q$. But $N_1, N_2$ are normal forms, hence $N_1 = Q = N_2$. By SN, every term has a normal form, hence it is unique.
\end{proof}

The decidability of KOS-TL Core is primarily supported by the following two properties:
\begin{enumerate}[label=(\arabic*)]
\item \textbf{Decidability of Type Checking}: Given $\Gamma, t, A$, determine whether $\Gamma \vdash t : A$ holds.
\item \textbf{Decidability of Equivalence Judgment}: Given $\Gamma, t, u$, determine whether $t \equiv u$ (whether the two terms are logically equivalent in the current context) holds.
\end{enumerate}
As the mathematical foundation of the entire system, the Core layer must theoretically guarantee that all basic operations (type checking, equivalence judgment) terminate in a finite number of steps under any circumstances.

\begin{Theorem}{Core Layer Decidability Theorem}

The type checking problem and term equivalence judgment problem in the KOS-TL Core language are decidable.
\end{Theorem}

\begin{proof}
This proof is built on the foundations of strong normalization (Strong Normalization, SN) and confluence (Confluence).

\textbf{1. Finiteness of Reduction Sequences}

According to the strong normalization theorem proven by the Tait-Girard method, any well-formed term $t$ in KOS-TL Core has no infinite reduction sequences. This means that starting from any term, through reduction steps such as $\beta, \delta, \zeta$, etc., it inevitably reaches a unique normal form $\textsf{nf}(t)$ in a finite number of steps.

\textbf{2. Algorithmization of Equivalence Judgment}

The process of determining $t \equiv u$ can be transformed into:
\begin{enumerate}[label=(\arabic*)]
\item Reduce $t$ to normal form $t^*$.
\item Reduce $u$ to normal form $u^*$.
\item Compare whether $t^*$ and $u^*$ are syntactically identical.
\end{enumerate}
Since steps 1 and 2 are guaranteed to complete in finite time, and step 3 is simple symbol matching, $t \equiv u$ is decidable.

\textbf{3. Recursive Termination of Type Checking Algorithm}

The type checker processes term $t$ recursively according to its syntactic structure:
\begin{itemize}
\item For application terms $(f \, a)$, check if the type of $f$ is a $\Pi$-type and determine if the type of $a$ matches.
\item During type matching, invoke the aforementioned equivalence judgment algorithm.
\end{itemize}
Since the term construction is finite and equivalence judgment is decidable, the entire recursive process must terminate.
\end{proof}

\subsection{Application Example: Quality Anomaly Knowledge Modeling}

In manufacturing scenarios, a "qualified batch" is not just a data record; it must include evidence of passing quality inspection.

\begin{itemize}
\item \textbf{Type Definition}:
$$\textsf{QualifiedBatch} \equiv \Sigma(b : \textsf{BatchID}). \Sigma(res : \textsf{Result}). \textsf{Id}_{\textsf{Result}}(res, \textsf{Pass})$$
\item \textbf{Logical Interpretation}:
This type requires that any instance must contain a batch ID, a quality inspection result, and an identity term proving that the result equals $\textsf{Pass}$.
\item \textbf{Construction Attempt}:
If the batch's quality inspection result is $\textsf{Failure}$, then since the type is empty (no constructors), the system will reject instantiation of the object at the core layer, thereby preventing unqualified products from entering subsequent processes at the logical layer.
\end{itemize}

\begin{Example}{Constructing a Temperature Reading Within Safe Range}

(1) Declare atomic predicates (as axioms or basic judgments)

In the initialization context of the Core layer, we need to introduce a dimensional judgment predicate (is\_unit):
$$\Gamma \vdash \textsf{is\_unit\_Celsius} : \Pi(v: \textsf{Val}). \textsf{Prop}$$
Additionally, introduce the comparison predicate construction $\textsf{is\_safe}$:
$$\Gamma \vdash L, H : \textsf{Val} \quad \Gamma \vdash \textsf{is\_safe} \equiv \lambda v. \textsf{And}(L \le v, v \le H) : \textsf{Val} \to \textsf{Prop}$$

(2) Construct the $Temp$ Type

Using the $\Sigma$-type construction rule:
$$\frac{\Gamma \vdash \textsf{Val} : \mathcal{U} \quad \Gamma, v:\textsf{Val} \vdash \textsf{is\_unit\_Celsius}(v) : \textsf{Prop}}{\Gamma \vdash \Sigma(v: \textsf{Val}). \textsf{is\_unit\_Celsius}(v) : \mathcal{U}}$$
At this point, $Temp \equiv \Sigma(v: \textsf{Val}). \textsf{is\_unit\_Celsius}(v)$ formally becomes a valid type.

(3) Construct $QualifiedTemp$ on This Basis

Now, we overlay the logical safety predicate $\textsf{is\_safe}$:
$$QualifiedTemp \equiv \Sigma(t: Temp). \textsf{is\_safe}(\textsf{proj}_1(t))$$

(4) Construct Object Instance

Assume $v = 25$: $p_{unit} : \textsf{is\_unit\_Celsius}(25)$
$p_{range} : \textsf{is\_safe}(25)$
$obj = \langle \langle 25, p_{unit} \rangle, p_{range} \rangle$
\end{Example}

\section{KOS-TL Kernel Layer: State Evolution and Operational Semantics}

The Kernel Layer is the dynamical core of KOS-TL. Based on the static type system of the Core layer, it introduces a time dimension and state transition mechanism, realizing the atomic evolution of the knowledge base from state $\Sigma$ to $\Sigma'$ through controlled event-driven processes.

\subsection{Syntax}

The KOS-TL Kernel Layer introduces the concept of ``dynamics,'' extending the domain from static objects to state transition trajectories. It serves as the bridge connecting logic and execution.

The KOS-TL Kernel Layer can be represented as a triple:
$$\langle \Sigma, \textsf{Ev}, \Delta \rangle$$

\begin{Definition}{State ($\Sigma$)}

The Kernel Layer state ($\Sigma$) is defined as a configuration triple:
$$\Sigma \equiv \langle \mathcal{K}, \mathcal{TS}, \mathcal{P} \rangle$$
\begin{enumerate}[label=(\arabic*)]
    \item Knowledge Base ($\mathcal{K}$ - Knowledge Base)
    $$\mathcal{K} = \{ (id_i, t_i, A_i) \mid \Gamma_{Core} \vdash t_i : A_i \}$$
    It stores all currently verified truths (Facts) in the system.
    \item Logical Clock ($\mathcal{TS}$ - Logical Clock)
    Based on the Core layer base sort $\textsf{Time}$. It is not merely a counter but a measure of the total order of states.
    $$\mathcal{TS} : \textsf{Time} \quad \text{satisfying the monotonicity rule: } \Sigma \to \Sigma' \implies \mathcal{T}' > \mathcal{T}$$
    \item Pending Queue ($\mathcal{P}$ - Pending Events)
    An ordered sequence composed of restricted events (Events).
\end{enumerate}
\end{Definition}

\begin{Definition}{Event $\textsf{Ev}$}
An event $\textsf{Ev}$ is a well-formed quintuple under the global context $\Gamma$:
$$\textsf{Ev} \equiv \langle \textsf{Args}, \textsf{Pre}, \textsf{Op}, \textsf{Post}, \textsf{Prf} \rangle$$
The type constraints and logical semantics of each component are as follows:
\begin{enumerate}[label=(\arabic*)]
    \item $\textsf{Args}$ (Argument Set):
        $$\textsf{Args} : A$$ (where $A \in \mathcal{U}_{Core}$).
        Instantiated data ingested from the external world (Runtime layer), such as sensor\_value or transaction\_amount.
    \item $\textsf{Pre}$ (Precondition Predicate):
        $$\textsf{Pre} : \textsf{Args} \to \Sigma \to \textsf{Prop}$$
        A dependent proposition defining the logical prerequisites that the event must satisfy in the current state $\Sigma$. It can reference the current knowledge base $\mathcal{K}$ or logical time $\mathcal{TS}$.
    \item $\textsf{Op}$ (Operation Operator):
        $$\textsf{Op} : \textsf{Args} \to \Sigma \to \Sigma$$
        The core evolution function. It describes how to generate the new state $\Sigma'$.
        \begin{itemize}
            \item $\mathcal{K} \to \mathcal{K}'$: Addition and removal of knowledge items.
            \item $\mathcal{TS} \to \mathcal{TS} + \Delta t$: Stepping of the logical clock.
            \item $\mathcal{P} \to \mathcal{P}'$: Update of the pending intent queue (consuming itself or deriving new intents).
        \end{itemize}
    \item $\textsf{Post}$ (Postcondition/Invariant):
    $$\textsf{Post} : \Sigma' \to \textsf{Prop}$$
    Defines safety criteria (Safety Properties) that must be satisfied after the transformation, such as ``energy conservation,'' ``non-negative account,'' or ``clock monotonicity.''
    \item $\textsf{Prf}$ (Proof Term):
$$\textsf{Prf} : \textsf{Pre}(\textsf{Args}, \Sigma)$$
This is the ``passport'' of the event. When signals are refined into events in the Runtime layer, a constructive proof that the precondition holds must be constructed. If there is no valid $\textsf{Prf}$, the kernel will reject execution of the event.
\end{enumerate}
\end{Definition}

\begin{Definition}{Transition Record ($\Delta$)}

To support causal tracing, $\Delta$ needs to record clock jumps:
$$\Delta \subseteq \Sigma \times \textsf{Ev} \times \Sigma$$
A typical transition record item:
$$\delta = \langle \langle \mathcal{K}, \mathcal{TS}, \mathcal{P} \rangle \xrightarrow{e} \langle \mathcal{K}', \mathcal{TS}', \mathcal{P}' \rangle \rangle$$
$$\forall \langle \Sigma \xrightarrow{e} \Sigma' \rangle \in \Delta, \quad \Sigma'.\mathcal{T} > \Sigma.\mathcal{T}$$.
This theoretically locks the ``arrow of time,'' ensuring the irreversibility of knowledge evolution.
Every new knowledge item $ku_{new}$ injected into $\mathcal{K}'$ implicitly carries the current $\mathcal{T}'$.
During a transition, $e$ is dequeued from $\mathcal{P}$, and after executing $\textsf{Op}$, its result is merged into $\mathcal{K}'$.
\end{Definition}

\begin{Definition}{Evolutionary Determinism}

Given a state $\Sigma$ and event $e$, if there exists a $\Sigma'$ satisfying the operational semantics, then its normal form (Normal Form) is unique in the sense of intensional equivalence.
\end{Definition}

\subsection{Operational Semantics}

The evolution of the Kernel Layer follows ``Small-step Operational Semantics.'' Let $\Sigma$ be the system's configuration (Configuration); its state transition rules are defined as:
\begin{equation} \frac{ e = \textsf{head}(\Sigma.\mathcal{P}) \quad \Gamma, \Sigma.\mathcal{K}, \Sigma.\mathcal{TS} \vdash p : \textsf{Pre}(\textsf{Args}_e, \Sigma) \quad \Sigma' = \textsf{Op}(\textsf{Args}_e, \Sigma) \quad \Sigma' \vdash p' : \textsf{Post}(\textsf{Args}_e, \Sigma') }{ \langle \Sigma, e \rangle \longrightarrow_{KOS} \Sigma' }
\end{equation}
where:
\begin{itemize}
    \item Intent Trigger Condition ($e = \textsf{head}(\Sigma.\mathcal{P})$)

    This specifies the source of the event $e$. Transitions are not random but driven by the head event of the pending queue $\mathcal{P}$. This ensures the ordered nature of evolution, meaning the kernel schedules according to the logical order of the intent queue.
    \item Environment-Aware Proof Judgment ($\Gamma, \Sigma.\mathcal{K}, \Sigma.\mathcal{T} \vdash p$)

    Explicitly lists the context on which the proof term $p$ depends. The proof of the precondition not only depends on the global context $\Gamma$ but must also be consistent with facts in the current knowledge base $\mathcal{K}$ and the logical time $\mathcal{TS}$. This implements the logic that events can only occur on the correct time and facts.
    \item Parameterized Operator Application ($\Sigma' = \textsf{Op}(\textsf{Args}_e, \Sigma)$)

    Introduces $\textsf{Args}_e$. Emphasizes that the transition is an overall transformation of the current triple configuration based on the specific parameters carried by the event (refined from the Runtime layer).
    \item Postcondition Self-Consistency ($\Sigma' \vdash p' : \textsf{Post}(\textsf{Args}_e, \Sigma')$)

    Explicitly states that $\textsf{Post}$ is judged under the new state $\Sigma'$. This defines the hard threshold for logical commit (Commit). If the evolved state cannot satisfy its safety invariant (e.g., the clock did not step forward, or the knowledge base has inconsistencies), the judgment fails, and the transition rule is invalid.
\end{itemize}

This semantics stipulates that a valid knowledge transition must simultaneously satisfy "provable premise" and "compliant result". If any condition cannot be proven in the Core layer, the state remains unchanged (i.e., execution rollback).
The Kernel layer does not handle retry strategies for physical failures; it only defines "logically valid evolution trajectories".Any physical attempt that fails Post validation manifests as an "unoccurred transition" in the Kernel layer, thereby enforcing atomicity of transactions at the logical layer.

\begin{Example}{Reduction Example Demonstration}

\label{exam:KOSKernel-reduction}
Assume the following scenario. The sensor data fusion assumes there are two independent sensors in the system: $ku_1$ (temperature) and $ku_2$ (humidity). We need a merge function combine to encapsulate them into an ``environment state'' object.
\begin{enumerate}[label=(\arabic*)]
    \item Basic Type and Predicate Definitions
    The target type $Env \equiv \Sigma(t:Temp).(Humi)$, where the environment state is a dependent pair containing temperature and humidity. Among them:
    \centerline{$Temp \equiv \Sigma(v:\textsf{Val}). \textsf{is\_T}(v)$}
    \centerline{$Humi \equiv \Sigma(v:\textsf{Val}). \textsf{is\_H}(v)$}
    \item Specific Knowledge Items (Instances)
    [[
    $ku_1 = \langle 25, p_T \rangle : Temp$
    $ku_2 = \langle 60, p_H \rangle : Humi$
    ]]
    \item Merge Function ($\Pi$-type)
    Define a function that receives temperature and humidity and returns an environment object:
    $$\textsf{combine} \equiv \lambda t:Temp. \lambda h:Humi. \langle t, h \rangle$$
    Its type is $\Pi(t:Temp). \Pi(h:Humi). Env$.
\end{enumerate}
Now we demonstrate the reduction process of applying $\textsf{combine}$ to $ku_1$ and $ku_2$. This typically occurs when the Kernel receives two signals and attempts to update the global state.

Step 1: Construct the Initial Application Term

In the Kernel's control flow, a term to be reduced is generated:
$$(\textsf{combine} \,\, ku_1) \,\, ku_2$$

Step 2: First $\beta$-reduction (Substitute Temperature)

According to the $\beta$-reduction rule $(\lambda x. M) N \to M[N/x]$:
$$(\lambda t. \lambda h. \langle t, h \rangle) \,\, ku_1 \to \lambda h. \langle ku_1, h \rangle$$
The function ``consumes'' the temperature data, becoming a specialized function ``waiting for humidity input.''

Step 3: Second $\beta$-reduction (Substitute Humidity)

$$(\lambda h. \langle ku_1, h \rangle) \,\, ku_2 \to \langle ku_1, ku_2 \rangle$$
The humidity data is filled in, generating a complete pair.

Step 4: Unfolding and Structural Reduction ($\iota$-reduction)

If the system needs to further extract the original values (e.g., for executing analyze), then $\iota$-reduction occurs:
$$\textsf{proj}_1(\langle ku_1, ku_2 \rangle) \to ku_1 = \langle 25, p_T \rangle$$
$$\textsf{proj}_1(\textsf{proj}_1(\langle ku_1, ku_2 \rangle)) \to 25$$
The above reduction processes are accompanied by type judgment. According to type preservation (Subject Reduction), every term in the entire reduction process must be well-formed:
Initial term: $(\textsf{combine} \,\, ku_1) \,\, ku_2$ has type $Env$.
Intermediate term: $\lambda h. \langle ku_1, h \rangle$ has type $\Pi(h:Humi). Env$.
Final term: $\langle ku_1, ku_2 \rangle$ has type $Env$.
In this example, the reduction operation completes the transformation from ``logical intent'' (how to merge data) to ``logical fact'' (the merged data object). From the Core layer's perspective, $(\textsf{combine} \,\, ku_1) \,\, ku_2$ and $\langle ku_1, ku_2 \rangle$ are judgmentally equal (Judgmentally Equal). They are different expressions of the same truth. From the Kernel layer's perspective, reduction is a computational evaluation. It consumes CPU cycles, merging two scattered memory pointers into a new $\Sigma$ structure.
\end{Example}

KOS-TL builds a "firewall" through static type semantics before logical execution, with type mismatch interception (Type Mismatch Interception) occurring before reduction. According to the Core layer's judgment rules, if a term (Term) cannot pass type checking, it will never be pushed into the Kernel's reduction engine.

Assume we have the merge function combine, which expects a humidity object $Humi$:
$$\textsf{combine} : \Pi(t:Temp). \Pi(h:Humi). Env$$
Now, the Runtime layer erroneously captures a pressure signal $p : Press$ and attempts to perform the merge:
$$(\textsf{combine} \,\, ku_1) \,\, p$$
The Kernel invokes the $\Pi$-elimination rule (Application Rule):
$$\frac{\Gamma \vdash f : \Pi(h:Humi).Env \quad \Gamma \vdash p : A}{\Gamma \vdash f\,p : Env [p/h] \quad (\text{requiring } A \equiv Humi)}$$
The kernel attempts to judge $Press \equiv Humi$.
$Humi \equiv \Sigma(v:\textsf{Val}). \textsf{is\_H}(v)$
$Press \equiv \Sigma(v:\textsf{Val}). \textsf{is\_P}(v)$
Since the predicates $\textsf{is\_H} \neq \textsf{is\_P}$, type unification (Unification) fails.
Thus, the term $(\textsf{combine} \,\, ku_1) \,\, p$ is judged ill-typed (Ill-typed). The reduction engine rejects the $\beta$-reduction execution, the system state $\sigma$ remains unchanged, and a type error exception is triggered.

\begin{Example}{Causal Backtracking Analysis}

Based on Example~\ref{exam:KOSKernel-reduction}, when the system discovers that although the merged result is "type correct", it is "logically anomalous" (e.g., the value of $Env$ exceeds the safe range), it needs to use the Id type (identity type) for causal backtracking.

Assume we have obtained the merged object $obj = \langle ku_1, ku_2 \rangle$, but the analyze predicate deems it invalid.
The backtracking process follows the following logical reduction:
\begin{enumerate}[label=(\arabic*)]
    \item Deconstruction
    Through the Core layer's projection operator $\textsf{proj}_i$, decompose the composite object back to the original evidence:
    $$t = \textsf{proj}_1(obj) \quad h = \textsf{proj}_2(obj)$$
    \item Identity Verification
    The kernel constructs an equivalence statement, requiring proof that the current data is consistent with the input source:
    $$\textsf{Id}_{Temp}(t, ku_1) \wedge \textsf{Id}_{Humi}(h, ku_2)$$
    If refl (reflexivity proof) cannot be constructed here, it indicates a computational error or memory corruption during merging.
    \item Root Cause Localization
    The backtracking analysis function analyze will reverse-search along the reduction steps. In KOS-TL, this manifests as checking the proof term: inspect the right projection $\textsf{proj}_2(ku_1)$ of $ku_1$, i.e., the temperature safety proof $p_T$. If $p_T$ validation fails, judge: the root cause lies in the input data of sensor 1. If $p_T$ validation passes, judge: the root cause lies in the logical computation of the merge function combine.
\end{enumerate}
\end{Example}

\subsection{General Operators}

\subsubsection{State Projection Operators}
The projection operators define the logic for extracting components from complex dependent pairs (Dependent Pairs).

\begin{itemize}
\item \textbf{Knowledge Extraction Operator ($\textsf{get\_K}$)}
\begin{itemize}
\item \textbf{Core Type}: $\textsf{get\_K} : \Pi(\sigma : \Sigma). \textsf{Set}(\textsf{Facts})$
\item \textbf{Operator Definition}: $\textsf{get\_K} \equiv \lambda \sigma. \textsf{proj}_1(\sigma)$
\item \textbf{Description}: This operator uses the first projection to extract the knowledge base $\mathcal{K}$. In the Core layer, it ensures that the returned set items are all well-formed type instances.
\end{itemize}
\item \textbf{Clock Reading Operator ($\textsf{now}$)}
\begin{itemize}
    \item \textbf{Core Type}: $\textsf{now} : \Pi(\sigma : \Sigma). \textsf{Time}$
    \item \textbf{Operator Definition}: $\textsf{now} \equiv \lambda \sigma. \textsf{proj}_1(\textsf{proj}_2(\sigma))$
    \item \textbf{Description}: Extracts the middle item $\mathcal{T}$ of the triple. This operator is the foundation for all temporal logic judgments (e.g., ``whether the contract has expired'').
\end{itemize}
\end{itemize}

\subsubsection{Intention Scheduling Operators}
The scheduling operators manage the intent queue $\mathcal{P}$ through recursive list operations.

\begin{itemize}
\item \textbf{Intent Push Operator ($\textsf{schedule}$)}
\begin{itemize}
\item \textbf{Core Type}: $\textsf{schedule} : \Pi(\sigma : \Sigma). \Pi(e : \textsf{Ev}). \Sigma$
\item \textbf{Operator Definition}: $\textsf{schedule} \equiv \lambda \sigma. \lambda e. \langle \textsf{get\_K}(\sigma), \textsf{now}(\sigma), \textsf{append}(\textsf{proj}_2(\textsf{proj}_2(\sigma)), e) \rangle$
    \item \textbf{Description}: This operator constructs a new $\Sigma$ instance. Its core is to append a quintuple conforming to the $\textsf{Ev}$ type restrictions to the end of the queue.
\end{itemize}
\end{itemize}

\subsubsection{Evolution Control Operators}
This is the core driving the system's forward evolution, involving the fusion of computation and consistency judgment.

\begin{itemize}
\item \textbf{Clock Stepping Operator ($\textsf{tick}$)}
\begin{itemize}
\item \textbf{Core Type}: $\textsf{tick} : \Pi(\sigma : \Sigma). \Sigma$
\item \textbf{Operator Definition}: $\textsf{tick} \equiv \lambda \sigma. \langle \textsf{get\_K}(\sigma), \textsf{now}(\sigma) + 1, \textsf{consume}(\sigma) \rangle$
    \item \textbf{Description}: It not only increments the time count but is usually accompanied by the consumption of the current event, representing the completion of a logical step.
\end{itemize}
    \item \textbf{Knowledge Unification Operator ($\textsf{unify}$)}
\begin{itemize}
    \item \textbf{Core Type}: $\textsf{unify} : \Pi(\sigma : \Sigma). \Pi(f : \textsf{Fact}). \Sigma$
    \item \textbf{Operator Definition}: $\textsf{unify} \equiv \lambda \sigma. \lambda f. \textsf{if } \textsf{is\_consistent}(\textsf{get\_K}(\sigma), f) \text{ then } \langle \textsf{get\_K}(\sigma) \cup \{f\}, \textsf{now}(\sigma), \dots \rangle \text{ else } \sigma$
    \item \textbf{Description}: This is the most complex operator. Before merging a new fact, it uses the Core layer's judgment rules to verify the logical compatibility (Consistency) of $f$ with the existing $\mathcal{K}$.
\end{itemize}
\end{itemize}

\subsubsection{Causal \& Trace Operators}
Utilize the identity type (Identity Type) for deep auditing.

\begin{itemize}
\item \textbf{Identity Verification Operator ($\textsf{verify\_id}$)}
\begin{itemize}
\item \textbf{Core Type}: $\textsf{verify\_id} : \Pi(\sigma_1 : \Sigma). \Pi(\sigma_2 : \Sigma). \textsf{Type}$
\item \textbf{Operator Definition}: $\textsf{verify\_id} \equiv \lambda \sigma_1. \lambda \sigma_2. \textsf{Id}_{\Sigma}(\sigma_1, \sigma_2)$
\item \textbf{Description}: Returns a proposition type (Prop). To execute in the Kernel, a constructive proof (e.g., $\textsf{refl}$) must be provided to verify whether the two configurations are logically the same truth.
\end{itemize}
\end{itemize}

When the above Core layer operators are invoked by the Kernel layer, their execution follows the following reduction path:
\begin{enumerate}[label=(\arabic*)]
\item \textbf{Parameter Substitution ($\beta$-reduction)}: Substitute the current state instance of the Kernel layer (e.g., $\sigma_{current}$) into the $\lambda$-term of the operator.
\item \textbf{Structural Unfolding ($\iota$-reduction)}: The projection operator $\textsf{proj}$ extracts specific components from the triple.
\item \textbf{State Materialization}: The final term obtained from reduction (e.g., new $\Sigma'$) is stored in the kernel storage, becoming the input for the next cycle.
\end{enumerate}

\begin{Definition}{Termination of Operators}

Due to the Core layer's computational model based on strong normalization (Strong Normalization), all kernel general operators must terminate and produce results in a finite number of steps. This theoretically avoids ``dead loops'' in the kernel during state transitions.
\end{Definition}

\subsection{Logical Properties of the Kernel Layer}
In the KOS-TL kernel architecture, state preservation (Preservation), also often referred to as type preservation (Subject Reduction) in the state machine dimension, ensures that the logical "well-formedness" of the system does not collapse due to data inflow during environment evolution or transaction commit.

\begin{Theorem}{State Preservation}

Let $\Gamma$ be the system global context. If the kernel state $\Sigma$ is well-formed (denoted $\Gamma \vdash \Sigma \,\, \textsf{ok}$), and there exists a transition step triggered by event $e$ such that $\Sigma \xrightarrow{e} \Sigma'$, then the new state $\Sigma'$ after the transition is still well-formed:
\[
\Gamma \vdash \Sigma \,\, \textsf{ok} \quad \wedge \quad \Sigma \xrightarrow{e} \Sigma' \implies \Gamma \vdash \Sigma' \,\, \textsf{ok}
\]
where well-formedness $\Sigma \,\, \textsf{ok}$ is defined as: for all fact items $ku_i \in \Sigma$ contained in the state, there exists a type $A_i$ such that $\Gamma \vdash ku_i : A_i$, and $\Sigma$ satisfies consistency $\Sigma \not\vdash \bot$.
\end{Theorem}

\begin{proof}
We prove by structural induction on the transition operator according to the nature of event $e$.

\textbf{1. Internal Computation Step:}
If $e$ corresponds to an internal reduction in the kernel (e.g., expression simplification via $\beta$-reduction), then $\Sigma' = \Sigma$ and only the control item $ctrl$ changes.
\begin{itemize}
\item According to the \textbf{Subject Reduction} theorem of the Core layer: If $\Gamma, \Sigma \vdash ctrl : A$ and $ctrl \to ctrl'$, then $\Gamma, \Sigma \vdash ctrl' : A$.
\item Since $\Sigma$ itself does not change, its well-formedness $\Gamma \vdash \Sigma \,\, \textsf{ok}$ is automatically preserved.
\end{itemize}

\textbf{2. External Fact Injection:}
If $e$ corresponds to injecting a new fact $ku_{new}$ into the knowledge base, the transition is defined by $\textsf{unify}(\Sigma, ku_{new})$.
\begin{itemize}
\item \textbf{Type Pre-check}: The premise of the transition is that $ku_{new}$ must pass type checking, i.e., $\Gamma, \Sigma \vdash ku_{new} : A_{new}$.
\item \textbf{Consistency Conflict Handling}:
\begin{itemize}
\item \textit{Branch A (Compatible)}: If $\Sigma \cup \{ku_{new}\} \not\vdash \bot$, then $\Sigma' = \Sigma \cup \{ku_{new}\}$. By the weakening lemma (Weakening Lemma), the existing fact items remain well-typed in the larger context.
\item \textit{Branch B (Conflict)}: If $\Sigma \cup \{ku_{new}\} \vdash \pi : \bot$, the kernel does not merge directly but constructs $\Sigma' = \Sigma \cup \{ \textsf{Invalidated}(ku_{new}, \pi) \}$.
\end{itemize}
\item In both branches, $\Sigma'$ does not contain directly derivable $\bot$, and all elements have corresponding constructive proofs. Thus, $\Gamma \vdash \Sigma' \,\, \textsf{ok}$.
\end{itemize}

\textbf{3. Environment Elimination Step:}
If $e$ corresponds to the application of an elimination rule (e.g., extracting a projection item from a $\Sigma$-type fact).
\begin{itemize}
\item Assume $\langle a, p \rangle : \Sigma(x:A).B$ exists in $\Sigma$. The transition step produces $a : A$.
\item According to the semantic soundness of the $\Sigma$-elimination rule, the projected item $a$ has a predefined and valid type $A$.
\item This operation is merely an unfolding of existing well-formed knowledge and does not introduce inconsistencies, so $\Sigma'$ remains well-formed.
\end{itemize}

\textbf{Conclusion:} In summary, regardless of the transition event $e$, the new state $\Sigma'$ maintains logical well-formedness and consistency.
\end{proof}

In the KOS-TL Kernel layer, determinism (Determinism) is the cornerstone ensuring distributed consensus (Consensus) and logical traceability. Under the dependent type system, determinism not only means computational result consistency but also confluence (Confluence) of reduction paths.

\begin{Theorem}{Determinism of Kernel Evolution}

Let $\textsf{Op}$ be the kernel state transition function, $\Sigma$ the current well-formed kernel state, and $e$ the triggering event. If the transition rule is defined as $\Sigma' = \textsf{Op}(\Sigma, e)$, then for the same input pair $(\Sigma, e)$, the output new state $\Sigma'$ is unique in the sense of logical equivalence:
\[
\forall \Sigma, e, \Sigma'_1, \Sigma'_2: (\Sigma \xrightarrow{e} \Sigma'_1 \wedge \Sigma \xrightarrow{e} \Sigma'_2) \implies \Sigma'_1 \equiv \Sigma'_2
\]
where $\equiv$ denotes intensional equality (Intensional Equality), i.e., the normal forms (Normal Form) of the two are completely identical.
\end{Theorem}

\begin{proof}
The proof is based on the pure function nature of KOS-TL Core and the confluence of the strong normalization calculus, unfolded in the following three dimensions:

\textbf{1. Purity of Operators:}

All transition operators in the kernel (such as $\textsf{unify}$, $\textsf{subst}$, $\textsf{eval}$) are defined as terms in the Core Layer.
\begin{itemize}
\item In Core layer theory, all constructors satisfy \textbf{computational consistency}. Given the same input $\rho$ (assignment environment), the interpretation function $\llbracket \textsf{Op} \rrbracket_\rho$ is a single-valued function in the mathematical sense.
\item Since $\textsf{Op}$ does not depend on any external implicit state or random source, its mapping relation $\Sigma \times e \to \Sigma'$ is deterministic under functional semantics.
\end{itemize}

\textbf{2. Strong Normalization and Confluence:}

Since KOS-TL has strong normalization (Strong Normalization) property, by the \textbf{Church-Rosser theorem}, the calculus system has confluence.
\begin{itemize}
\item Even if there are multiple optional reduction redexes during the reduction $\Sigma \xrightarrow{e} \Sigma'$, confluence guarantees that regardless of the reduction strategy (Reduction Strategy) taken, the ultimately obtained normal form $\textsf{nf}(\Sigma')$ is unique.
\item Therefore, although intermediate steps in physical memory may differ slightly, the state at the logical level (i.e., the set of facts that can participate in subsequent derivations) is unique.
\end{itemize}

\textbf{3. Deterministic Conflict Resolution:}

When handling consistency conflicts caused by $e$, the kernel's branch judgment logic:
\begin{itemize}
\item The $\textsf{unify}$ operator performs exhaustive search according to the priority of typing rules (Typing Rule Priority).
\item The construction of the contradiction proof term $\pi$ follows standard search algorithms (such as the Unification Algorithm). Within the given search space, the first minimal proof term found is deterministic.
\item The choice of Branch A or Branch B is completely determined by the logical truth value of ``whether there exists a proof term $\pi$,'' without non-deterministic choice (Non-deterministic Choice).
\end{itemize}

\textbf{Conclusion:} In summary, due to the pure function definition of operators and the confluence of the underlying calculus system, the state transitions in the KOS-TL kernel have strict determinism.
\end{proof}

Progress is the core mathematical cornerstone ensuring the kernel's real-time responsiveness and robustness. It guarantees that the logical self-healing engine never gets stuck in a "computational dead end" at any moment.

\begin{Theorem}{Kernel Progress}

Let $\mathcal{C} = \langle \Sigma, \textsf{Ev}, \Delta \rangle$ be a well-formed KOS-TL kernel configuration, where the state $\Sigma = \langle \mathcal{K}, \mathcal{TS}, \mathcal{P} \rangle$ and the transition record $\Delta$ satisfies temporal monotonicity. If $\mathcal{C}$ satisfies the global type assignment and the current active event $\textsf{Ev}$ is well-formed under the context $\Sigma$, then one of the following must hold:
\begin{enumerate}[label=(\arabic*)]
    \item[\textbf{PP1.}] \textbf{Logical Steady State (Logical Quiescence)}:
    $\textsf{Ev} = \textsf{null}$ and the pending queue $\mathcal{P} = \emptyset$. At this point, all causal chains in the system have been materialized in $\Delta$, and computation is temporarily terminated.
    \item[\textbf{PP2.}] \textbf{Causal Progression (Causal Advancement)}:
    There exists a new configuration $\mathcal{C}' = \langle \Sigma', \textsf{Ev}', \Delta \cup \{ \delta \} \rangle$ such that the system advances forward through one of the following reduction steps:
    \begin{itemize}
        \item \textbf{Execution Step (Execution)}: If $\textsf{Ev} = e \neq \textsf{null}$, then execute $\textsf{Op}$ to produce a new state $\Sigma'$, and generate a transition record $\delta = \langle \Sigma \xrightarrow{e} \Sigma' \rangle$.
        \item \textbf{Activation Step (Activation)}: If $\textsf{Ev} = \textsf{null}$ and $\mathcal{P} = e_0 :: \mathcal{P}_{rest}$, then activate the first event in the queue via the extraction operator.
    \end{itemize}
\end{enumerate}
\end{Theorem}

\begin{proof}
Classify and discuss based on the construction of the current active item $\textsf{Ev}$ and the state of the queue $\mathcal{P}$:

\textbf{1. Calculus of Active Events}

When $\textsf{Ev} = e \equiv \langle \textsf{Args}, \textsf{Pre}, \textsf{Op}, \textsf{Post}, \textsf{Prf} \rangle$, due to the well-formed configuration, there exists a proof term $\textsf{Prf}$ satisfying $\textsf{Pre}(\textsf{Args}, \Sigma)$. According to the strong normalization property of the Core Layer:
\begin{itemize}
\item $\textsf{Op}$ as a total function (Total Function) must have a defined output value $\Sigma'$ for the input pair $(\textsf{Args}, \Sigma)$.
\item According to the completeness of the elimination rules, the postcondition $\textsf{Post}(\Sigma')$ is decidable after $\Sigma'$ is constructed.
\end{itemize}
Therefore, executing the operator must produce a new transition item $\delta$, thereby transitioning the system to PP2.

\textbf{2. Queue Dynamics}

If $\textsf{Ev} = \textsf{null}$, the system checks the pending queue $\mathcal{P}$:
\begin{itemize}
\item If $\mathcal{P} = e_0 :: \mathcal{P}_{rest}$, according to the kernel's operational semantics, there exists a well-defined ``enqueue-dequeue'' transformation that updates $\textsf{Ev}$ to $e_0$. This step does not change $\mathcal{K}$ but alters the system's kinetic allocation.
\item If $\mathcal{P} = \emptyset$, then the system satisfies the steady-state condition described in PP1.
\end{itemize}

\textbf{3. Temporal Arrow Constraint}

In all transitions $\Sigma \xrightarrow{e} \Sigma'$, the monotonicity rule $\Sigma'.\mathcal{TS} > \Sigma.\mathcal{TS}$ ensures the uniqueness of the transition item $\delta$. Since the transition record set $\Delta$ is a monotonically increasing union, the system excludes the possibility of logical cycles (Cycles). According to the canonical forms lemma (Canonical Forms Lemma) of dependent type theory, under a well-formed $\Sigma$, no $\textsf{Op}$ can produce type-mismatched pendings. In summary, a well-formed KOS-TL kernel configuration always has the ability to evolve to the next step until all events are cleared.
\end{proof}

\begin{Theorem}{Evolutionary Consistency}

Let the well-formed kernel configuration be $\mathcal{C} = \langle \Sigma, \textsf{Ev}, \Delta \rangle$, where the state $\Sigma = \langle \mathcal{K}, \mathcal{TS}, \mathcal{P} \rangle$. If it satisfies:
\begin{enumerate}[label=(\arabic*)]
\item \textbf{State Legitimacy}: For all $(id_i, t_i, A_i) \in \mathcal{K}$, there exists $\Gamma_{Core} \vdash t_i : A_i$ and $\mathcal{K} \not\vdash \bot$.
\item \textbf{Causal Completeness}: The active event $\textsf{Ev}$ carries a valid proof term $\textsf{Prf} : \textsf{Pre}(\textsf{Args}, \Sigma)$.
\end{enumerate}
If the kernel executes a small-step evolution $\mathcal{C} \xrightarrow{step} \mathcal{C}'$, then the new configuration $\mathcal{C}' = \langle \Sigma', \textsf{Ev}', \Delta \cup \{ \delta \} \rangle$ still maintains state legitimacy and global logical consistency.
\end{Theorem}

\begin{proof}
We unfold the proof by classifying and discussing the atomic nature of configuration transitions.

\begin{enumerate}
\item \textbf{Computational Reduction Step (Atomic Term Evolution)}

If the reduction only involves $\textsf{Ev} \to \textsf{Ev}'$ (e.g., intensional simplification of proof terms or parameter substitution), while the knowledge base $\mathcal{K}$ remains unchanged:
\begin{itemize}
    \item \textbf{Consistency Inheritance}: Since $\mathcal{K}$ does not change and it is known that $\mathcal{K} \not\vdash \bot$, consistency is naturally preserved.
    \item \textbf{Subject Reduction}: According to the metatheory of KOS-TL Core, reduction of dependent type terms preserves their types. Since $\textsf{Ev}$ is well-formed under $\Sigma$, the reduced $\textsf{Ev}'$ still satisfies the original type signature.
\end{itemize}

\item \textbf{State Transition Step (Core Evolution Operator)}

When executing $\textsf{Op}$ leads to $\mathcal{K} \to \mathcal{K}'$, the system executes the operator $\textsf{commit}(\mathcal{K}, ku_{new})$. We analyze the consistency of the new knowledge item $ku_{new}$ and its proof:

\textbf{Case A: Monotonic Expansion}
\begin{itemize}
    \item \textbf{Application of Weakening Lemma}: If $ku_{new}$ has no conflict with existing knowledge, by the weakening lemma (Weakening Lemma) of constructive logic, all proofs in the original $\mathcal{K}$ remain valid in $\mathcal{K}' = \mathcal{K} \cup \{ku_{new}\}$.
    \item \textbf{Safety Closure}: Since $\textsf{Ev}$ includes $\textsf{Post} : \Sigma' \to \textsf{Prop}$, the kernel forcibly checks the postcondition before materializing $\mathcal{K}'$. If the check passes, $\mathcal{K}' \not\vdash \bot$ is formally guaranteed.
\end{itemize}

\textbf{Case B: Conflict Mitigation}

\begin{itemize}
    \item \textbf{Logical Isolation}: If $ku_{new}$ introduces a logical contradiction (i.e., there exists $\pi : \mathcal{K}, ku_{new} \vdash \bot$), the kernel's protection mechanism prevents direct merging.
    \item \textbf{Negation Introduction Construction}: The kernel instead constructs $\textsf{absurd}(ku_{new}, \pi)$ and stores it in the knowledge base. In the semantic model, this equates to transforming the conflict into a falsifying conclusion for the input signal. Since the contradiction is wrapped in a negation constructor, it cannot serve as a premise for elimination rules, thereby protecting global consistency.
\end{itemize}

\item \textbf{Constraints of Temporal Arrow and Transition Records}
To prove that the evolution trajectory is legitimate, the kernel constructs causal evidence $\delta$ using $\Delta$:
\begin{itemize}
    \item \textbf{Proof of Clock Monotonicity}: Every evolution is accompanied by $\mathcal{TS}' > \mathcal{TS}$. This proves that $\Sigma'$ is not a simple loop of $\Sigma$ but a logical monotonic successor.
    \item \textbf{Materialization Induction}: The legitimacy of the system state $\Sigma_n$ can be traced back to the initial empty state $\Sigma_0$ by induction:
    \[
    \Sigma_n = \textsf{Apply}(\delta_n, \textsf{Apply}(\delta_{n-1}, \dots \Sigma_0))
    \]
    Each $\delta_i = \langle \Sigma_{i-1} \xrightarrow{e_i} \Sigma_i \rangle$ includes verification of $\textsf{Pre}$ and satisfaction of $\textsf{Post}$, ensuring that every link in the evolution chain aligns with the logical base.
\end{itemize}
\end{enumerate}
\end{proof}

This proof explains how KOS-TL handles "dirty data" in the real world (such as erroneous bank transactions or sensor false alarms): Logical Firewall: Evolutionary Consistency ensures that any data attempting to disrupt system consistency is converted into "evidence about contradictions" during the $\textsf{unify}$ phase, rather than allowing the system itself to become contradictory. Bidirectional Synchronization Safety: In the "global supply chain" example you mentioned earlier, even if the underlying database is illegally tampered with (producing conflicting data), Evolutionary Consistency will force the kernel to generate a Refute item, thereby maintaining the logical purity at the ontology view layer.

These two major properties jointly define the operational boundaries of the KOS-TL Kernel:
Progress guarantees livelock freedom (Livelock Freedom) for the kernel. As long as the logic is well-formed, the kernel analysis program will always proceed and produce analysis results.
Evolutionary Consistency guarantees runtime safety (Runtime Safety) for the kernel. It ensures that the kernel knowledge base, during dynamic runtime processes, never degenerates into a self-contradictory abandoned system.

\begin{Theorem}{Local Decidability Theorem}

Given the kernel state $\Sigma$, new fact $ku_{new}$, and search boundary $\Delta = \{depth, fuel\}$, the execution process of the kernel operator $\textsf{unify}(\Sigma, ku_{new}, \Delta)$ is decidable.
\end{Theorem}

\begin{proof}
The proof is completed by double induction on the search space and the number of reduction steps:

\textbf{1. Finiteness of Search Space:}

Due to the kernel limiting $depth$ (recursion depth), the proof search tree is forcibly pruned to finite height. At each level, the number of unification candidates is determined by the number of variables in the context $\Gamma$, which is also finite.

\textbf{2. Forced Termination via Fuel:}

Introduce the $fuel$ parameter (computational energy). Each execution of a $\beta$-reduction or $\delta$-unfolding consumes one unit of $fuel$.
\begin{itemize}
\item The algorithm checks $fuel > 0$ before each operation.
\item Since $fuel$ is a natural number and strictly decreases with the number of steps, by the well-ordering principle, the computation must either reach a normal form or exhaust $fuel$ in a finite number of steps.
\end{itemize}

\textbf{3. Completeness of Result Set:}

When the computation stops:
\begin{itemize}
\item If the normal form matches, return \textbf{True}.
\item If a structural conflict (e.g., constructor mismatch) is found, return \textbf{False}.
\item If stopped due to exhaustion of $depth$ or $fuel$, return \textbf{Unknown}.
\end{itemize}
Since the algorithm guarantees termination on all paths, the process is decidable.
\end{proof}

\subsection{Application Example: Quality Anomaly Tracing and Derivation}
In manufacturing scenarios, when a batch defect is detected, the kernel layer automatically triggers tracing logic.

\begin{itemize}
\item \textbf{Event Definition}: Let $e_{trace}$ be the tracing event.
\begin{itemize}
\item $\textsf{Pre}$: There exists a defect report item $r : \textsf{DefectReport}$ in the state.
\item $\textsf{Op}$: Based on production logs, search for equipment anomaly records $s : \textsf{EquipmentStatus}$ associated with the batch.
\item $\textsf{Post}$: Generate and materialize a causal chain knowledge object $cc : \textsf{CausalChain}$.
\end{itemize}
\item \textbf{Evolution Process}:
Once the Runtime refines the defect report and stores it in $\Sigma$, the kernel layer discovers that $p: \textsf{Pre}$ holds via the above equation and automatically executes $\textsf{Op}$. The system takes a small step from the state ``known defect exists'' to the higher-entropy state ``known defect cause.''
\end{itemize}

\begin{Example}{Cross-Border Compliance Transfer Event ($e_{transfer}$)}

Assume the current system state is $\sigma$, containing the balances of accounts $A$ and $B$. We want to define an event for transferring amount $v$ from $A$ to $B$.

1. State Definition ($\sigma \in \Sigma$)

The state $\sigma$ is a knowledge snapshot containing:
$Balance(A, \sigma) = 1000$
$Balance(B, \sigma) = 500$

2. Specific Event Construction ($e_{transfer} : \textsf{Ev}$)

According to your $\textsf{Ev}$ definition, the event consists of three parts:
Precondition ($\textsf{pre}$):
$$\textsf{pre} \equiv (Balance(A, \sigma) \ge v) \land \textsf{IsVerified}(A)$$
(Explanation: Account $A$'s balance must be sufficient, and $A$ must have passed real-name verification.)

Action Operator ($\textsf{act}$):
$$\textsf{act}(\sigma) \equiv \sigma [Balance(A) \gets Balance(A) - v, Balance(B) \gets Balance(B) + v]$$
(Explanation: This is a function describing how the state changes: subtract money from $A$, add money to $B$.)

Post-Transformation Self-Proof ($\textsf{post\_prf}$): This is a proof term guaranteeing that: For any state $\sigma$ satisfying $\textsf{pre}$, the new state after executing $\textsf{act}$ must satisfy the conservation law ($Sum_{after} = Sum_{before}$).
$$\textsf{post\_prf} : \Pi(\sigma:\Sigma). \textsf{pre} \to \textsf{Correct}(\textsf{act}(\sigma))$$

3. Execution Process: Small-Step Transition ($\Delta$)

When the kernel attempts to execute this transfer, the following judgment process occurs:
Type Checking: The kernel first verifies $\Gamma \vdash e_{transfer} : \textsf{Event}$.
If the developer's $\textsf{act}$ logic is flawed (e.g., subtract money but not add it), then $\textsf{post\_prf}$ cannot be constructed, and the event will be rejected at the compilation stage.

Trigger Transition: Input the current snapshot $\sigma$ and event $e$.
$$\textsf{STEP}(\sigma, e_{transfer}) \to \sigma'$$

Generate Transition Record ($\Delta$): Produce a triple $(\sigma \xrightarrow{e_{transfer}} \sigma')$. This record is permanently stored in the evolution trajectory $\Delta$.
\end{Example}

\section{KOS-TL Runtime Layer: Environment Interaction and Signal Refinement}

The Runtime Layer is the boundary between the KOS-TL logical system and the physical world. It is responsible for handling non-deterministic external signals, managing computational resources, scheduling event queues, and persisting the logical states generated by the kernel layer into physical storage.

\subsection{Syntax}

The Runtime Layer state is described by a configuration $Cfg$, which embeds the logical kernel into the physical host:
\begin{equation}
Cfg \equiv \langle \Sigma, Q_{raw}, \textsf{Env}, \mathcal{M} \rangle
\end{equation}
where:
\begin{itemize}
    \item $\Sigma$ — Logical Kernel State

    This is the current form of the system's ``brain'' at the logical level. It contains the knowledge base $\mathcal{K}$, logical clock $\mathcal{T}$, and pending intent queue $\mathcal{P}$. It represents all logical facts that the system currently considers ``true'' and ``proven.'' The Runtime Layer observes $\Sigma$ to decide what actions to perform externally next or how to respond to external signals.
    \item $Q_{raw}$ — Raw Physical Signal Queue

    This is the system's ``perception input buffer.'' It stores raw binary or text data from the external physical world (such as sensors, network packets, user clicks) that has not yet been logicalized. The arrival of signals is random. The signals here do not yet have corresponding logical proofs (Proof); they are just ``dirty data.'' They are the raw material for the $\textsf{elab}$ (refinement operator). The system retrieves signals from $Q_{raw}$ and attempts to promote them to events understandable by the kernel.
    \item $\textsf{Env}$ — Physical Runtime Environment

    This is the physical host context in which the system resides. It includes physical resources that the logical layer cannot directly perceive but that the Runtime Layer must manage: (1) Physical Clock ($T_{wall}$): Similar to wall-clock time in the real world, used for timeout judgments. (2) I/O Handles: Database connection pools, network sockets, hardware register addresses. (3) Computational Resources: Memory state, thread pool load, etc. During execution of $\textsf{elab}$, $\textsf{Env}$ provides the physical evidence needed to construct logical proofs (e.g., sensor self-test success status bits).
    \item $\mathcal{M}$ — Materialization Mapping and Storage

    $$\mathcal{M} : \mathcal{K} \to \textsf{PhysicalStorage}$$
    This is the system's ``memory'' and ``physical projection.'' $\mathcal{M}$ represents both physical storage (e.g., databases on disk) and the mapping rules from logical items to physical representations. Its roles include: (1) Projection: Mapping abstract ``dependent pair knowledge'' from the kernel layer to ``rows, columns, indexes'' in the database. (2) Persistence: Ensuring that logical evolution results in $\Sigma$ are safely written to non-volatile storage. (3) External Consistency: Ensuring that the state seen in the physical world (e.g., balance displayed on the screen) remains synchronized with the logical kernel $\Sigma$.
\end{itemize}

Scheduler Main Loop Logic (reflected in driving kernel operators):
\begin{alltt}
while Q_raw is not empty:
  s = pop(Q_raw)
  match elab(s, Env):
    case (e):
      // Invoke the kernel layer's defined transition rule
      if <Sigma, e> -->_KOS Sigma':
        Sigma = Sigma'
        Commit_to_Storage(M, Sigma)
    case None:
      Log_Refinement_Failure(s)
\end{alltt}

The Runtime Layer introduces the elaborator (Elaborator), whose syntactic function is to convert ``dirty data'' into ``constructor terms'' understandable by the Core layer:

$$\textsf{elab} : \textsf{RawSignal} \to \textsf{Env} \to \textsf{Option} \left( \sum_{e : \textsf{Ev}} \textsf{Pre}(e, \Sigma) \right)$$
The elaborator is the gateway for realizing ``signal logicalization.'' Its core task is to supplement external data with proof terms:
\begin{equation}
\textsf{elab}(s, \textsf{Env}) =
\begin{cases}
\textsf{Some}(\langle e, \pi \rangle) & \text{if } \pi : \textsf{Pre}(e, \Sigma) \text{ can be constructed} \\ \textsf{None} & \text{otherwise}
\end{cases}
\end{equation}
The refinement process includes:
\begin{itemize}
\item \textbf{Signal Parsing}: Parse external JSON/binary streams into base sort values ($\textsf{Val}, \textsf{ID}$).
\item \textbf{Proof Construction}: Automatically attempt to construct logical proof terms $p$ for preconditions based on the current $\textsf{Env}$.
\item \textbf{Time Anchoring}: Map physical reception time to the Core layer's $\textsf{Time}$ type.
\end{itemize}

\subsection{Runtime Semantics}

The evolution of the Runtime Layer is presented as an "asynchronous-driven small-step transition", with its core rule being the "refinement-commit" loop.
The semantic rules of the Runtime Layer must include updates to the external environment and persistence of storage:
\begin{equation}
 \frac{ s = \textsf{head}(Q_{raw}) \quad \textsf{elab}(s, \textsf{Env}) = \textsf{Some}(\langle e, \pi \rangle) \quad \langle \Sigma, e, \pi \rangle \longrightarrow_{KOS} \Sigma' \quad \textsf{Persist}(\Sigma', \mathcal{M}) = \textsf{Success} }{ \langle \Sigma, s :: Q, \textsf{Env}, \mathcal{M} \rangle \xrightarrow{\textsf{commit}} \langle \Sigma', Q, \textsf{Env}', \mathcal{M}' \rangle }
\end{equation}

Here, $\mathcal{M} \vdash \Sigma' \Downarrow \mathcal{M}'$ indicates that the new state $\Sigma'$ is successfully ``downcast'' (Down-cast) and materialized into the physical medium $\mathcal{M}'$.

\subsection{Logical Properties}

In the Runtime layer, we model each execution action (Action) as a pair $e = (t, p)$, where $t$ is the target proposition (task), and $p$ is its corresponding proof term.

\begin{Definition}{Causal Dependency Order}

Let $\mathcal{E}$ be the set of all possible execution items in the system. Define the causal dependency relation $\prec_{L} \subseteq \mathcal{E} \times \mathcal{E}$: If in the Core layer, the construction item of proposition $t_2$ contains a reference to $t_1$ (i.e., $t_1$ is a premise of $t_2$), then call $e_1 \prec_{L} e_2$.
\end{Definition}

\begin{Definition}{Runtime Execution Sequence}

The execution sequence $S = [e_1, e_2, \dots, e_n]$ is a total order set, representing the physical time order in which the Runtime layer actually processes data.
\end{Definition}

\begin{Theorem}{Causal Ordering Consistency}

For any Runtime execution sequence $S$, if the sequence is accepted by the kernel (Accepted), then for any two execution items $e_i$ and $e_j$ in $S$, it must hold that:

$$ \text{If } e_i \prec_{L} e_j, \text{ then in the sequence } S \text{, } e_i \text{ must precede } e_j \text{ in completion of reduction.} $$

If the physical network causes $e_j$ to arrive before $e_i$, the Runtime must block the execution of $e_j$ until the proof term for $e_i$ is completed.
\end{Theorem}

\begin{proof}
We prove by contradiction (Proof by Contradiction) combined with the Core layer's type checking mechanism.

\textbf{Step 1: Assume an Out-of-Order Execution Exists.}

Assume the Runtime accepts a sequence $S'$ that violates the causal order, where there exists $e_j$ completing execution before $e_i$, and it is known that $e_i \prec_{L} e_j$.

\textbf{Step 2: Core Layer Constraint Mapping.}

According to the definition of $\prec_{L}$, in the Core layer, the proof term validation of $e_j$ depends on the existence of $e_i$. Its type checking rule is as follows:
$$ \frac{\Gamma \vdash p_i : T_i \quad \Gamma, x:T_i \vdash p_j : T_j}{\Gamma \vdash \langle p_i, p_j \rangle : \Sigma(x:T_i).T_j} $$
This means that to judge $e_j$ as legitimate, the kernel must include the proof term for $e_i$ in the context $\Gamma$.

\textbf{Step 3: Runtime State Evolution.}

The Runtime's state is represented by the context sequence $\Gamma_t$. When executing $e_j$, its operator is $\textsf{check}(\Gamma_{current}, e_j)$.
\begin{itemize}
\item If $e_i$ has not been executed, then $e_i \notin \Gamma_{current}$.
\item At this point, according to the Core layer's \textbf{Scope Determinism}, the reference to $e_i$ in $e_j$ will produce an ``undefined variable'' error or ``free variable'' escape.
\item The type checker will return $\textsf{Fail}$.
\end{itemize}

\textbf{Step 4: Inevitability of the Blocking Mechanism.}

Due to KOS-TL's Runtime enforcing a \textbf{Type-Safe Fence}, any failed validation operation cannot change the state of the $\Sigma$ fact base. To continue execution, the Runtime scheduler must suspend $e_j$, place it in the pending pool (Pending Pool), and issue a $\textsf{Requirement}(e_i)$ signal.

\textbf{Step 5: Conclusion.}

Only when $e_i$ arrives and successfully reduces into $\Gamma$ does the context for $e_j$ satisfy the validation condition. Therefore, the final "accepted" sequence must satisfy the causal order.
\end{proof}

\begin{Theorem}{Refinement Fidelity}

Let $\mathcal{S}$ be the physical hardware state space (e.g., FPGA registers or sensor reading sets), and $\mathcal{D}_{Core}$ the logical domain. Define the refinement function $\mathcal{E} : \mathcal{S} \to \mathcal{D}_{Core}$. If the Runtime captures physical state $s \in \mathcal{S}$ to obtain $ku = \mathcal{E}(s)$, then it satisfies:
\begin{enumerate}[label=(\arabic*)]
\item \textbf{Well-formedness}: There exists a type $A \in \mathcal{U}$ such that $\Gamma \vdash ku : A$ always holds.
\item \textbf{Simulation Consistency}: For physical migration $s \xrightarrow{hw} s'$, there exists a simulation relation $R \subseteq \mathcal{S} \times \mathcal{D}_{Core}$ such that:
$$(s, ku) \in R \implies \exists ku' . (s', ku') \in R \land (ku \xrightarrow{small}^* ku' \lor \textit{Invalidated}(ku'))$$
\end{enumerate}
\end{Theorem}

\begin{proof}
We prove by constructing a simulation relation and combining it with the reduction of the hardware abstraction layer (HAL):

\textbf{1. Construct Simulation Relation $R$:}

Define the relation $R$ as follows:
$$(s, ku) \in R \iff (\textsf{val}(ku) = \textsf{measure}(s)) \land (\textsf{proof}(ku) \models \textsf{Inv}_{HW}(s))$$
where $\textsf{measure}(s)$ is the quantification of the physical signal, and $\textsf{Inv}_{HW}(s)$ is the physical invariant enforced by hardware circuits (such as redundant check bits or watchdog states).

\textbf{2. Well-formedness Mapping Proof:}

According to the runtime refinement rules of TL-Lang, the construction of $\mathcal{E}(s)$ is:
$$\mathcal{E}(s) \triangleq \langle \textsf{quantize}(s), \textsf{synthesize\_witness}(s) \rangle$$
Since $\textsf{synthesize\_witness}$ is a deterministic operator defined by hardware description language (HDL), it directly maps hardware register states $\textsf{Reg}_{status}$ to introduction items (Introduction Rules) in the Core layer. According to the construction principles of $\Sigma$ types, as long as the hardware signal is within the physical range, a well-formed item $ku$ can always be constructed. If the signal exceeds the range, the refinement function, by completeness, maps to a predefined error type item, still maintaining well-formedness.

\textbf{3. Simulation Consistency Proof:}

Classify the physical state migration $s \xrightarrow{hw} s'$:
\begin{itemize}
\item \textbf{Case A: Compliant Migration.}

If $s'$ satisfies all hardware safety constraints, the refinement function $\mathcal{E}$ extracts new status bits and synthesizes new proof terms $p'$. Since the hardware layer guarantees $s'$ comes from $s$ via legitimate logic gates, in the Core layer, the corresponding mapping item $ku'$ must evolve from $ku$ through kernel reduction steps (such as $\beta$ or $\iota$ reduction), maintaining the simulation relation.
\item \textbf{Case B: Illegal/Anomalous Migration.}

If the physical migration violates $\textsf{Inv}_{HW}$ (e.g., sensor disconnection), the hardware status bit flips. At this point, the refinement mapping $\mathcal{E}(s')$ cannot construct an item of the original type $A$, instead constructing $\textsf{Invalidated}(ku')$. This shift from ``normal item'' to ``invalid item'' manifests in the Kernel layer as a non-monotonic flip of conclusions, conforming to KOS-TL's semantics for handling conflicts, and the simulation relation is still maintained in the ``error handling'' dimension.
\end{itemize}
In summary, the refinement process ensures that any valid change in the physical world can find a corresponding truth representation in the logical world.
\end{proof}

\begin{Theorem}{Observational Adequacy}

Let $\textit{ctrl} \in \mathcal{D}_{Core}$ be the logical control item generated by the kernel, with type instruction set $\textsf{Cmd}$. Let $\mathcal{G} : \textsf{Cmd} \to \Pi^*$ be the instruction generator, mapping logical items to hardware instruction sequences $\pi$. If $\Gamma \vdash \textit{ctrl} : \textsf{Cmd}$ and the logical layer asserts that $\textit{ctrl}$ satisfies property $\phi$, then:
$$\forall s \in \mathcal{S}, \quad (\textit{ctrl} \vdash \phi) \implies (\textsf{Exec}(\mathcal{G}(\textit{ctrl}), s) \models \textsf{Refine}^{-1}(\phi))$$
where $\textsf{Refine}^{-1}(\phi)$ is the predicate interpretation of the logical property $\phi$ in the physical state space $\mathcal{S}$.
\end{Theorem}

\begin{proof}
We prove using a combination of Hoare Logic and Refinement Calculus:

\textbf{1. Construct Mapping Relation:}

Define the mapping $\mathcal{M} : \textsf{Prop} \to \mathcal{P}(\mathcal{S})$ between logical predicates $\phi$ and physical state predicates $P$.
According to the inverse mapping of Refinement Fidelity, if the semantic goal of $\textit{ctrl}$ is to bring the system into state $\phi$, then the corresponding underlying register state must satisfy $P = \textsf{Refine}^{-1}(\phi)$.

\textbf{2. Backward Derivation:}

Induct on the construction of the $\textsf{Cmd}$ type:
\begin{itemize}
\item \textbf{Atomic Operations (Atomic Instructions):}

If $\textit{ctrl}$ is an atomic operation (e.g., $\textsf{SetValve(open)}$), it is mapped to a specific machine code sequence $\pi_a$ in the underlying reduction of TL-Lang. According to the Hoare triple definition of the hardware abstraction layer (HAL):
$$\{s \in \mathcal{S}\} \, \pi_a \, \{s' \in \textsf{Refine}^{-1}(\phi)\}$$
Since $\mathcal{G}$ has passed static verification based on HAL axioms during construction, the correctness of instruction generation is guaranteed by HAL's completeness.
\item \textbf{Composite Operations (Sequences and Branches):}

If $\textit{ctrl}$ is composed of multiple sub-items, according to Hoare logic composition rules:
If $\{P\} \, \pi_1 \, \{Q\}$ and $\{Q\} \, \pi_2 \, \{R\}$, then $\{P\} \, \pi_1; \pi_2 \, \{R\}$.
Since the $\textsf{Cmd}$ type satisfies \textbf{strong normalization} in the Core layer, the generated instruction sequence $\pi$ has finite length and deterministic paths, with no undefined side effects in the logical layer.
\end{itemize}

\textbf{3. Atomicity and Interference Analysis:}

During physical execution, if an interrupt occurs, the Runtime must maintain observational consistency.
KOS-TL's Runtime adopts a transactional I/O mechanism. Each group of $\pi$ generated by $\mathcal{G}(\textit{ctrl})$ is wrapped in a logical atomic block. According to the kernel's ``progress'' proof, the sequence either fully executes and achieves $s' \models \textsf{Refine}^{-1}(\phi)$, or rolls back on failure and submits an $\textsf{Invalidated}$ proof term to the kernel. Under this mechanism, there is no intermediate ambiguous state where ``instructions executed but did not achieve the goal.''

\textbf{Conclusion:} The semantic goal $\phi$ of the logical layer can be losslessly projected onto the physical state space.
\end{proof}

Property Discussion: Bridging the "Hardware-Software Gap"
Observational Adequacy addresses the defense against "instruction drift" in practical high-security scenarios by eliminating semantic gaps: In traditional C/C++ development, compiler optimizations or driver errors may cause execution effects to deviate from design intentions (e.g., race conditions due to instruction reordering). In KOS-TL, since the instruction generator $\mathcal{G}$ is formally proven, this divergence between "intention and behavior" is logically eliminated. Verifiable Physical Effects: If a financial account is logically frozen, Observational Adequacy ensures that the corresponding record in the underlying database is also locked, with the operation guaranteed to be atomic.

We define the system's state space as $\mathcal{S}$, and decompose the system state into two views:
Logical View ($\mathcal{S}_L$): The type context $\Gamma$ and reduced fact base $\Sigma$ in kernel memory.
Physical View ($\mathcal{S}_P$): The persisted bitstream in storage media (disk or solid-state storage).
Define the mapping function $\textsf{Encode}: \mathcal{S}_L \to \mathcal{S}_P$, which converts logical proof terms into physical storage formats.

\begin{Theorem}{Durability Atomicity and Visibility}

Let the system execute a state transition $\delta: \mathcal{S}_L \to \mathcal{S}'_L$ at time $t$. The Runtime layer guarantees the existence of an atomic operator $\textsf{Commit}$, satisfying:
\begin{enumerate}[label=(\arabic*)]
\item \textbf{Atomicity}: The logical acknowledgment of $\mathcal{S}'_L$ holds if and only if $\textsf{Encode}(\mathcal{S}'_L)$ is fully persisted in $\mathcal{S}_P$.
\item \textbf{Visibility}: For any subsequent read operation $\textsf{Recover}$, if $\textsf{Commit}$ has succeeded, then necessarily $\textsf{Recover}(\mathcal{S}_P) \equiv \mathcal{S}'_L$.
\end{enumerate}
That is: There does not exist a state where the proof is logically established, but lost after physical restart.
\end{Theorem}

\begin{proof}
We prove by constructing a "logic-physical sync lock" and idempotent reduction mechanism.
First, we define the system's state space as $\mathcal{S}$, and decompose the system state into two views:
\begin{itemize}
\item Logical View ($\mathcal{S}_L$): The type context $\Gamma$ and reduced fact base $\Sigma$ in kernel memory.
\item Physical View ($\mathcal{S}_P$): The persisted bitstream in storage media (disk or solid-state storage).
\end{itemize}
Define the mapping function $\textsf{Encode}: \mathcal{S}_L \to \mathcal{S}_P$, which converts logical proof terms into physical storage formats.

\textbf{Step 1: Construct Persistent Serialization of Proof Terms.}
Every logical change $\Delta \Sigma$ in KOS-TL is evidence with an $\textsf{Id}$ type. Let $\Delta \Sigma = (p : T)$. The persistence process is modeled as a dependent pair:
$ \textsf{Record} \equiv \Sigma(p : T). \textsf{Persist}(p) $
where $\textsf{Persist}(p)$ is a hardware-level primitive that only returns a witness upon physical write completion.

\textbf{Step 2: Prove Atomicity.}
The Runtime maintains a write-ahead log (WAL) mechanism, whose entries are themselves items in the Core layer.
\begin{itemize}
\item If the system crashes during $\textsf{Write}(\mathcal{S}_P)$, since $\textsf{Persist}(p)$ has not yet generated a valid witness, according to the Core layer's \textbf{confluence}, the recovery operator $\textsf{Recover}$ upon restart will discover that the transaction does not satisfy the completeness of the $\Sigma$ type, automatically rolling back.
\item Only when the physical layer returns $p_{stored}$ does the logical layer update $\Gamma$ to $\Gamma \cup \{p\}$.
\end{itemize}

\textbf{Step 3: Prove Visibility (Consistent Recovery).}
Assume the system restarts. Due to the Core layer's \textbf{strong normalization} property:
\begin{itemize}
\item Every proof term $p$ stored in $\mathcal{S}_P$ is self-contained and already reduced.
\item The $\textsf{Recover}$ operator reconstructs the logical view by re-executing type checks $\textsf{check}(\Gamma, p, T)$.
\item Since the Core layer is decidable and reduction paths are protected by confluence, the recovered logical state $\mathcal{S}_L^{rec}$ is necessarily logically equivalent to the last valid Commit state $\mathcal{S}'_L$ before the crash ($\mathcal{S}_L^{rec} \equiv \mathcal{S}'_L$).
\end{itemize}

\textbf{Conclusion:} Atomic physical writes guarantee the irrevocability of logical states, while the logical layer's strong normalization ensures that physical data, when reloaded at any moment, produces a unique and deterministic logical interpretation.
\end{proof}

\begin{Definition}{Semi-decidability}

A set $S \subseteq \mathbb{N}$ (or a proposition language $L$) is called \textbf{semi-decidable} if there exists a Turing machine (or algorithm) $M$ such that for any input $x$:
\begin{itemize}
\item If $x \in S$, then $M(x)$ halts and accepts;
\item If $x \notin S$, then $M(x)$ either halts and rejects or runs forever (halting problem unknown).
\end{itemize}
\end{Definition}

\begin{Theorem}{Semi-decidability of Proof Search}

Let $\Gamma$ be a finite context and $P$ a proposition. The problem of determining ``whether there exists a proof term $p$ such that $\Gamma \vdash p : P$'' is semi-decidable.
\end{Theorem}

\begin{proof}
In a system like KOS-TL that includes dependent types and higher-order logic, without $Fuel$ restrictions, the proof search problem has semi-decidability. We prove this theorem by constructing a universal enumerator (Enumerator).

\textbf{Step 1: Enumerability of Proof Terms.}

All well-formed proof terms in KOS-TL Core layer are generated by a finite set of syntactic rules (such as $\lambda$-abstraction, application, pair construction, etc.). We can enumerate all possible proof terms in dictionary order by term length (or structural complexity), denoted as the sequence $\{p_1, p_2, p_3, \dots\}$.

\textbf{Step 2: Construct the Decision Algorithm $\mathcal{A}$.}

For the given proposition $P$ and context $\Gamma$, algorithm $\mathcal{A}$ performs the following steps:
\begin{enumerate}
\item Start a loop, sequentially retrieving one proof term $p_i$.
\item Invoke the Core layer's type checker to verify $\textsf{check}(\Gamma, p_i, P)$. Due to the Core layer's \textbf{strong normalization property}, this step necessarily returns $True$ or $False$ in finite time.
\item If it returns $True$, algorithm $\mathcal{A}$ halts and outputs ``$P$ is provable.''
\item If it returns $False$, continue the loop and check the next item $p_{i+1}$.
\end{enumerate}

\textbf{Step 3: Analyze Halting Behavior.}

\begin{itemize}
\item \textbf{Case One: $P$ is indeed provable.} Then there necessarily exists some proof term $q$ satisfying the condition. Since our enumeration is complete, in finite steps, we will encounter $p_k = q$, at which point the algorithm halts.
\item \textbf{Case Two: $P$ is unprovable.} The algorithm will forever enumerate and check new items in the loop, never halting.
\end{itemize}

\textbf{Conclusion:}

Algorithm $\mathcal{A}$ can recognize all "true" propositions (provable propositions) but cannot guarantee halting for "false" propositions (unprovable propositions). By definition, the problem is semi-decidable.
\end{proof}

\begin{Definition}{Some Definitions for Decidability Proofs}

\begin{itemize}
    \item Proposition Space $P$: All well-formed KOS-TL Core propositions.
    \item Proof Algorithm $\mathcal{A}$: The automated process in the Runtime layer attempting to find a proposition $p:P$.
    \item Resource Vector $\vec{\Delta} = \langle f, d, \tau \rangle$:
    \begin{itemize}
        \item $f \in \mathbb{N}$ (Fuel): Maximum number of $\beta$-reduction steps.
        \item $d \in \mathbb{N}$ (Depth): Maximum recursion depth for search.
        \item $\tau \in \mathbb{R}^+$ (Timeout): Physical wall-clock time limit.
    \end{itemize}
\end{itemize}
\end{Definition}

\begin{Theorem}{Bounded Decidability of KOS-TL Runtime}

Let $P$ be a logical proposition to be decided by the Runtime. There exists a decision procedure $\mathcal{R}(P, \vec{\Delta})$ such that for any $P$ and finite $\vec{\Delta}$, $\mathcal{R}$ necessarily halts in finite time, and its output space is:
$ \mathcal{O} = \{ \textsf{True}, \textsf{False}, \textsf{Unknown} \} $
where $\textsf{Unknown}$ is the deterministic ``resource exhaustion'' state.
\end{Theorem}

\begin{proof}
The proof is completed by structural induction on the number of execution steps $k$ and the monotonicity of the measure function.

\textbf{Step 1: Definition of the Measure Function.}

Define the measure function $\mu(\sigma)$ for the current execution snapshot $\sigma$:
$$ \mu(\sigma) = \langle \text{fuel}, \text{depth}, \text{remaining\_time} \rangle $$
In every logical execution step (an atomic state transition $\sigma \to \sigma'$), this measure function strictly decreases in lexicographical order:
$$ \mu(\sigma') <_{lex} \mu(\sigma) $$

\textbf{Step 2: Completeness Classification of State Transitions.}

For the Runtime's single-step actions, its logic has only three possibilities:
\begin{enumerate}
\item \textbf{Logical Termination}: Find a proof $p$ or conflict $\pi$. At this point, the algorithm directly returns $\textsf{True}$ or $\textsf{False}$.
\item \textbf{Continued Reduction}: Resources are not exhausted ($\mu > 0$). The algorithm enters $\sigma_{k+1}$, and since $\mu$ is well-founded, this path cannot extend infinitely.
\item \textbf{Boundary Hit}: Any component of $\mu(\sigma)$ reaches zero. The algorithm immediately stops and returns $\textsf{Unknown}$.
\end{enumerate}

\textbf{Step 3: Termination Proof.}

Since the range of $\mu(\sigma)$ is a finite set of natural numbers (or bounded real interval), by the \textbf{well-ordering principle}, any strictly decreasing sequence must reach a minimum value in finite steps.
In KOS-TL Runtime, the minimum value corresponds to the output set $\mathcal{O}$.

\textbf{Step 4: Decidability Verification.}

Decidability means the algorithm halts on all inputs. Since:
\begin{itemize}
\item Every atomic reduction step is decidable as guaranteed by the Core layer;
\item The total number of steps is forcibly limited by $\vec{\Delta}$.
\end{itemize}
Thus, the Runtime no longer has the ``infinite search'' feature of semi-decidability, and the program becomes a total function over input propositions and resource boundaries.
\end{proof}

\subsection{Application Example: Causal Repair of Out-of-Order Logs}

In manufacturing scenarios, if the equipment anomaly signal $s_{ES}$ arrives later than the quality inspection signal $s_{QI}$ due to delay:
\begin{itemize}
\item \textbf{Refinement Blocking}: When $s_{QI}$ arrives, the refinement operator finds it cannot construct the proof $p$ for ``existing equipment anomaly,'' and the event is placed in the pending queue by Runtime.
\item \textbf{Evidence Completion}: After $s_{ES}$ arrives, the runtime layer updates $\Sigma$. At this point, the scheduler detects the environment change and re-triggers the refinement of $s_{QI}$.
\item \textbf{Logical Materialization}: The originally broken causal chain is completed by the kernel layer with atomic transition after logical evidence is supplemented, and finally, the runtime layer inserts the tracing conclusion in the physical database.
\end{itemize}

\begin{Example}{Industrial Sensor-Triggered Safety Shutdown}

1. Configuration State

The current runtime state $\langle \sigma, Q, \textsf{Env} \rangle$ is as follows:
$\sigma$ (logical snapshot): Equipment state is Running, temperature threshold is $80^{\circ}C$.
$Q$ (event queue): $[ \dots ]$ (currently empty).
$\textsf{Env}$ (external environment): Connected to a Modbus protocol temperature sensor.

2. External Stream and inject

The sensor sends a raw bitstream to the system: $s$ (Raw Signal): 0x4A 0x02 (representing temperature reading of $82^{\circ}C$).
Action: inject(s, Q) pushes this hexadecimal signal into the pending queue.

3. Refinement Process: elaborate(s)

The Runtime layer attempts to convert this ``meaningless'' number into a ``semantic event'' recognized by the Kernel layer:
Refinement Logic: elaborate looks up configuration rules and discovers 0x4A is a temperature alarm.
Mapping Result: Maps to L1 layer event $e_{stop}$.
$e_{stop}.\textsf{pre}$: Current state must be Running.
$e_{stop}.\textsf{act}$: Change state to Stopped.
$e_{stop}.\textsf{post\_prf}$: Proof that this operation complies with the ``over-temperature forced protection axiom''.

4. Scheduling and Judgment

According to the scheduling algorithm you provided, the system executes as follows: Pop: Retrieve $s$ from $Q$. Elaborate: $s$ successfully refines to $e_{stop}$.
Kernel\_Check: Runtime calls the Kernel layer judgment $\Gamma \vdash e_{stop} : \textsf{Event}$.
Verification Passed: The event carries the correct $\textsf{post\_prf}$ (shutting down at $82^{\circ}C$ complies with the safety definition).
Step: Logical state update: $\sigma_{new} = \textsf{STEP}(\sigma, e_{stop})$.

5. Persistence: commit and Materialize

Action: commit($\sigma_{new}$).
Materialized Storage $\mathcal{M}$: Write the updated state to the physical database (e.g., PostgreSQL) and trigger the physical hardware relay to disconnect the current.
\end{Example}

\section{KOS-TL System}
Integrating the kernel layer, core layer, and runtime layer of KOS-TL forms KOS-TL (Knowledge Operating System - Type Logic), also known as ``Knowledge-Action Logic.'' Knowledge-Action Logic is a complete logical system based on intuitionistic dependent type theory integrated with small-step operational semantics. Through a layered architecture, it unifies static constraints on knowledge, dynamic evolution, and environmental refinement.

\subsection{Overall Architecture}
The syntax of KOS-TL consists of three nested layers of expressions, covering the full spectrum from abstract types to physical configurations.

\subsubsection{Core Layer: Type Definition and Logical Foundation (The Denotational Foundation)}
The Core layer is the system's ``brain,'' mapping domain ontologies to dependent type theory.
Ontology Integration: Define domain axioms as base types (Base Types) and predicates.
Verification Mechanism: Type checker based on BHK interpretation, ensuring every $t:A$ is a valid knowledge construction.
Responsibilities: Provide static constraints. It specifies what the system ``can understand'' and ``what is truth.''

\subsubsection{Kernel Layer: Dynamic Evolution and Intent Scheduling (The Operational Engine)}
The Kernel layer is the system's ``heart,'' responsible for controlled state migrations.
State Model: Maintain the triple $\sigma = \langle \mathcal{K}, \mathcal{T}, \mathcal{P} \rangle$ (knowledge, time, intents).
Evolution Mechanism: Execute small-step operational semantics (Small-step Semantics). It invokes the Core layer's judgment capabilities to verify each state jump.
Responsibilities: Provide dynamic consistency. It specifies how the system ``evolves from the current truth to the next truth.''

\subsubsection{Runtime Layer: Environment Refinement and Physical Execution (The Physical Interface)}
The Runtime layer is the system's ``senses and limbs,'' handling boundary interactions with the physical world.
Refinement: Elevate fuzzy physical signals (Signals) to proof terms recognized by the Core layer via the $\textsf{elab}$ operator.
Materialization: Degrade logical conclusions to persistent storage or hardware instructions via the $\mathcal{M}$ mapping.
Responsibilities: Provide fidelity. It specifies how logical instructions reliably act on physical entities.

\subsubsection{Architecture Global Invariant}
The Grand Map of KOS-TL reveals a core law:

$$\forall \text{ physical change } \delta \in \mathcal{M}, \quad \exists \text{ logical proof } p \in \textsf{Core} \quad \text{s.t.} \quad \textsf{TypeCheck}(p, \textsf{Ontology}) = \textsf{Pass}$$

\subsection{Global Interaction Protocol}
This protocol describes how a physical pulse traverses the four-layer architecture and ultimately solidifies into a globally accepted truth.

\subsubsection{Phase I: Refinement and Injection}
\begin{enumerate}[label=(\arabic*)]
    \item Triggering Party
    Runtime Layer (External Environment)
    \item Action
    \begin{itemize}
    \item The physical sensor generates a raw signal $s \in Q_{raw}$.
    \item Runtime invokes the core operator $\textsf{elab}(s, \textsf{Env})$.
    \item Cross-layer Interaction: $\textsf{elab}$ references predicate templates defined in the Ontology layer and constructs a dependent pair proof term $p : \textsf{Pre}(e, \Sigma)$ in the Core layer.
    \item Result: Generates a valid intent item $\langle e, p \rangle$.
    \end{itemize}
\end{enumerate}

\subsubsection{Phase II: Kernel Enqueue and Sequencing}
\begin{enumerate}[label=(\arabic*)]
    \item Triggering Party
        Kernel Layer
    \item Action
    \begin{itemize}
        \item The kernel receives the intent item from Runtime.
        \item Invokes the Kernel operator $\textsf{schedule}(\Sigma, e)$ to mount the event to the intent queue $\mathcal{P}$.
        \item At this point, the system clock $\mathcal{T}$ remains unchanged, but the configuration of $\Sigma$ has undergone logical pre-allocation.
    \end{itemize}
\end{enumerate}

\subsubsection{Phase III: Logical Reduction and Judgment}
\begin{enumerate}[label=(\arabic*)]
    \item Triggering Party
    Kernel Layer (Core Engine)
    \item Action
    \begin{itemize}
        \item The kernel loops to invoke $\textsf{peek}(\Sigma)$ to retrieve the head event from the queue.
        \item Core Validation: Perform judgment based on the Core layer's type checking rules:
  $$\Gamma, \mathcal{K}, \mathcal{T} \vdash p : \textsf{Pre}(e, \Sigma)$$
        \item Reduction Computation: Execute $\textsf{Op}(e)$. At this point, the Core layer performs $\beta$ and $\iota$ reductions to compute candidate new states $\Sigma_{try}$.
        \item Postcondition Closure: Verify $\Sigma' \vdash p' : \textsf{Post}(e)$.
    \end{itemize}
\end{enumerate}

\subsubsection{Phase IV: Atomic Materialization and Persistence}
\begin{enumerate}[label=(\arabic*)]
  \item Triggering Party
  Runtime Layer (Storage Subsystem)
  \item Action
  \begin{itemize}
  \item The kernel issues the validated $\Sigma'$ to Runtime.
  \item Runtime invokes the materialization mapping $\mathcal{M} \vdash \Sigma' \Downarrow \mathcal{M}'$.
  \item Physical Confirmation: The underlying database returns ACK, and the logical clock executes $\textsf{tick}$, formally completing the state jump.
  \item Causal Anchoring: Record the transition item $\delta = \langle \Sigma \xrightarrow{e} \Sigma' \rangle$ in physical storage.
  \end{itemize}
\end{enumerate}

\begin{table}[h]
\centering
\caption{Entity Attribute Table in System Evolution Process}
\label{tab:entity-attributes}
\begin{tabular}{|c|l|l|l|l|}
\hline
Step & Entity & Data Form & Responsible Layer & Property \\
\hline
1 & Signal & Raw Bitstream (Raw Bits) & Physical & Non-deterministic \\
\hline
2 & Proof & Dependent Pair $\langle e, p \rangle$ & Runtime/Core & Constructive \\
\hline
3 & Intent & Pending Queue $\mathcal{P}$ & Kernel & Ordered \\
\hline
4 & Reduction & $\lambda$-term Evolution & Core/Kernel & Deterministic \\
\hline
5 & Fact & Persistent Knowledge $\mathcal{K}$ & Runtime & Immutable \\
\hline
\end{tabular}
\end{table}

Protocol Consistency Guarantee (Global Invariant)
This protocol enforces a global invariant:
``Any bit flip in physical storage must have a complete proof chain extending from Ontology to Core.''
This means the KOS-TL system has no ``undefined behavior.'' Any operation not satisfying this protocol path (e.g., illegal injection, proof missing, clock reversal) will be automatically intercepted at its respective layer and rolled back to the previous well-formed state $\Sigma_{last}$.

\subsection{Interaction Interface}

\subsubsection{Core and Kernel Interaction Interface: Type Judgment Interface ($\textsf{Logic-Kernel Interface}$)}
\begin{itemize}
\item \textbf{Direction}: Kernel calls Core.
\item \textbf{Interaction Content}: The Kernel submits the current intent $e$ and its carried proof term $p$ to the Core.
\item \textbf{Interface Primitives}: $\textsf{check}(\Gamma, p, \textsf{Pre}(e))$ and $\textsf{reduce}(\textsf{Op}(e), \sigma)$.
\item \textbf{Properties}: Intensional. It is purely logical and does not perceive physical time or hardware states.
\end{itemize}

\subsubsection{Kernel and Runtime Interaction Interface: Evolution-Driven Interface ($\textsf{Kernel-Runtime Interface}$)}
\begin{itemize}
\item \textbf{Direction}: Bidirectional.
\item \textbf{Interaction Content}:
  \begin{itemize}
  \item \textit{Upward} (Runtime $\to$ Kernel): Push refined event pairs $\langle e, p \rangle$ into the queue.
  \item \textit{Downward} (Kernel $\to$ Runtime): Issue validated new states $\sigma'$ for materialization.
  \end{itemize}
\item \textbf{Interface Primitives}: $\textsf{schedule}(e, p)$ and $\textsf{commit}(\sigma')$.
\item \textbf{Properties}: Atomicity. Ensures synchronization between logical state jumps and physical storage updates.
\end{itemize}

\subsubsection{Core and Runtime Lateral Dependency: Refinement Template Interface ($\textsf{Refinement Interface}$)}
\begin{itemize}
\item \textbf{Direction}: Runtime references Core.
\item \textbf{Interaction Content}: The Runtime's $\textsf{elab}$ operator needs to reference ontology templates defined in the Core layer to construct valid proofs.
\item \textbf{Properties}: Constructiveness. Ensures that data extracted from physical signals conforms to the sorts (Sorts) defined in the logical specification.
\end{itemize}

\begin{figure}[h]
\centering
\begin{tikzpicture}[
    node distance=1.5cm,
    layer/.style={rectangle, rounded corners, draw=black, very thick, minimum width=8cm, minimum height=2.5cm, fill=white},
    interface/.style={fill=blue!10, draw=blue!50, dashed, thick, rounded corners, minimum width=6cm, minimum height=0.8cm},
    dataflow/.style={thick, color=orange!80!black},
    font=\sffamily,
    scale = 0.55
]
\node[layer, fill=red!5] (core) {
    \begin{tabular}{c}
    \textbf{Core Layer (Logical Foundation)} \\
    \footnotesize Dependent Type Checking / $\beta,\iota$-reduction / Ontology Definition ($\Gamma$)
    \end{tabular}
};
\node[layer, fill=green!5, below=of core] (kernel) {
    \begin{tabular}{c}
    \textbf{Kernel Layer (Evolution Kernel)} \\
    \footnotesize State Triple $\sigma \langle \mathcal{K}, \mathcal{T}, \mathcal{P} \rangle$ / Intent Scheduling / Small-Step Transitions
    \end{tabular}
};
\node[layer, fill=blue!5, below=of kernel] (runtime) {
    \begin{tabular}{c}
    \textbf{Runtime Layer (Execution Environment)} \\
    \footnotesize Signal Refinement (elab) / Physical Storage (M) / Hardware Interaction
    \end{tabular}
};
\node[interface] (int1) at ($(core.south)!0.5!(kernel.north)$) {\footnotesize \textbf{Interface: Type Judgment} (Judgment Term $p$, Predicate $Pre$)};
\node[interface] (int2) at ($(kernel.south)!0.5!(runtime.north)$) {\footnotesize \textbf{Interface: Commit/Schedule} (New State $\sigma'$, Event $e$)};
\draw[dataflow] (runtime.west) -- ++(-1,0) |- node[pos=0.25, left, align=right] {Refined Events\\$\langle e, p \rangle$} (kernel.west);
\draw[dataflow] (kernel.east) -- ++(1,0) |- node[pos=0.25, right, align=left] {Logical Validation Request\\$\Gamma \vdash p : A$} (core.east);
\draw[dataflow] ($(core.south west)!0.3!(core.south)$) -- ($(kernel.north west)!0.3!(kernel.north)$)
    node[midway, left] {\scriptsize Reduction Result};
\draw[dataflow] ($(kernel.south east)!0.3!(kernel.south)$) -- ($(runtime.north east)!0.3!(runtime.north)$)
    node[midway, right] {\scriptsize Materialization Instruction};
\node[draw, right=2cm of runtime, fill=gray!10] (env) {Physical World};
\draw[dashed, <->] (runtime) -- (env) node[midway, above] {\scriptsize Signals/IO};
\end{tikzpicture}
\caption{KOS-TL Layered Interaction Interface Diagram}
\label{figure:KOSinterface}
\end{figure}
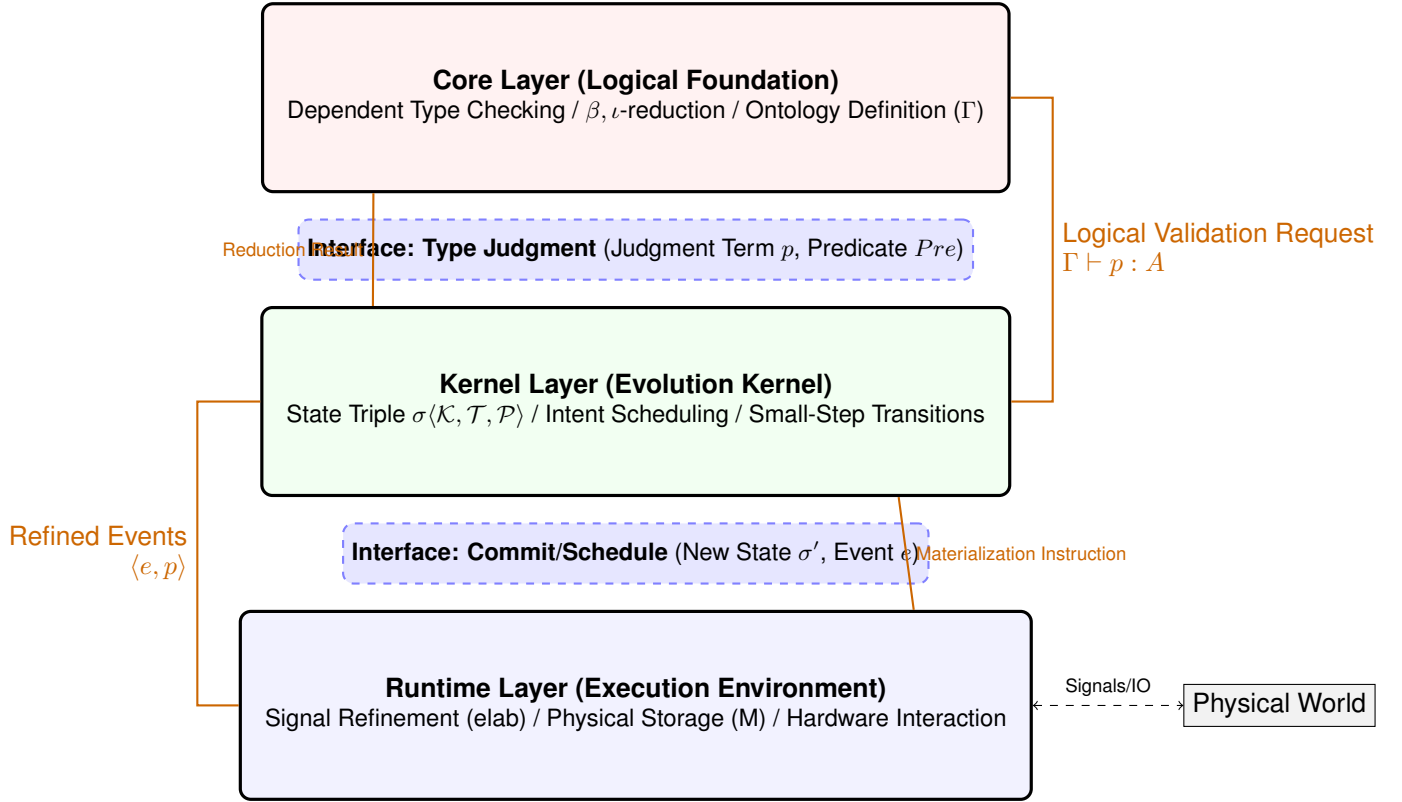

KOS-TL couples "knowledge" (static knowledge and proofs) with "action" (dynamic state transitions) via $\Sigma$-types. From an overall perspective, it is a self-consistent, computable logical entity: The Core provides the semantic framework, the Kernel provides evolutionary power, and the Runtime provides environmental mapping. This architecture enables complex systems not only to store data but also to achieve causal tracing and compliance self-verification through logical reduction.

\subsection{System Properties}

\begin{Theorem}{Knowledge Monotonicity}

Let $\sigma$ be the kernel knowledge base (i.e., the set of accepted facts), and $ku$ a well-formed knowledge object such that $\Gamma \vdash ku : A$. If $\sigma$ satisfies the introduction condition for $ku$ (denoted $\sigma \vdash ku$), then for any subsequent state $\sigma'$ satisfying evolutionary consistency, if there is no conflict proof $\pi$ against $ku$ (i.e., $\sigma' \nvdash \textsf{refute}(ku)$), then:

$$\sigma \subseteq \sigma' \implies (\sigma' \vdash ku)$$
That is: Established truths remain unchanged under valid expansions of the knowledge base.
\end{Theorem}

\begin{proof}
We prove using the Kripke Semantics framework and the Weakening Lemma of constructive logic:

\textbf{1. Establish the Frame Expansion Model:}

We define the kernel state evolution as a Kripke frame $\langle W, \le, \Vdash \rangle$, where:
\begin{itemize}
\item $W$ is the set of all possible knowledge base states.
\item $\le$ is a partial order on $W$, where $\sigma \le \sigma'$ indicates $\sigma'$ is a valid evolutionary successor of $\sigma$.
\item $\Vdash$ is the forcing relation, where $\sigma \Vdash ku$ indicates that in state $\sigma$, the proof term for knowledge object $ku$ is constructible.
\end{itemize}

\textbf{2. Prove Persistence of the Core Layer:}

The Core layer of KOS-TL is based on intuitionistic type theory. In intuitionistic logic, all operators ($\Pi, \Sigma$, etc.) satisfy persistence. We induct on the proof structure of $ku$:
\begin{itemize}
\item \textbf{Base Items}: If $ku$ is an atomic fact (e.g., physical constant or verified ID), by the Kripke model definition, if $\sigma \Vdash ku$ and $\sigma \le \sigma'$, since $\sigma \subseteq \sigma'$, then $ku$ and its original proof evidence still exist in $\sigma'$.
\item \textbf{Composite Items}: If $ku = \langle v, p \rangle$ is a dependent pair, by the induction hypothesis, the value $v$ remains unchanged under expansion. For the proof term $p$, since $\sigma'$ only adds new facts without introducing a counterproof against $p$ (guaranteed by the theorem premise), by the type theory's \textbf{Weakening Lemma}, $\Gamma, \sigma \vdash p : P \implies \Gamma, \sigma' \vdash p : P$ still holds.
\end{itemize}

\textbf{3. Exclusion of Causal Cancellation:}

In KOS-TL, an item only moves from the ``current active base'' to the ``historical archive'' when the kernel explicitly constructs a contradiction item $\textsf{contra}(ku)$.
If $\sigma' \nvdash \textsf{refute}(ku)$, then there is no opposing evidence in $\sigma'$'s search space that can reduce-collapse with $ku$. Therefore, the logical evolution operator $\textsf{unify}$ maintains the accessibility of $ku$.

\textbf{Conclusion}: $\sigma' \Vdash ku$ holds, and knowledge has monotonicity.
\end{proof}

Property Discussion: Significance for ``Causal Tracing''
Knowledge Monotonicity solves the ``memory inconsistency'' problem in complex systems:
Evidence Persistence: It guarantees that if the banking system proves a transaction compliant at time T1, unless evidence forgery is discovered at T2 (counterproof), this compliance conclusion will never inexplicably disappear due to database cleanup or addition of other transactions.
Decision Consistency: This enables unmanned systems built on KOS-TL (e.g., autonomous driving) to maintain long-term environmental cognition, avoiding ``forgetting'' previous safety boundaries when processing new sensor information.

\begin{Theorem}{Computational Reflexivity}

In the KOS-TL kernel, there exists a reflection operator $\textsf{reflect}$, such that for any well-formed term $t \in \mathcal{D}_{Core}$ and its evolution step $\textsf{step} : t \xrightarrow{small} t'$ in the Kernel layer, the system can automatically synthesize an internal proof term $\pi$, satisfying:
$$\Gamma \vdash \pi : \textsf{EvalPath}(t, t')$$
where $\textsf{EvalPath}$ is a dependent type recording the axiomatized derivation sequence from $t$ to $t'$. This means every state change in the kernel comes with a ``meta-proof'' of its own legitimacy.
\end{Theorem}

\begin{proof}
We prove using identity types in Martin-L\"{o}f type theory and meta-circular mapping:

\textbf{1. Algebraic Mapping of Reduction Steps:}

Since the Core layer of KOS-TL is based on pure, side-effect-free dependent type calculus, its computational semantics is referentially transparent. Every reduction step $t \xrightarrow{small} t'$ is not a random memory overwrite but the application of a specific reduction rule (e.g., $\beta$-reduction or $\iota$-reduction).

\textbf{2. Automatic Synthesis of Proof Terms:}

For each type of basic reduction executed by the kernel, we define a mapping function $\mathcal{R}$:
\begin{itemize}
\item \textbf{Beta Reduction}: When executing $(\lambda x. M) N \to M[N/x]$, the kernel constructs $\pi = \textsf{refl}_{\beta}(M, N)$ using the internal axiom $\textsf{beta\_axiom}$.
\item \textbf{Iota Reduction}: When executing $\textsf{proj}_1 \langle a, b \rangle \to a$, the kernel constructs $\pi = \textsf{refl}_{\iota}(a, b)$ using $\textsf{proj\_axiom}$.
\end{itemize}
Since all reduction rules have corresponding axioms defined in the Core layer, the kernel can synchronously record the axiom sequence used while performing computation.

\textbf{3. Establish Equivalence Using J-Eliminator:}

In dependent type theory, the only constructor for the equivalence type $\textsf{Id}_A(t, t')$ is $\textsf{refl}$. According to the J-operator (Identity Elimination), if two terms are equivalent under logical reduction, they are indistinguishable under all logical predicates.
By mapping each execution action $t \to t'$ of the Kernel to the application process of the J-operator, the kernel is actually continuously constructing a mathematical testimony of ``why I changed from $t$ to $t$'').

\textbf{4. Meta-circular Self-Audit:}

There exists a subroutine $\textsf{Audit} \subset \textsf{Kernel}$, which takes the proof term $\pi$ and path $\textsf{EvalPath}$ as input. Due to the strong normalization property of KOS-TL, $\textsf{Audit}$ can verify in finite steps whether $\pi$ indeed supports the transition from $t$ to $t'$.
\end{proof}

Property Discussion: Significance for ``Autonomous Systems''
Computational Reflexivity elevates KOS-TL to the level of a **``self-aware system'':
Full-Time Automatic Audit: Traditional systems require external audit logs, while KOS-TL's logs are its execution paths. This means audit is not ``post-hoc smoke,'' but ``preemptive proof.''
Decision Transparency: In autonomous driving or financial transactions, when the system makes a decision (e.g., emergency obstacle avoidance or transaction interception), reflexivity ensures the system can immediately output a human-readable and mathematically valid ``compliance explanation report.''
Logical Basis for Self-Repair: When the system detects a deviation in the hardware refinement mapping, it can pinpoint the conflicting logical operator by reflexively comparing the ``expected path'' with the ``actual path.''

\begin{Theorem}{System-Wide Safety}

Let $\mathcal{S}$ be the system's physical state space, and $\textsf{Safe} \subseteq \mathcal{S}$ the predefined physical safe subset. If the initial state $s_0 \in \textsf{Safe}$ of the KOS-TL system, then for any physical evolution sequence $s_0 \xrightarrow{hw} s_1 \xrightarrow{hw} \dots \xrightarrow{hw} s_n$, it always holds that:
$$\forall i \ge 0, \quad s_i \in \textsf{Safe}$$
Under the premise conditions:
\begin{enumerate}[label=(\arabic*)]
\item The Core layer satisfies consistency (Consistency).
\item The Kernel layer satisfies progress (Progress) and evolutionary consistency.
\item The Runtime layer satisfies refinement fidelity (Refinement Fidelity).
\end{enumerate}
\end{Theorem}

\begin{proof}
The proof uses layered induction, mapping physical evolution to reductions of logical proof terms.

\textbf{1. Base Case}

For the initial state $s_0$, according to Runtime's refinement fidelity:
$$\mathcal{E}(s_0) = ku_0 \quad \text{and} \quad \Gamma \vdash ku_0 : \textsf{Qualified}(s_0)$$
Since $s_0 \in \textsf{Safe}$, the corresponding predicate $\textsf{is\_safe}(\textsf{proj}_1(ku_0))$ has a proof term $p_0$ in the Core layer.

\textbf{2. Inductive Step}

Assume the system is in $s_i \in \textsf{Safe}$ at step $i$. Consider the migration to step $i+1$:

\textbf{A. Physical Disturbance and Refinement:}

When the physical environment changes $s_i \xrightarrow{hw} s_{i+1}$ (e.g., sensor value changes or hardware failure), the Runtime immediately captures the change and attempts to construct a new knowledge object $ku_{i+1}$:
$$\mathcal{E}(s_{i+1}) = ku_{i+1}$$

\textbf{B. Kernel Logical Judgment:}

The kernel submits $ku_{i+1}$ to the unify operator. This produces two branches:
\begin{itemize}
\item \textbf{Branch 1: $s_{i+1}$ Still in Safe:}
The kernel can successfully construct $p_{i+1} : \textsf{is\_safe}(s_{i+1})$ based on Core layer rules. According to the Kernel's evolutionary consistency, the state $\sigma$ updates to include $ku_{i+1}$, maintaining safety.
\item \textbf{Branch 2: $s_{i+1}$ Attempts to Cross Safe Boundary:}
At this point, no proof term of type $\textsf{is\_safe}(s_{i+1})$ can be constructed in the Core layer.
According to Core layer consistency (cannot prove false propositions), the kernel's logical engine produces a reduction block (or type conflict).
\end{itemize}

\textbf{C. Self-healing Loop:}

According to Kernel progress, the kernel does not deadlock; it executes find\_root\_cause and triggers analyze.
The Runtime receives the safety instruction $\pi$ from the kernel, and according to observational adequacy, this instruction enforces execution at the physical layer (e.g., circuit breaker, switch to redundant path), pulling the physical state back to $s'_{i+1} \in \textsf{Safe}$.

\textbf{3. Reductio ad Absurdum}

Assume there exists some $s_j \notin \textsf{Safe}$:
\begin{enumerate}[label=(\arabic*)]
\item This means Runtime must refine a logical item $ku_j$ such that $\Gamma \vdash ku_j : \textsf{is\_safe}$.
\item But $s_j \notin \textsf{Safe}$ means $\textsf{is\_safe}(s_j)$ is equivalent to $\bot$ (false proposition) in the Core layer.
\item Then it derives $\Gamma \vdash ku_j : \bot$.
\item This violates the Core layer consistency theorem (no item for false propositions exists in the system).
\item Hence $s_j \notin \textsf{Safe}$ is logically unconstructible.
\end{enumerate}
\end{proof}

\begin{table}[htbp]
\centering
\caption{The Logical Spectrum of KOS-TL}
\begin{tabular}{|l|l|l|l|}
\hline
\textbf{Property Name} & \textbf{Belonging Layer} & \textbf{Core Value} & \textbf{Formal Metaphor / Definition} \\
\hline
Consistency & Core & Root out logical contradictions & $\sigma \nvdash \bot$ \\
\hline
Strong Normalization & Core & Ensure real-time response & $\forall t, \exists v: \text{NormalForm}, t \twoheadrightarrow v$ \\
\hline
Progress & Kernel & Continuous self-healing operation & $\mathcal{C} \notin \text{Final} \implies \exists \mathcal{C}': \mathcal{C} \xrightarrow{small} \mathcal{C}'$ \\
\hline
Evolutionary Consistency & Kernel & Safe state evolution & $\sigma \xrightarrow{T} \sigma' \implies \text{TypeCheck}(\sigma') = \text{Success}$ \\
\hline
Monotonicity & \textbf{System-wide / Kernel} & Causal evidence persistence & $\sigma \subseteq \sigma' \implies (\sigma \Vdash ku \implies \sigma' \Vdash ku)$ \\
\hline
Fidelity & Runtime & Lossless physical mapping & $(s, ku) \in \text{SimulationRelation}$ \\
\hline
Adequacy & Runtime & Lossless instruction delivery & $\text{Exec}(\mathcal{G}(\textit{ctrl}), s) \models \text{Refine}^{-1}(\phi)$ \\
\hline
Reflexivity & \textbf{System-wide} & Full-path audit tracking & $\forall t \to t', \exists \pi: \text{Id}(t, t')$ \\
\hline
\end{tabular}
\end{table}

\subsection{Characteristics and Applications}

In compliance auditing for multinational banks, processing hundreds of millions of SWIFT transaction records is not merely a \emph{big data} challenge, but fundamentally a challenge of \emph{logical correctness}. Traditional systems typically oscillate between statistical anomaly detection (e.g., identifying frequent high-value transfers) and hard-coded rule engines, which often leads to an overwhelming number of false positives. With the introduction of KOS-TL, compliance auditing is transformed from a ``probabilistic black box'' into a \emph{formal causal system}.

Below, we elaborate in detail on the logical architecture and execution workflow of this integrated approach.

\subsubsection{Logical Abstraction: Defining Anti-Money Laundering Axioms}

In KOS-TL, compliance is not a flag in a database, but a \emph{proof goal}.

\paragraph{A. Invariants of Fund Flows}

At the Core layer, compliant transactions are defined using dependent types. A compliant transaction $T$ must satisfy:
\[
\textsf{ValidTx} \equiv
\Sigma(t : \textsf{TxData}).\;
\Sigma(e : \textsf{Evidence}).\;
\textsf{CheckCompliance}(t, e)
\]
where:
\begin{itemize}
  \item $t$: SWIFT message data (sender, receiver, amount),
  \item $e$: business logic evidence (e.g., hashes of trade contracts, customs declarations),
  \item $\textsf{CheckCompliance}$: a logical function requiring semantic alignment between $t$ and $e$ (e.g., consistency between goods value and transfer amount within acceptable tolerance).
\end{itemize}

\paragraph{B. Topological Axiom: Acyclicity}

A core feature of money laundering is \emph{layering and integration}, often manifested as funds circulating through multiple entities and eventually returning to the origin. We define a path type $\textsf{Path}(A, A)$. If an inhabitant of this type can be constructed without any substantive transformation of fund attributes, a logical contradiction is detected.

\subsubsection{Integrated Execution: From Massive Data to Logical Evidence}

\paragraph{Stage I: Large-scale Filtering (Database Layer -- Efficiency First)}

Underlying databases (e.g., ClickHouse or Neo4j) leverage high concurrency to perform initial graph-based analyses.

\begin{itemize}
  \item Task: Identify suspicious cycles or high-risk node associations among hundreds of millions of records.
  \item Result: Approximately 10{,}000 suspicious transaction chains are extracted. At this stage, they are statistically suspicious but not yet conclusively classified.
\end{itemize}

\paragraph{Stage II: Evidence Request and Refinement (Runtime Layer -- Fidelity)}

The KOS-TL kernel takes over these 10{,}000 chains. For each chain, the Runtime module issues evidence backfill requests to business systems.

\begin{itemize}
  \item Operation: Request underlying contracts and bills of lading for the corresponding SWIFT transactions.
  \item Fidelity: Evidence is refined into $ku_{\textit{evidence}}$, accompanied by immutable timestamps and provenance proofs.
\end{itemize}

\paragraph{Stage III: Dependent Type Verification (Core/Kernel Layer -- Rigor)}

This is the core step of KOS-TL. The kernel attempts to construct a \emph{compliance proof term} $p$ for each suspicious chain.

\begin{Theorem}{Transaction Compliance Verification}

For a suspicious chain $L = \{t_1, t_2, \dots, t_n\}$ to be marked as \textsf{Verified}, a total proof term must be constructed:
\[
P_{\text{total}} = \langle p_1, p_2, \dots, p_n \rangle
\]
such that each $p_i$ proves that the evidence $e_i$ eliminates the ``cyclicity hypothesis'' induced by the chain.
\end{Theorem}

\subsubsection{Logical Interception}
If a transaction corresponds to fictitious trade (e.g., transfer amount does not match the contract hash), the Core layer cannot construct a proof term. By the consistency theorem, the kernel cannot evolve the chain state to \textsf{Verified}.

\subsubsection{Case Study: Cooling System Fault Handling}

\emph{Physical state} $s$: Abnormal voltage fluctuations from the pressure sensor of cooling pump A.

\emph{Safety goal} $\textsf{Safe}$: Pressure must remain within $[P_L, P_H]$, and the sensor must possess valid calibration proof.

\paragraph{Step 1: Refinement and Fidelity (Runtime Layer)}
A raw voltage signal of $2.4V$ is refined:
\[
ku_{\text{press}} = \langle 120kPa, p_{\text{calib}} \rangle : \textsf{Press}
\]
Without $p_{\text{calib}}$, no term of type \textsf{Press} can be constructed.

\paragraph{Step 2: Evolution and Monotonicity (Kernel Layer)}
The kernel updates the global state:
\[
\sigma_{\text{new}} = \textsf{unify}(\sigma, ku_{\text{press}})
\]
Previously recorded facts (e.g., ``Pump A is active'') are preserved.

\paragraph{Step 3: Reduction and Consistency (Core Layer)}
The kernel evaluates:
\[
\textsf{is\_safe}(v) \equiv (P_L \le v \le P_H)
\]
With $P_H = 110kPa$, the proof term $p_{\text{safe}} : \textsf{is\_safe}(120)$ cannot be constructed. The system cannot falsely assert safety.

\paragraph{Step 4: Reflexive Audit and Self-Healing}
A control instruction $\textit{ctrl}$ is generated (e.g., activate pump B, shut down pump A), along with proof:
\[
\pi : \textsf{Id}(\sigma_{\text{fault}}, \sigma_{\text{recovery}})
\]
This proof explains the decision in terms of Core-level constraints.

\paragraph{Step 5: Observational Adequacy (Runtime Layer)}
The logical instruction $\textsf{Close(Pump\_A)}$ is refined into concrete bus signals, guaranteeing the intended physical effect.

\section{Application of KOS-TL}

\subsection{Application Background: Quality Anomaly Traceability in Manufacturing}

Consider a large discrete manufacturing enterprise whose core challenge is the following:

\begin{quote}
\emph{When a certain batch of products exhibits severe quality defects, can the system automatically trace its production process, identify potential anomalies related to equipment, personnel, or raw materials, and produce an executable and explainable causal chain?}
\end{quote}

This problem exhibits several typical characteristics:

\begin{itemize}
  \item Heterogeneous data sources (work orders, equipment logs, personnel schedules, quality inspection records);
  \item Strong temporal ordering and causal constraints;
  \item Inference results must directly support production decisions and responsibility attribution.
\end{itemize}

The system involves the following core tables (originating from different subsystems):
\begin{enumerate}[label=(\arabic*)]
  \item Product records: \texttt{Product(ProductID, Model)}
  \item Batch records: \texttt{Batch(BatchID, ProductID, ProduceDate)}
  \item Production lines: \texttt{ProductionLine(LineID, Factory)}
  \item Process routes: \texttt{ProcessRoute(Model, StepName, StepOrder, TargetLineType)}
  \item Process thresholds: \texttt{ProcessThreshold(Model, StepName, ParamName, MinValue, MaxValue)}
  \item Step execution details: \texttt{StepExecution(WOID, StepName, StartTime, EndTime, EquipID)}
  \item Sensor time series: \texttt{SensorTimeSeries(EquipID, ParamName, Value, Timestamp, DeviceStatus)}
  \item Work orders: \texttt{WorkOrder(WOID, BatchID, LineID)}
  \item Operators: \texttt{Operator(OperatorID, Name, Role)}
  \item Operation logs: \texttt{OperationLog(LogID, WOID, OperatorID, Time)}
  \item Equipment: \texttt{Equipment(EquipID, LineID)}
  \item Equipment status: \texttt{EquipmentStatus(EventID, EquipID, Status, Time)}
  \item Process parameters: \texttt{ProcessParam(LogID, ParamName, Value)}
  \item Quality inspections: \texttt{QualityInspection(InspectID, BatchID, Result, Time)}
  \item Defect reports: \texttt{DefectReport(ReportID, BatchID, DefectType)}
  \item Supply chain records: \texttt{SupplierPart(PartID, SupplierID, BatchID)}
\end{enumerate}

To illustrate the approach, we trace a causal reasoning workflow for quality anomaly analysis in bearing production, as shown in Table~\ref{tab:causal-reasoning}.

\begin{table}[h]
\centering
\caption{Causal Reasoning for Quality Traceability in Bearing Manufacturing}
\label{tab:causal-reasoning}
\begin{tabular}{>{\centering\arraybackslash}p{2cm} >{\centering\arraybackslash}p{3cm} >{\raggedright\arraybackslash}p{9cm}}
\toprule
\textbf{Step} & \textbf{System Action} & \textbf{Concrete Data Example} \\
\midrule
Input & Quality inspection reports anomaly & Batch\_202310-01 detected ``non-uniform hardness'' at 10:00 on 2023-10-10. \\
Type instantiation & Construct $f_{fail}$ & $f_{fail} : \mathsf{FailureEvent} = \langle \text{"B2310"}, \text{"HARD\_ERR"}, \text{10:00} \rangle$ \\
Kernel reasoning & Search for causal evidence & Retrieved that the batch passed through furnace HeatTreatment\_03 at 08:00, and a temperature fluctuation $a_{temp}$ occurred at 07:55. \\
Logical synthesis & Construct causal chain & $r = \langle f_{fail}, a_{temp}, \text{prf}_{causal} \rangle$. The proof term $\text{prf}_{causal}$ automatically verifies $07{:}55 < 10{:}00$. \\
\bottomrule
\end{tabular}
\end{table}

By applying KOS-TL reasoning, the system does not return a simple SQL query result to the user, but rather a \emph{logical proof package}. When the user opens the report, the system can expand $\text{prf}_{causal}$ and directly locate the original PLC logs corresponding to the temperature fluctuation, since these logs are integral components in the construction of the report $r$.

\subsection{The Application Workflow of KOS-TL}

In practical deployment, KOS-TL mainly involves the following stages:
\begin{enumerate}
    \item Definition of initial atomic types, predicate types, events, and constraints.
    This part belongs to the \emph{Core} layer. Such definitions essentially determine the
    boundary of logical validity of the system being modeled.
    \item The \emph{Runtime} layer serves as the interface between the system and the external world.
    Through the runtime, the KOS-TL system acquires data and refines it into typed objects
    (logically operable entities).
    \item The \emph{Kernel} layer is responsible for concrete knowledge operations.
\end{enumerate}

\subsubsection{Kernel Layer: Rule Definitions and Logical Constraints}

For the problem of \emph{causal reasoning in bearing production quality traceability},
the Kernel layer defines the corresponding types and constraints.

\paragraph{(1) Basic Atomic Types}

The basic atomic types are shown in Table~\ref{tab:domain-types-refinement},
including $\mathsf{BatchID}$, $\mathsf{Machine}$, and $\mathsf{Time}$.

\begin{table}[h]
\centering
\caption{Domain Concepts and Type Refinement Relations}
\label{tab:domain-types-refinement}
\begin{tabular}{>{\centering\arraybackslash}p{2cm} >{\raggedright\arraybackslash}p{3.5cm} >{\raggedright\arraybackslash}p{3.5cm} >{\raggedright\arraybackslash}p{5cm}}
\toprule
\textbf{Domain Concept} & \textbf{Logical Type (Core)} & \textbf{Kernel Atomic Type} & \textbf{Refinement Logic} \\
\midrule
Time & Time & Float / UInt64 & Direct mapping, representing Unix timestamps or logical clocks. \\
BatchID & BatchID & Val / String & $\Sigma(s:Val).\text{Proof}(isIDFormat(s))$ \\
Machine & Machine & Val / Enum & $\Sigma(v:Val).\text{Proof}(v\in EquipRegistry)$ \\
\bottomrule
\end{tabular}
\end{table}

\paragraph{(2) Predicate Types}

Predicate types include:
\begin{itemize}
    \item $\mathsf{InRoute}(b, m)$, which defines whether batch $b$ is allowed to be processed on machine $m$.
    \item $\mathsf{Overlap}(t, dur)$, which defines whether time point $t$ falls within duration $dur$.
\end{itemize}

\paragraph{(3) Events and Constraints}

\begin{enumerate}[label=(\roman*)]
    \item \textbf{Failure Event Type} ($\mathsf{FailEvt}$)
    \[
    \mathsf{FailEvt} \equiv \Sigma(b: \mathsf{BatchID}). \Sigma(err: \mathsf{ErrorCode}). \Sigma(t: \mathsf{Time}). \mathsf{Proof}(t \in \text{Shift}_{QA})
    \]
    This type not only records which batch failed, but also enforces a proof that the inspection time lies within the QA shift.

    \item \textbf{Process Step Type} ($\mathsf{ProcStep}$)
    \[
    \mathsf{ProcStep} \equiv \Sigma(b: \mathsf{BatchID}). \Sigma(m: \mathsf{Machine}). \Sigma(dur: \mathsf{Time} \times \mathsf{Time}). \mathsf{Proof}(\mathsf{InRoute}(b, m))
    \]
    The predicate $\mathsf{InRoute}$ guarantees that the batch is processed on machine $m$ according to the defined process route.

    \item \textbf{Environmental Anomaly Type} ($\mathsf{Anomaly}$)
    \[
    \mathsf{Anomaly} \equiv \Sigma(m: \mathsf{Machine}). \Sigma(p: \mathsf{Param}). \Sigma(v: \mathsf{Val}). \Sigma(t: \mathsf{Time})
    \]

    \item \textbf{Causal Validity Constraint} ($\mathsf{CausalProof}(a,f)$)
    \[
    \mathsf{isBefore}(t(a), t(f)) \land \mathsf{isSameResource}(\text{location}(a), \text{process}(f))
    \]
    Traceability is defined as a proof search problem:
    \[
    \forall f : \mathsf{Failure}, \exists (a, \pi) : \Sigma(a:\mathsf{Anomaly}) . \mathsf{CausalProof}(a, f)
    \]

    \item \textbf{Causal Proof} ($\mathsf{CausalProof}$)
    \[
    \mathsf{CausalProof}(a, f) \equiv \Sigma(e : \mathsf{ProcStep}). \text{Prop}_{causal}(a, e, f)
    \]
    where $\text{Prop}_{causal}$ enforces:
    \begin{itemize}
        \item Temporal logic: $a.t \in e.dur \land e.dur.end < f.t$;
        \item Spatial logic: $a.m = e.m$;
        \item Batch consistency: $e.b = f.b$.
    \end{itemize}

    \item \textbf{Root Cause Report} ($\mathsf{RootCauseReport}$)
    \[
    \mathsf{RootCauseReport} \equiv \Sigma(f : \mathsf{FailEvt}) . \Sigma(a : \mathsf{Anomaly}) . \mathsf{CausalProof}(a, f)
    \]
    Semantically, this definition encodes:
    \begin{itemize}
        \item the existence of a failure $f$;
        \item the existence of a physical anomaly $a$;
        \item a non-Boolean causal proof witnessing their relation.
    \end{itemize}
\end{enumerate}

\subsubsection{Runtime Layer: Data Acquisition and Elaboration}

The Runtime layer extracts data from external databases and refines it into
Kernel-level objects, thereby \emph{logicalizing} raw data.

Typical source tables include:
\begin{itemize}
    \item \texttt{Product\_Master}: process routes;
    \item \texttt{Execution\_Log}: work order execution records;
    \item \texttt{IoT\_Sensor\_Stream}: sensor streams;
    \item \texttt{Quality\_Report}: inspection results.
\end{itemize}

These data are elaborated into proof-carrying objects such as:
\[
f_0 = \text{mkFailure}(\text{Batch\_{202310-01}}, \text{Hardness\_Issue}, 10{:}00, \pi_{QA})
\]

\subsubsection{Core Layer: Proof Construction and Small-Step Evolution}

The Core layer executes proof construction following Kernel rules.
It operates via small-step semantics, gradually evolving configurations
until a \emph{RootCauseReport} is materialized in the global knowledge base.

In summary, KOS-TL integrates runtime data acquisition, kernel-level logical constraints,
and core-level proof synthesis into a unified, type-safe reasoning pipeline.
Rather than producing opaque query results, the system yields formally verified,
explainable causal reports whose correctness is guaranteed by construction.

\subsection{KOS-TL Adapting to Changes in Business Rules}

We continue with the same example from the previous subsection---\emph{bearing heat treatment}---and assume that the business rules change as follows:

To prevent temper brittleness, if a \emph{voltage anomaly} occurs during the heat treatment process, the system must additionally check whether the \emph{cooling water circulation pressure} during the same time period is abnormal. Only when \emph{both} anomalies are present can the situation be classified as a severe quality defect.

In the KOS-TL framework introduced in the previous subsection, this change requires modifying \emph{only} the type definition of \textsf{CausalProof} in the \textbf{Core} layer.

\subsubsection{Type definition before modification}

\[
\mathsf{CausalProof}(a, f) \;\equiv\;
\Sigma(e : \mathsf{ProcStep}).\; \text{Prop}_{time}(a, e, f)
\]

\subsubsection{Type definition after modification (injecting the new rule)}

We redefine \textsf{CausalProof} as a dependent type that must contain evidence of a \emph{dual anomaly}:

\[
\mathsf{CausalProof}(a, f) \;\equiv\;
\Sigma(e : \mathsf{ProcStep}).\;
\Sigma(w : \mathsf{WaterPressureAnomaly}).\;
\text{Prop}_{joint}(a, w, e, f)
\]

The new constraint requires not only the presence of a voltage anomaly $a$, but also enforces the existence of a water pressure anomaly $w$ within the same production process $e$.

Once the type in the Core layer changes, a chain reaction is automatically triggered in the \textbf{Kernel} layer through its small-step reduction logic:

\begin{enumerate}[label=(\arabic*)]
  \item \textbf{Constructor invalidation.}
  The original constructor \textsf{mkCausalProof} immediately fails type checking in the Kernel due to missing parameters (namely $w$ and its corresponding proof).

  \item \textbf{Automatic triggering of new search.}
  Upon detecting that the target type requires a \textsf{WaterPressureAnomaly}, the Kernel automatically initiates a search over water pressure data in the environment $\sigma$.
\end{enumerate}

\subsubsection{Reorganization of proof synthesis paths}

\textbf{Case 1 (Voltage anomaly only).}
The Kernel cannot find a matching water pressure anomaly $w$, and therefore fails to construct a complete $\pi_{\mathit{causal}}$ term. The inference result is automatically classified as \emph{invalid}.

\textbf{Case 2 (Dual anomalies).}
The Kernel automatically composes the voltage anomaly, the water pressure anomaly, and the production process into a new result $r_{\mathit{final}}$.

The essence of this capability is an extreme form of \emph{Type-Directed Development}, as summarized in Table~\ref{tab:kos-tl-mechanisms}.

\begin{table}[htbp]
\centering
\caption{Mechanisms of KOS-TL}
\label{tab:kos-tl-mechanisms}
\begin{tabular}{p{2.5cm}p{8cm}p{5cm}}
\toprule
\textbf{Property} & \textbf{Mechanism} & \textbf{Implication} \\
\midrule
Self-healing &
If the runtime data source does not provide water pressure data, the Kernel raises a ``missing type'' error rather than producing an incorrect conclusion. &
Strictly guarantees the safety of conclusions and prevents blind traceability. \\
\addlinespace[0.5ex]
Push-down logic &
New rules are propagated downward through the signature of \textsf{mkCausalProof}; the Kernel's search algorithm automatically detects new parameter requirements. &
Developers do not need to rewrite search algorithms; algorithms adapt automatically to type changes. \\
\addlinespace[0.5ex]
Zero redundancy &
Existing traceability code does not need to be removed; once the referenced types are updated, its behavior changes automatically. &
Achieves true ``configuration-as-logic.'' \\
\bottomrule
\end{tabular}
\end{table}

In this example, neither the \textsf{analyze} function nor the \textsf{getProductionContext} function is modified. By changing only the type signatures in the Core layer, one redefines the physical boundary of \emph{what counts as truth}.

The Kernel layer behaves like a fully automated puzzle-solving machine: once the puzzle template (Core) is changed, it automatically adjusts its strategy for searching puzzle pieces (Data) and the final assembled picture (Result).

This is the core value of KOS-TL when dealing with complex and evolving industrial environments: ensuring atomic-level consistency between the evolution of system logic and changes in business rules.

\subsection{KOS-TL Enabling Cross-Domain Logical Consistency}

KOS-TL can further achieve cross-domain logical consistency through a \emph{shared logical kernel} and \emph{cross-domain type references}.

In KOS-TL, neither the financial domain nor the quality domain directly accesses raw databases. Instead, both domains subscribe to the same \emph{knowledge objects} materialized by the Kernel layer.

Continuing the previous example of \emph{bearing heat treatment}, consider the following scenario:

\emph{If an abnormal heat-treatment voltage is detected (a quality risk), then the payment to the raw material supplier associated with the affected batch must automatically enter a ``pending audit'' state, and the corresponding financial voucher must include evidence of the quality anomaly (financial risk control).}

KOS-TL realizes this requirement by defining \emph{cross-dependent types} at the Core layer.

\subsubsection{Logical Extension in the Quality Domain}

At the quality Core layer, we have already defined
\[
r_{\mathit{quality}} : \mathsf{RootCauseReport},
\]
which certifies that voltage fluctuation caused the quality defect.

\subsubsection{Type Definition in the Financial Domain (New)}

At the financial Core layer, we introduce a new type $\mathsf{AuditLock}$:
\[
\mathsf{AuditLock} \;\equiv\;
\Sigma(inv : \mathsf{Invoice}) .
\Sigma(r : \mathsf{RootCauseReport}) .
\mathsf{Proof}(inv.batch = r.f.b)
\]

The constructor of the financial object $\mathsf{AuditLock}$ \emph{mandatorily} requires an input
$r : \mathsf{RootCauseReport}$. Without a generated quality report $r$, the financial layer cannot construct a valid audit-lock object.

Once quality traceability is completed, a new item $r_{\mathit{final}}$ is added to the Kernel state $\sigma$. At this moment, the cross-domain logical engine automatically triggers the next reduction step:

\begin{itemize}
  \item \textbf{Financial Observer Triggered:}
  The financial system’s \textsf{analyzeAudit} procedure detects that a quality anomaly report $r_{\mathit{final}}$ exists in $\sigma$ with the same batch as invoice $inv_{01}$.
  \item \textbf{Cross-Domain Item Composition:}
  Through unification, the Kernel directly injects the quality-domain evidence $r_{\mathit{final}}$ into the financial-domain $\mathsf{AuditLock}$ object.
  \item \textbf{Atomic Update:}
  \[
  \langle \sigma, \mathsf{FinanceGoal} \rangle
  \xrightarrow{\text{small}}
  \langle \sigma \cup \{ \mathsf{lock}_{01} \}, \mathsf{unit} \rangle
  \]
  where
  \[
  \mathsf{lock}_{01} = \langle inv_{01}, r_{\mathit{final}}, \pi_{\mathit{match}} \rangle .
  \]
\end{itemize}

In the cross-domain architecture of KOS-TL, the financial procedure \textsf{analyzeAudit} is not an isolated software module. Instead, it is a \emph{predicate listener} mounted on the shared logical Kernel. Its core responsibility is to monitor the evolution of the state $\sigma$ and to automatically derive audit actions whenever specific \emph{financial–quality cross-constraints} are satisfied.

\subsubsection{The \textsf{analyzeAudit} Function}

At the Core layer, \textsf{analyzeAudit} is defined as a higher-order function whose purpose is to construct an audit-lock item:
\[
\mathsf{analyzeAudit} :
\Pi(inv : \mathsf{Invoice}) \to \mathsf{Option}(\mathsf{AuditLock})
\]

Its internal derivation logic follows the rule:
\[
\frac{
  inv \in \sigma
  \quad
  \exists r : \mathsf{RootCauseReport}
  \;\text{s.t.}\;
  \mathsf{Unify}(inv.batch, r.f.b)
}{
  \mathsf{derive}(\mathsf{AuditLock}(inv, r))
}
\]

\begin{itemize}
  \item \textbf{Input:} A financial invoice $inv$ pending processing.
  \item \textbf{Trigger Condition:}
  A quality report $r$ exists in the Kernel state whose batch unifies with the invoice batch.
  \item \textbf{Output:}
  If the condition holds, an audit object containing quality evidence is produced; otherwise, \textsf{None} is returned.
\end{itemize}

When a voltage fluctuation occurs on the production line and a quality report is generated, \textsf{analyzeAudit} undergoes the following evolution at the Kernel layer:

\paragraph{Step 1: Cross-Domain Detection}

The Kernel detects the newly materialized $r_{\mathit{final}}$. Since \textsf{analyzeAudit} subscribes to changes in the \textsf{BatchID} type, it is immediately activated.

\paragraph{Step 2: Dependency Extraction}

Using the index \texttt{inv.batch}, the program locates the associated financial invoice. For example:
\[
\text{Batch\_10-01} \;\longrightarrow\; \text{Inv\_2023\_009 (Supplier: SteelCo)}
\]

\paragraph{Step 3: Proof Transparency}

This is the most critical step. \textsf{analyzeAudit} does not merely flag an issue; it directly \emph{references} the proof object $\pi_{\mathit{causal}}$ from the quality domain:
\[
\mathsf{lock}_{item}
=
\langle \mathsf{Inv}_{009}, r_{\mathit{final}}, \pi_{\mathit{match}} \rangle
\]
This means that the financial system now holds the \emph{physical evidence} of the quality anomaly, such as voltage curves and computational records.

\paragraph{Step 4: Action Materialization}

Once $\mathsf{lock}_{item}$ is instantiated in the Kernel, the Runtime-layer financial plugin captures this state change and immediately executes the following actions in the ERP system:
\begin{itemize}
  \item \textbf{Payment Suspension:} Freeze settlement of \textsf{Inv\_2023\_009}.
  \item \textbf{Audit Endorsement:} Automatically attach the logical trace of $r_{\mathit{final}}$ to the invoice for review.
\end{itemize}

\begin{algorithm}
\caption{KOS-TL Cross-Domain Audit Procedure: $\mathsf{analyzeAudit}(r_{\mathit{final}})$}
\begin{algorithmic}[1]
\REQUIRE Newly materialized root-cause report
$r_{\mathit{final}} = \langle f, a, e, \pi_{\mathit{causal}} \rangle :
\mathsf{RootCauseReport}$
\ENSURE Audit-lock item $\mathsf{AuditLock}$ or $\emptyset$
\STATE \textbf{Step 1: Cross-Domain Detection}
\STATE Monitor Kernel state $\sigma$; trigger upon materialization of $r_{\mathit{final}}$.
\STATE Extract batch index: $b \leftarrow r_{\mathit{final}}.f.b$
\COMMENT{$BatchID$ serves as the unique key for cross-domain unification}
\STATE \textbf{Step 2: Dependency Extraction}
\STATE Retrieve related invoices from financial state $\sigma_{\mathit{fin}}$:
\STATE $inv \leftarrow \{ i \in \sigma_{\mathit{fin}} \mid
i : \mathsf{Invoice} \wedge \mathsf{Unify}(i.batch, b) \}$
\IF{$inv = \emptyset$}
  \STATE \RETURN $\emptyset$
\ENDIF
\STATE \textbf{Step 3: Proof Transparency}
\STATE Construct proof predicate $\pi_{\mathit{match}}$ linking $inv$ and $r_{\mathit{final}}$.
\STATE Reference quality proof: $\pi_{\mathit{ref}} \leftarrow r_{\mathit{final}}.\pi_{\mathit{causal}}$
\STATE Instantiate audit lock:
$lock_{item} \leftarrow \langle inv, r_{\mathit{final}}, \pi_{\mathit{match}}, \pi_{\mathit{ref}} \rangle$
\STATE \textbf{Step 4: Action Materialization}
\STATE \textbf{Atomic materialization:} $\sigma \leftarrow \sigma \cup \{ lock_{item} \}$
\STATE Execute $\mathsf{FreezePayment}(inv)$
\STATE Execute $\mathsf{AttachEvidence}(inv, r_{\mathit{final}}.\mathit{trace})$
\RETURN $lock_{item}$
\end{algorithmic}
\end{algorithm}

Suppose an additional financial rule is introduced:
\emph{``An audit lock is triggered only if the estimated quality loss exceeds $20\%$ of the invoice amount.''}

In KOS-TL, this requires only adding a logical predicate to the definition of \textsf{analyzeAudit}:
\[
\mathsf{Proof}(r.\mathit{loss\_estimate} > inv.\mathit{total} \times 0.2)
\]

Even if a quality report is generated, if the estimated loss is insufficient, the unification in \textsf{analyzeAudit} fails, and the financial system automatically maintains a \emph{normal settlement} state. This form of logical pushdown ensures that financial decision-making remains mathematically consistent with the physical realities of the production line.

\subsection{Counterfactual Reasoning in KOS-TL}

In KOS-TL, counterfactual reasoning is not realized through \emph{guessing}, but through kernel-level simulation over \textbf{parallel state spaces}.

\subsubsection{How KOS-TL Defines Counterfactuals}

In traditional AI, counterfactuals are typically expressed as probabilistic queries of the form
$P(y \mid do(x))$.
In contrast, KOS-TL defines a counterfactual as the evaluation of a virtual configuration:

\emph{``If, in the state $\sigma_0$, the fact $a$ (voltage fluctuation) had not occurred, would $r_{\mathit{final}}$ (the failure proof) still be instantiable in the kernel?''}

\subsubsection{Implementation Mechanism: Virtual Context}

KOS-TL realizes counterfactual reasoning through \emph{branching} of the context $\Gamma$.

\begin{itemize}
  \item \textbf{Shadow State Construction ($\sigma'$):}
  The kernel copies the current knowledge base $\sigma_0$, but selectively removes or modifies a specific fact (e.g., removing $a_{\mathit{volt}}$).
  \item \textbf{Hypothetical Evaluation:}
  Small-step evaluation is restarted under the new configuration
  \[
  \langle \Gamma, \sigma', \mathsf{analyze}(f_0) \rangle .
  \]
  \item \textbf{Lemma Comparison:}
  If the evaluation reduces to $\bot$ (the empty type), then $a_{\mathit{volt}}$ is a necessary condition for $f_0$ (necessary causation).
  If the evaluation can still generate $r'_{\mathit{final}}$, then redundant causation exists, or $a_{\mathit{volt}}$ is merely confounding noise.
\end{itemize}

\subsubsection{Three Application Scenarios of Counterfactual Reasoning}

\paragraph{A. Root-Cause Sensitivity Analysis}

\emph{Question:}
If the voltage fluctuation were only $2\%$ instead of $10\%$, would the hardness still fail to meet specifications?

\emph{Kernel Action:}
Modify the value of $a_{\mathit{volt}}$ in $\sigma$ and observe whether
$\mathsf{mkCausalProof}$ still passes the physical threshold checks defined at the Core layer.

\paragraph{B. Liability Attribution}

\emph{Question:}
If the batch of raw material supplied by Vendor~A had not been used, would fluctuations in furnace M\_03 alone still cause the failure?

\emph{Kernel Action:}
Remove the raw-material fact in a virtual environment and observe whether the proof chain collapses.
This directly supports the recovery and compensation logic in the upper-layer \textsf{analyzeAudit} procedure.

\subparagraph{C. Preemptive Simulation}

\emph{Question:}
If the power of furnace M\_03 were increased by $5\%$ tomorrow, would it trigger a similar failure $f_0$?

\emph{Kernel Action:}
This form of \emph{forward-looking counterfactual} allows the system to complete a ``virtual accident'' in logical space before any physical incident occurs.

\subsubsection{Logical Formulation}

In the formal language of KOS-TL, counterfactual reasoning is typically written as:
\[
\mathcal{C} \vdash \neg a \;\Rightarrow\; \neg (\exists \pi : \mathsf{Proof}(f))
\]
That is, under the current configuration, one proves that \emph{if $a$ does not exist, then no proof object for $f$ can be instantiated}.

\subsubsection{Decoupling Construction and Validation}

In the \textsf{analyze} function, $r_{\mathit{final}}$ is indeed a candidate construction that includes $a_{\mathit{volt}}$.
Counterfactual reasoning instead asks:

\emph{``If the atomic fact $a_{\mathit{volt}}$ is removed from the axiom system, can another valid proof $\pi'$ still be constructed that points to $f_0$?''}

\begin{itemize}
  \item If yes, then $a_{\mathit{volt}}$, although present in $r_{\mathit{final}}$, is only a \emph{sufficient but non-necessary condition} (or part of multiple causation).
  \item If no, then $a_{\mathit{volt}}$ is a \emph{necessary condition}.
\end{itemize}

\subsubsection{The Kernel’s Parallel-Space Mechanism: $\sigma$ vs.\ $\sigma \setminus \{a\}$}

During counterfactual reasoning, the Kernel does not operate on a single triple, but generates an environmental slice.

\begin{itemize}
  \item \textbf{Actual trajectory:}
  \[
  \langle \Gamma, \sigma, f_0 \rangle \;\Rightarrow\; r_{\mathit{final}}
  \quad (\text{containing } a_{\mathit{volt}})
  \]
  \item \textbf{Counterfactual trajectory:}
  \[
  \langle \Gamma, \sigma \setminus \{a_{\mathit{volt}}\}, f_0 \rangle \;\Rightarrow\; \bot
  \quad (\text{derivation collapses})
  \]
\end{itemize}

When the counterfactual trajectory reduces to $\bot$ (the bottom type / empty set), it logically reinforces the authority of the actual trajectory.
In logic, this is known as \emph{affirmation by negation}.

\subsubsection{Elimination Rules in Type Theory}

Counterfactual reasoning effectively computes \emph{contribution}:
\[
\mathsf{Contrib}(a, f)
\;\iff\;
(\sigma \vdash f)
\;\wedge\;
(\sigma \setminus \{a\} \nvdash f)
\]
Here, $\nvdash$ denotes that, without $a$, the remaining knowledge base cannot derive a proof of the failure $f$.

If this condition holds, it rules out the possibility that \emph{even without voltage fluctuation, hardness would still degrade due to other causes}.

\subsubsection{Multiple-Cause Illustration}

Suppose that, in addition to voltage fluctuation ($a_1$), there is also excessive impurity in the raw material ($a_3$).

\begin{itemize}
  \item \textbf{Initial derivation:}
  Two candidate reports are obtained: $r_1(a_1)$ and $r_3(a_3)$.
  \item \textbf{Counterfactual tests:}
  \begin{itemize}
    \item Remove $a_1$; if $a_3$ can still derive the failure,
    then $a_1$ is not the unique root cause.
    \item Remove $a_3$; if the derivation still holds,
    then $a_3$ is likely only noise.
  \end{itemize}
\end{itemize}

\subsection{Logic as the System Kernel}

Through the reasoning processes above, KOS-TL demonstrates its core strengths as a logical system:

\begin{enumerate}
  \item \textbf{Construction as Evidence}:
  The generated $\mathsf{RootCauseReport}$ is not merely a piece of text.
  It contains pointers to the original equipment logs and quality inspection records, together with a complete logical proof chain, thereby providing a \emph{hard guarantee} of explainability.

  \item \textbf{State Preservation}:
  At every Small-step transition, the kernel verifies the $Post$ constraints.
  If a tracing result violates physical logic (e.g., a cause occurring after its effect), the transition fails, ensuring that the system state always remains within a logically consistent space.

  \item \textbf{Computational Completeness}:
  Unlike Description Logic with its Open World Assumption, KOS-TL leverages strong normalization at the kernel level to guarantee that any tracing program will, in finitely many steps, either yield a definitive causal conclusion or report insufficient evidence.
\end{enumerate}

In this example, KOS-TL transforms quality tracing from \emph{post-hoc auditing} into \emph{real-time logical calculus}.
The system is not asking ``why did it fail,'' but instead attempts to construct a complete logical proof of ``the failure itself.''
This paradigm shift—from \textbf{truth judgment} to \textbf{proof construction}—endows complex knowledge management in manufacturing with operating-system-level rigor.

This section, through a large-scale manufacturing quality anomaly analysis scenario, illustrates how KOS-TL maps heterogeneous business tables into dependent types and realizes automatic knowledge derivation and causal tracing via Small-step operational semantics.

\begin{table}[htbp]
\centering
\caption{Fundamental Differences Between Description Logic (DL) and KOS-TL}
\begin{tabular}{l|l|l}
\hline
\textbf{Dimension} & \textbf{Description Logic (DL)} & \textbf{KOS-TL} \\ \hline
Core Paradigm & Static ontological consistency & Dynamic operational semantics \\
Logical Property & Truth evaluation & Proof construction \\
Temporal Handling & External extensions (Temporal DL) & Intrinsic temporal ordering constraints \\
Application Goal & Knowledge description and querying & Knowledge operating system kernel \\ \hline
\end{tabular}
\end{table}

\begin{quote}
\textit{``KOS-TL is not designed merely to describe the world, but to run a logically closed-loop operational control system with knowledge as its kernel.''}
\end{quote}

In KOS-TL, knowledge derivation is essentially a process of constructing and applying higher-order functions.

Traditional systems \emph{use hard-coded functions to process data}.
KOS-TL, by contrast, embodies the idea that \emph{data drives the synthesis of logical terms, which in turn form derived function bodies endowed with reasoning capability}.
This mechanism of ``data-derived functions'' can be understood along the following three dimensions.

\subsubsection{From Data Tuples to Proof Terms}

At the Runtime layer, raw data such as \texttt{(Batch\_202310-01, 10:00)} is merely passive information.
During elaboration, however, it is encapsulated into $\Sigma$-terms with logical signatures.

\emph{Essence}:
From the Kernel’s perspective, these terms are no longer simple ``values,'' but small, composable function fragments.
For example, the proof $\pi_{route}$ carried by $e_{proc}$ is itself a logical function capable of proving ``route validity.''

\subsubsection{Dynamically Generated Causal Chains as Function Composition}

What appears as the \textsf{analyze} function or the construction of a \textsf{RootCauseReport} is, in reality, the Kernel dynamically composing a new logical function based on real-time data.

\emph{Process}:
When the Kernel detects a match between $a_{volt}$ and $f_0$, it does not simply connect them.
Instead, it constructs a new lambda term:
\[
\lambda(t).\;\mathsf{proof\_of\_overlap}(t, e_{proc}.\mathsf{dur})
\]

\emph{Implication}:
The newly generated logical term $\pi_{final}$ is a \emph{specialized function}, tailored specifically to explain this particular hardness defect occurring at a specific time on a specific machine.

\subsubsection{Division of Derivation Responsibilities Across the Three-Layer Architecture}

This example illustrates how functions ``flow'' across different layers and become ``materialized'':

\begin{table}[h]
\centering
\caption{Layers and the Essence of Derivation in the Bearing Case}
\label{tab:layer-derived-behavior}
\begin{tabular}{>{\centering\arraybackslash}p{1.5cm} >{\raggedright\arraybackslash}p{4cm} >{\raggedright\arraybackslash}p{9cm}}
\toprule
\textbf{Layer} & \textbf{Essence of Derivation} & \textbf{Manifestation in the Bearing Case} \\
\midrule
Core & Defines the ``space'' of functions & Defines the higher-order template \textsf{mkCausalProof}, specifying required input and output types. \\
Runtime & Provides the ``operands'' of functions & Extracts concrete values from SQL tables and elaborates them into logically meaningful atomic facts (e.g., $f_0, a_{volt}$). \\
Kernel & Performs function ``evaluation'' & Through Small-step semantics, fills fragmented data into Core-layer templates, materializing a concrete causal proof function $R$. \\
\bottomrule
\end{tabular}
\end{table}

\subsubsection{Why Is This Called ``Derivation''?}

In traditional software, adding a tracing logic such as ``voltage fluctuation causes hardness non-uniformity'' would require manually writing a function like \texttt{checkVoltage()}.

In KOS-TL, by contrast, one only defines general causal principles at the Core layer (e.g., $t(a) < t(f)$ together with physical correlation).
When the specific voltage data for \texttt{Batch\_202310-01} enters the system, the Kernel uses this data to \emph{derive} a batch-specific, specialized proof function instance.

\subsubsection{Conclusion}

Functions are not pre-written and fixed; they are dynamically derived logical results based on the current system state $\sigma$ and the observed fact $f$.

This example perfectly embodies the philosophy of \emph{``Proof as Program''}:
deriving a proof of a quality root cause is equivalent to deriving a logical function capable of explaining that failure.
Data (voltage, time, batch) are no longer merely objects to be processed; they become the building blocks of the very \emph{logical function} that explains the event.

\section{Conclusion}

\subsection{Philosophical Paradigm: From ``Theories of Truth'' to ``Executable Norms''}

Since Frege, the core of traditional logic (first-order logic, description logic) has been static truth valuation ($\mathcal{M} \models \varphi$), aiming to characterize \emph{what the world is like}.
KOS-TL explicitly rejects this ``single-model centrism'' and realizes a fundamental transformation in the role of logic:

\begin{itemize}
\item \textbf{State-based Semantics}:
The basic unit of semantics is no longer an eternal, immutable model, but a dynamically evolving sequence of states ($\sigma_0 \to \sigma_1 \to \dots$).
The focus of logical judgment shifts from ``whether a proposition is true'' to ``whether a state transition is valid.''

\item \textbf{Constructivist Stance}:
Inheriting the essence of Martin-L\"of Type Theory, the principle of ``propositions as types, proofs as programs'' is generalized to
``knowledge as types, operations as programs, and events as constructors.''
In KOS-TL, the existence of knowledge is determined by whether it can be constructed, making the system philosophically self-consistent and transparent.

\item \textbf{Unity of Knowing and Acting}:
Logic is no longer merely descriptive, but becomes an executable normative system.
It not only characterizes truth, but also specifies how knowledge is operated on, updated, and executed, bridging the gap between logical reasoning and physical action.
\end{itemize}

\subsection{Logical Characteristics: The Fusion of Event-Driven Reasoning and Operational Semantics}

The essential innovation of KOS-TL lies in introducing operational semantics into the logical core, thereby enabling ``computation as reasoning'':

\begin{itemize}
\item \textbf{Events as First-class Constructors}:
Unlike the passive fact records of description logic, events in KOS-TL serve as the nexus between understanding and action.
They are constructors in type theory, defining the legitimate mechanisms of object generation and providing a foundation for handling temporal sensitivity, causality, and state transitions.

\item \textbf{Operational Derivation Rules}:
Derivation rules are no longer purely truth-preserving; they become operational norms.
They define how the system may \emph{legally} produce new semantic annotations (e.g., risk alerts) under specific contextual and temporal constraints.
This normative stance allows the system to exhibit different reasoning strategies across scenarios, with high flexibility and explainability.

\item \textbf{Small-step Operational Semantics}:
Logical judgment takes the formal shape $\langle \Sigma, c \rangle \to \Sigma'$.
This fine-grained evolution enables KOS-TL to precisely manage monotonicity and resource consumption within the knowledge base, making it closer to a ``causal reasoning computer.''
\end{itemize}

\subsection{System Capabilities: Computation-Reflective Autonomy and Auditability}

Reflexivity in KOS-TL is elevated from a programming technique to a form of formal self-introspection:

\begin{itemize}
\item \textbf{Self-certifying Computation as Proof}:
Based on the Curry--Howard correspondence, at every computational step (reduction), the kernel simultaneously synthesizes an equivalence proof (an $\text{Id}$ proof).
Thus, each step leaves an immutable logical footprint, certifying that system behavior conforms to predefined axioms.

\item \textbf{End-to-end Logical Auditability}:
With reflexivity, auditing no longer relies on external textual logs, but becomes a real-time mathematical verification process.
The completeness of proof chains directly determines the legitimacy of system states, enforcing ``transparent governance.''

\item \textbf{Logical Self-inspection and Healing}:
Reflexivity allows the kernel to ``look back'' over decision paths, locate axiom conflicts when contradictions arise, and trigger self-healing operators, providing a foundation for autonomous system operation.
\end{itemize}

\subsection{Engineering Paradigm: Expanding the Boundaries of Type-Driven Development}

KOS-TL elevates the type system from a ``memory safety tool'' to a set of ``axioms of system autonomy,'' offering key insights for modern system design:

\begin{itemize}
\item \textbf{From Type Safety to Logical Determinacy}:
Through bidirectional physical--logical refinement types, physical laws and business rules are internalized as type properties.
These ``evidence-carrying types'' ensure that illegal states are unrepresentable at the design level.

\item \textbf{Persistent Dependent-Type Storage}:
Breaking away from databases as raw byte heaps, storage becomes a runtime extension of the type system.
Monotonicity constraints on types manage data lifecycles and ensure causal consistency in knowledge evolution.

\item \textbf{Cross-layer Logical Lenses (Refinement Lenses)}:
Through rigorous bidirectional refinement mappings, KOS-TL bridges the gap between high-level business entities (e.g., ``compliant transfers'') and low-level physical storage (e.g., SQL records), ensuring that every physical action faithfully reflects high-level logical intent.
\end{itemize}

\subsection{The Essence of KOS-TL}

KOS-TL is a logical system that seamlessly integrates the constructive semantics of intuitionistic type theory, operational semantics, and knowledge engineering practice.
It is not merely a collection of algorithms, but a \textbf{``constitution for systems.''}

It demonstrates that, through dependent types ($\Pi$ and $\Sigma$), event constructors, and reflexivity, one can build an intelligent autonomous system that is logically self-consistent, causally traceable, and tightly aligned with the physical world.
In this framework, proofs evolve from simple correctness checks into the central driving force behind decision-making in complex real-world systems.


\end{document}